\documentclass{article}
\usepackage{iclr2027_conference,times}
\iclrfinalcopy
\usepackage[T1]{fontenc}
\usepackage{amsmath,amssymb,amsthm,mathtools,bm}
\usepackage{graphicx,booktabs,hyperref,url}
\definecolor{linkblue}{RGB}{0,70,130}
\hypersetup{colorlinks=true,linkcolor=linkblue,citecolor=linkblue,urlcolor=linkblue,
  pdftitle={Moment-Accurate Gaussian Mixtures for Constant-Step Stochastic Approximation},
  pdfauthor={Xiaoli Li and Wei Biao Wu}}
\newtheorem{theorem}{Theorem}[section]
\newtheorem{proposition}[theorem]{Proposition}
\newtheorem{lemma}[theorem]{Lemma}
\newtheorem{corollary}[theorem]{Corollary}
\theoremstyle{definition}
\newtheorem{assumption}[theorem]{Assumption}

\newcommand{\R}{\mathbb R}
\newcommand{\E}{\mathbb E}
\newcommand{\Pp}{\mathbb P}
\newcommand{\Mcal}{\mathcal M}
\newcommand{\Scal}{\mathcal S}
\newcommand{\Ucal}{\mathcal U}
\newcommand{\N}{\mathsf N}
\newcommand{\Law}{\mathcal L}
\newcommand{\tr}{\operatorname{tr}}
\newcommand{\dist}{\operatorname{dist}}
\newcommand{\supp}{\operatorname{supp}}
\newcommand{\spec}{\operatorname{spec}}
\newcommand{\diag}{\operatorname{diag}}
\newcommand{\norm}[1]{\lVert#1\rVert}

\newcommand{\ip}[2]{\langle#1,#2\rangle}
\newcommand{\one}{\mathbf 1}
\newcommand{\dd}{\,\mathrm d}
\newcommand{\BL}{\mathrm{BL}}
\newcommand{\op}{\mathrm{op}}
\newcommand{\weak}{\Rightarrow}
\newcommand{\trans}{\mathsf T}
\allowdisplaybreaks
\title{Moment-Accurate Gaussian Mixtures for Constant-Step Stochastic Approximation}
\author{Xiaoli Li\\
University of Chicago\\
\texttt{xiaolil@uchicago.edu}
\And
Wei Biao Wu\\
University of Chicago\\
\texttt{wbwu@uchicago.edu}}
\begin{document}
\maketitle
\lhead{Preprint}
\begin{abstract}
Local Gaussian models of constant-step learning predict output variability and expected losses, but weak convergence alone does not justify these moment predictions.
We establish moment-accurate Gaussian mixtures by matching stationary energy with local Ornstein--Uhlenbeck limits, ruling out quadratic tail mass invisible to weak convergence.
For step size $a$, the second-order Wasserstein error is $o(\sqrt a)$, uniformly over invariant laws, using each law's actual root weights.
The assumptions combine confinement, descent, finitely many hyperbolic equilibria and root continuity with finite-variance innovations.
The result yields observable covariances, expected objective gaps and first-order mean shifts, while allowing singular covariances, compatible saddles and weights without a limit.
For additive noise given by a fixed invertible transform of independent standardized Student $t_3$ coordinates, symmetry gives an order-sharp $\sqrt a$ smooth-test bound.
Numerical transport calculations demonstrate the value of root-specific covariances; controlled SGD studies assess observable predictions across step sizes, batch sizes and model geometries.
\end{abstract}

\section{Introduction}\label{sec:intro}
With a fixed learning rate $a$, stochastic gradient descent (SGD) keeps fluctuating near a trained solution.
Local Gaussian models turn curvature and gradient-noise covariance into predictions of output variability and expected objective gaps \citep{MandtEtAl2017,Yaida2019,LiuZiyinUeda2021}.
These predictions concern moments, so their validity requires more than a Gaussian description of typical fluctuations.
When does local Gaussian theory accurately predict these quantities in a landscape with several equilibria?

The first obstacle is that weak convergence can hide the very quantities these predictions require.
As $R\to\infty$, $(1-R^{-2})\N(0,1)+R^{-2}\delta_R$ converges weakly to $\N(0,1)$, while its second moment tends to $2$.
A vanishing fraction of observations can therefore carry a nonvanishing share of quadratic energy.
Single-center moment results address this difficulty under their respective stability and regularity conditions \citep{DieuleveutDurmusBach2020,WeiLiLouWu2025}.
Multiple equilibria introduce a second obstacle: their local covariances differ, while their occupation weights depend on global dynamics and the selected invariant law.

We study stochastic approximation (SA) with constant step size $a$,
\begin{equation}\label{eq:recursion}
X_{n+1}^{(a)}=X_n^{(a)}-ag(X_n^{(a)})+a\xi(X_n^{(a)},\zeta_{n+1}),
\qquad X_n^{(a)}\in\R^d,
\end{equation}
where the observations are i.i.d.\ and the innovations are conditionally centered.
For an invariant draw $X\sim\pi_a$, let $J_a$ label its nearest root $m_j$ and write $w_{a,j}=\Pp_{\pi_a}(J_a=j)$.
At a stable root, set $A_j=Dg(m_j)$ and $C_j=\E[\xi(m_j,\zeta)\xi(m_j,\zeta)^T]$; the local covariance solves $A_jS_j+S_jA_j^T=C_j$.
For components with $\liminf_{a\downarrow0}w_{a,j}>0$, the moment predictions take the familiar form
\begin{equation}\label{eq:intro-predictions}
\begin{aligned}
\operatorname{Cov}(X\mid J_a=j)&=aS_j+o(a),\\
\operatorname{Cov}(f(X)\mid J_a=j)&=aDf(m_j)S_jDf(m_j)^T+o(a).
\end{aligned}
\end{equation}
The second formula applies to globally Lipschitz $f$ differentiable at $m_j$.
For gradient drift $g=\nabla F$, the mean objective gap is $\E[F(X)-F(m_j)\mid J_a=j]=a\tr(C_j)/4+o(a)$.
Our question is when these moment statements hold jointly across roots under second-moment noise assumptions.

We resolve this question by matching two descriptions of stationary energy.
Descent controls squared distance to the root set, while stationarity determines its leading energy balance.
Local Ornstein--Uhlenbeck (OU) limits satisfy the same balance, excluding quadratic mass hidden in rare displacements.
This proves uniform integrability of $\norm{Z_a}^2$, where $Z_a=(X-m_{J_a})/\sqrt a$, and upgrades weak Gaussian limits to moment accuracy.
Under confinement, descent, finitely many hyperbolic roots and the noise regularity in Assumption~\ref{ass:standing}, the resulting mixture has $o(\sqrt a)$ Wasserstein error, uniformly over invariant laws.
Each mixture uses its law's actual finite-$a$ weights: local derivatives determine fluctuations, while global communication and invariant-law selection determine occupation.

\paragraph{Contributions.}
\textbf{Moment-accurate mixtures.} Theorems~\ref{thm:local-moments}--\ref{thm:W2} cover all roots, including compatible saddles, singular covariances and weights without a limit.
\textbf{Observable predictions.} Section~\ref{sec:prediction} derives covariance, loss and first-order mean formulas; Section~\ref{sec:rate} quantifies smooth-test accuracy, including a sharp $\sqrt a$ endpoint for transformed independent $t_3$ noise.
\textbf{Numerical evidence.} Root-specific transport separates the mixture from pooled covariance models; SGD and heavy-tail studies assess the moment predictions using matched OU measurement references.

\paragraph{Relation to existing theory.}
Classical constant-gain analyses establish stationary asymptotics \citep{KushnerHuang1981,Pflug1986,KushnerShwartz1984,KushnerYin2003}.
The finite-variance single-center results of \citet{ChenMouMaguluri2022} identify scaled limits through an integral equation and Gaussian uniqueness conditions.
Further single-center analyses establish moment expansions, Gaussian approximation and stationary concentration \citep{DieuleveutDurmusBach2020,WeiLiLouWu2025,WangEtAl2026,MeradGaiffas2025}; Table~\ref{tab:theory-position} compares their assumptions and conclusions.
Our additional step is uniform quadratic-tail control across all roots using second innovation moments and root continuity.
Finite-time $W_1$ bounds with recursively defined Gaussian covariances address a complementary time scale \citep{haque2026gaussianiterates}.

\begin{table}[b]
\centering\small
\caption{Gaussian and moment results for SA in their stated model, moment and time settings.}
\label{tab:theory-position}
\begin{tabular}{p{0.27\linewidth}p{0.34\linewidth}p{0.29\linewidth}}
\toprule
Work & Setting and moment control & Relevant conclusion\\
\midrule
\citet{ChenMouMaguluri2022} & Single-center stability; additive iid finite-variance noise & Stationary scaling limit; Gaussian identification under uniqueness conditions\\
\citet{DieuleveutDurmusBach2020} & Strong convexity; higher regularity and moments & Stationary moment and bias expansions\\
\citet{WeiLiLouWu2025}, Thms.~3.4--3.5 & Strong convexity; $L^p$ regularity; joint long-time/small-step limit & CLT at $p=2$; convex-set bounds at $p\ge4$\\
\citet{WangEtAl2026}, Thm.~3.8 & Contractive setting; additive noise with finite third moment & Stationary $W_1$ rate; derived scaled third-moment bound\\
This paper & Finite hyperbolic roots; descent; second innovation moments and root continuity & Law-uniform all-root moments and actual-weight $W_2$ approximation\\
\bottomrule
\end{tabular}
\end{table}

Nonconvex SA theory characterizes invariant support, saddle exclusion and nonunique invariant laws \citep{FortPages1999,Benaim1996,Benaim1999,BianchiHachemSchechtman2022,ShirokoffZaleski2025}.
We resolve moments within the root-labeled components; temporal-average normality addresses a different random object \citep{YuEtAl2021}.
The covariance, quarter-trace loss and bias formulas have single-center antecedents \citep{Yaida2019,LiuZiyinUeda2021,DieuleveutDurmusBach2020,AllmeierGast2024}; energy matching justifies their rootwise use with low-moment remainders.
Occupation weights instead use global communication information, as in light- and heavy-tail metastability theory \citep{FreidlinWentzell2012,AzizianEtAl2024,WangRhee2023}.

Diffusion and modified equations describe finite-step dynamics \citep{LiTaiE2017,LiMalladiArora2021}; stationary approximation requires control through the infinite-time limit, as provided by Poisson methods under their assumptions \citep{MattinglyStuartTretyakov2010}.
Our energy argument works directly with discrete invariant laws.
Multiplicative noise links innovation and iterate tails \citep{GurbuzbalabanEtAl2021,HodgkinsonMahoney2021}, while equilibrium manifolds add tangential dynamics \citep{LiWangArora2022}.
Appendix~\ref{app:literature} develops these comparisons.

\section{Model and assumptions}\label{sec:model}

The observations $(\zeta_n)$ are i.i.d.\ and the measurable innovation map $\xi$ does not depend on $a$.
Write $P_af(x)=\E f(x-ag(x)+a\xi(x,\zeta_1))$ and
$C(x)=\E[\xi(x,\zeta_1)\xi(x,\zeta_1)^\trans]$.
The deterministic flow on physical time $t=na$ is $\dot x=-g(x)$.
We use $W_2(\mu,\nu)^2=\inf_{\gamma\in\Pi(\mu,\nu)}\int\norm{x-y}^2\,\gamma(\mathrm dx,\mathrm dy)$, where $\Pi(\mu,\nu)$ denotes all couplings. Throughout, constants are independent of sufficiently small $a$ and of the selected invariant law.

The assumptions combine confinement, locally linearizable roots and a descent energy with second-moment noise control and root continuity.
\begin{assumption}[Confinement, descent and local regularity]\label{ass:standing}
The following conditions hold.
\begin{enumerate}
\item \label{ass:noise} For every $x$, $\E\xi(x,\zeta_1)=0$ and
\begin{equation}\label{eq:cov-growth}
\tr C(x)\le c_0+c_1\norm{g(x)}^2.
\end{equation}
The kernels $P_a$ are Feller for sufficiently small $a$.
At every root $m_j$, $C(x)\to C_j:=C(m_j)$ as $x\to m_j$.
\item \label{ass:geometry} The drift $g$ is $C^2$ and globally Lipschitz, and
\begin{equation}\label{eq:dissipativity}
\ip{x}{g(x)}\ge\kappa\norm{x}^2-b,\qquad \kappa>0.
\end{equation}
Its root set $\Mcal=\{m_1,\ldots,m_\ell\}$ is finite and nonempty.
Each $A_j=Dg(m_j)$ is hyperbolic, meaning $\spec(A_j)\cap i\R=\varnothing$.
\item \label{ass:descent} There is $U\in C^2(\R^d)$ with
\begin{equation}\label{eq:lyapunov}
\begin{gathered}
|U(x)|\le C_U(1+\norm{x}^2),\qquad
\sup_x\norm{D^2U(x)}_{\op}\le L_U,\\
\ip{\nabla U(x)}{g(x)}\ge c_U\norm{g(x)}^2,
\end{gathered}
\end{equation}
where $c_U>0$.
\end{enumerate}
\end{assumption}

Concentration and root continuity supply stationary innovation-tail control (Lemma~\ref{lem:stationary-noise-tails}), even without uniform integrability on fixed root neighborhoods (Appendix~\ref{app:root-continuity-example}).
For additive noise, a finite second moment suffices for the noise conditions.
For a gradient drift $g=\nabla F$, one may take $U=F$ when the displayed growth and smoothness conditions hold.

Let $\Scal=\{j:\min_{\lambda\in\spec(A_j)}\Re\lambda>0\}$ and $\Ucal=\{1,\ldots,\ell\}\setminus\Scal$.
Thus $\Scal$ indexes the attracting equilibria of $\dot x=-g(x)$.
Break nearest-root ties by the smallest index and denote the resulting Voronoi cells by $\mathcal C_j$.
For $X\sim\pi_a$, define $J_a=j$ on $\mathcal C_j$, $X_{\Mcal}=m_{J_a}$, $Z_a=(X-m_{J_a})/\sqrt a$, and $w_{a,j}=\pi_a(\mathcal C_j)$.
All roots, including saddles, enter this labeling convention.

\section{Gaussian mixtures and moment accuracy}\label{sec:main}
We first locate stationary mass and identify its fluctuation scale around the roots.
This concentration estimate supplies the starting point for local Gaussian limits.
\begin{theorem}[Existence and concentration]\label{thm:concentration}
Under Assumption~\ref{ass:standing}, an invariant probability exists for every sufficiently small $a$.
Every invariant probability has a finite second moment and satisfies
\begin{equation}\label{eq:concentration}
\E\norm X^2\le K,\qquad \E\norm{g(X)}^2\le Ka,\qquad
\E\dist(X,\Mcal)^2\le Ka.
\end{equation}
The constants are uniform over invariant-law selections. In particular, $(J_a,Z_a)$ is tight.
\end{theorem}

\paragraph{From concentration to moments.}
Concentration fixes the scale, but rare excursions could still carry a nonvanishing share of the rescaled second moment.
The descent dissipation $D(x)=\langle\nabla U(x),g(x)\rangle$ controls this contribution:
the global error bound $\norm{g(x)}\ge c_g\dist(x,\Mcal)$ gives $D(X)/a\ge c_Uc_g^2\norm{Z_a}^2$.
Writing $E_j=D^2U(m_j)$, stationarity and Lemma~\ref{lem:stationary-noise-tails} yield
\begin{equation}\label{eq:energy-main}
\frac{\E D(X)}a=\frac12\sum_jw_{a,j}\tr(E_jC_j)+o(1).
\end{equation}
Along any convergent subsequence of root weights and localized laws, the OU moment identity gives exactly the same limiting energy.
Agreement of the distributional limit and its mean excludes hidden energy in the tails, yielding uniform integrability of $\norm{Z_a}^2$.
Appendices~\ref{app:ou}--\ref{app:energy} prove this matching and its uniformity over invariant laws.

\paragraph{Local Gaussian components.}
At a stable root, $A_j$ has eigenvalues with positive real parts, and the covariance is the unique solution
\begin{equation}\label{eq:local-covariance}
A_jS_j+S_jA_j^T=C_j,\qquad \Gamma_j=\N(0,S_j).
\end{equation}
To include saddle components, we retain fluctuations on their stable subspaces and identify when noise excites an unstable direction.
Let $P_j^s$ and $P_j^u=I-P_j^s$ be the real spectral projectors onto the positive- and negative-real-part generalized eigenspaces of $A_j$.
They need not be orthogonal. Define
\begin{equation}\label{eq:all-covariance}
\bar S_j=\int_0^\infty e^{-A_jt}P_j^sC_j(P_j^s)^Te^{-A_j^Tt}\dd t,
\qquad \bar\Gamma_j=\N(0,\bar S_j).
\end{equation}
The integral converges on the stable subspace; if that subspace is empty, $\bar\Gamma_j=\delta_0$.
At a stable root, $P_j^s=I$, so $\bar S_j=S_j$ and $\bar\Gamma_j=\Gamma_j$.

The next theorem supplies the tail control needed to turn these local Gaussian descriptions into covariance predictions.
\begin{theorem}[Rootwise moment accuracy]\label{thm:local-moments}
Under Assumption~\ref{ass:standing},
\begin{equation}\label{eq:square-ui}
\lim_{R\to\infty}\limsup_{a\downarrow0}\sup_{\pi\in\mathcal I_a}
\E_\pi[\norm{Z_a}^2\one\{\norm{Z_a}>R\}]=0,
\end{equation}
where $\mathcal I_a$ is the set of invariant probabilities. At every root, uniformly over $\pi\in\mathcal I_a$,
\begin{equation}\label{eq:second}
\E[\one_{\mathcal C_j}(X)Z_aZ_a^T]-w_{a,j}\bar S_j\to0,
\qquad \E[\one_{\mathcal C_j}(X)Z_a]\to0.
\end{equation}
For any selection with $\liminf_{a\downarrow0}w_{a,j}>0$,
\begin{equation}\label{eq:conditional}
W_2\bigl(\Law(Z_a\mid J_a=j),\bar\Gamma_j\bigr)\to0,
\qquad a^{-1}\operatorname{Cov}(X\mid J_a=j)\to\bar S_j.
\end{equation}
Such a component necessarily satisfies $P_j^uC_j(P_j^u)^T=0$.
\end{theorem}
Thus a saddle can retain mass only if its root noise acts within its stable subspace; Appendix~\ref{app:unexcited} gives an exact example.

We now assemble the local predictions using each invariant law's actual root weights.
Matching observations and Gaussian samples within the same root makes the transport guarantee explicit.
\begin{theorem}[All-root Gaussian approximation]\label{thm:mixture}\label{thm:W2}
Under Assumption~\ref{ass:standing}, let
$\bar G_{a,\pi}=\sum_jw_{a,j}\N(m_j,a\bar S_j)$.
Define the same-label scaled transport cost by
\begin{equation}\label{eq:same-label-cost}
D_a(\pi)^2=\sum_{j:w_{a,j}>0}w_{a,j}
W_2^2\bigl(\Law(Z_a\mid J_a=j),\bar\Gamma_j\bigr).
\end{equation}
Then
\begin{equation}\label{eq:same-label-main}
\sup_{\pi\in\mathcal I_a}D_a(\pi)\longrightarrow0,\qquad
W_2(\pi,\bar G_{a,\pi})\le\sqrt a\,D_a(\pi).
\end{equation}
In particular,
\begin{equation}\label{eq:original-W2}
\sup_{\pi\in\mathcal I_a}a^{-1/2}W_2(\pi,\bar G_{a,\pi})\longrightarrow0.
\end{equation}
If $P_j^uC_j(P_j^u)^T\ne0$, then $w_{a,j}\to0$ uniformly.
\end{theorem}
The same-label cost preserves each root's fluctuation model; rescaling its couplings gives the original-scale bound.
The weights may vary without converging, covering both communicating wells and disjoint invariant regions (Appendices~\ref{app:counter} and~\ref{app:long-communication}).

Noise excitation provides a useful specialization in which only stable components retain asymptotic mass.
\begin{assumption}[Excitation of unstable roots]\label{ass:excitation}
For every $j\in\Ucal$,
\begin{equation}\label{eq:excitation}
P_j^uC_j(P_j^u)^T\ne0.
\end{equation}
\end{assumption}
A sufficient condition is excitation of a left eigenvector:
\begin{equation}\label{eq:eigen-excitation}
A_j^Tv=\lambda v,\qquad \Re\lambda<0,\qquad v^*C_jv>0.
\end{equation}
The projection condition also covers defective blocks missed by this sufficient condition (Lemma~\ref{lem:hyperbolic-OU}).
For the following marked $W_2$, use squared cost $\one\{j\ne k\}+\norm{z-y}^2$ between $(j,z)$ and $(k,y)$.
\begin{corollary}[Stable components under excitation]\label{cor:stable-mixture}
Under Assumptions~\ref{ass:standing} and~\ref{ass:excitation}, let $s_a=\sum_{j\in\Scal}w_{a,j}$. Then, uniformly,
\begin{equation}\label{eq:target}
s_a\to1,\quad W_2\!\left(\Law(J_a,Z_a),Q_a\right)\to0,
\quad Q_a=\sum_{j\in\Scal}\frac{w_{a,j}}{s_a}(\delta_j\otimes\Gamma_j),
\end{equation}
and $\E[\one\{J_a\in\Ucal\}\norm{Z_a}^2]\to0$.
\end{corollary}

\section{Observable covariances, losses and mean shifts}\label{sec:prediction}
The moment theorem turns local parameter fluctuations into predictions for outputs, losses and systematic mean shifts.
We first propagate covariance through a differentiable observable.
\begin{corollary}[Observable fluctuations and covariance]\label{cor:prediction}\label{cor:observable}
Assume Theorem~\ref{thm:local-moments} and fix a root with $\liminf_a w_{a,j}>0$.
For measurable $f:\R^d\to\R^k$ differentiable at $m_j$,
\begin{equation}\label{eq:prediction}
\Law\left(\frac{f(X)-f(m_j)}{\sqrt a}\,\middle|\,\mathcal C_j\right)
\weak\N(0,Df(m_j)\bar S_jDf(m_j)^\trans).
\end{equation}
If $f$ is also globally Lipschitz, convergence holds in $W_2$ and
$a^{-1}\operatorname{Cov}(f(X)\mid\mathcal C_j)\to Df(m_j)\bar S_jDf(m_j)^\trans$.
\end{corollary}

At a minimum, the objective has no linear term, so its leading fluctuation is quadratic.
Second-moment accuracy makes the resulting prediction valid for the expected loss as well as its distribution.
\begin{corollary}[Loss law and mean]\label{cor:loss}
Assume Theorem~\ref{thm:local-moments}, let $g=\nabla F$ and fix a stable root with $\liminf_a w_{a,j}>0$.
For $H_j=D^2F(m_j)\succ0$ and $G_j\sim\Gamma_j$,
\begin{align}
\frac{F(X)-F(m_j)}a\mid\mathcal C_j&\weak\tfrac12G_j^\trans H_jG_j,\label{eq:losslaw}\\
\frac{\E[F(X)-F(m_j)\mid\mathcal C_j]}a&\to\tfrac12\tr(H_jS_j)=\tfrac14\tr C_j.\label{eq:lossmean}
\end{align}
The limit law is $\frac12\sum_r\mu_r\chi^2_{1,r}$ with independent chi-squares, where $\mu_r$ are the eigenvalues of $S_j^{1/2}H_jS_j^{1/2}$, including zeros.
Without individual weight lower bounds, centering at the nearest root gives
\begin{equation}\label{eq:weighted-loss}
\frac{\E[F(X)-F(m_{J_a})]}a-\frac14\sum_{j=1}^{\ell}w_{a,j}\tr C_j\to0.
\end{equation}
\end{corollary}
The global Lipschitz property controls observable growth; the bounded Hessian of $F$ controls quadratic loss by squared displacement. Appendix~\ref{app:observables} proves the resulting expectation limits.
The quarter-trace coefficient agrees with the local harmonic prediction of \citet{Yaida2019}; here its remainder and conditional interpretation follow from the moment theorem.

A centered Gaussian describes fluctuations of order $\sqrt a$, while nonlinear drift can shift the mean at order $a$.
The next expansion identifies this smaller systematic effect from the same rootwise covariance.
\begin{theorem}[Rootwise mean expansion]\label{thm:bias}
Under Assumption~\ref{ass:standing}, fix any root $j$ and define
\[
(t_j)_r=\tr[D^2g_r(m_j)\bar S_j],\qquad b_j=-\tfrac12A_j^{-1}t_j.
\]
Uniformly over invariant-law selections,
\begin{equation}\label{eq:mean-bias}
\E[\one_{\mathcal C_j}(X-m_j)]=a w_{a,j}b_j+o(a).
\end{equation}
For $f:\R^d\to\R^k$ of class $C^2$ with bounded coordinatewise Hessians, let
$k_{f,j}=Df(m_j)b_j+\tfrac12(D^2f(m_j):\bar S_j)$. Then
\begin{equation}\label{eq:observable-bias}
\E[\one_{\mathcal C_j}\{f(X)-f(m_j)\}]=a w_{a,j}k_{f,j}+o(a).
\end{equation}
If $\liminf_a w_{a,j}>0$, the conditional means are $m_j+a b_j+o(a)$ and $f(m_j)+a k_{f,j}+o(a)$, respectively.
\end{theorem}
Here the $r$th coordinate of $D^2f:\bar S_j$ is $\tr[D^2f_r \bar S_j]$.
The proof refines the cutoff flux to $o(a^2)$ and expands the rootwise drift balance (Appendix~\ref{app:bias}), extending single-center bias expansions to rootwise low-moment asymptotics \citep{DieuleveutDurmusBach2020,AllmeierGast2024}.
A nonzero mean shift at a root whose weight stays bounded away from zero contributes order $\sqrt a$ to the same-label scaled error (Appendix~\ref{app:mean-offset}).
For a fixed test loss, the coefficient is
\begin{equation}\label{eq:test-bias}
k_{{\rm test},j}=\nabla\ell_{\rm test}(m_j)^Tb_j+\tfrac12\tr[D^2\ell_{\rm test}(m_j)\bar S_j].
\end{equation}
The first term vanishes for the full objective generating a gradient drift, yielding~\eqref{eq:lossmean}.
For a scalar globally Lipschitz observable with bounded Hessian, these coefficients also separate global variance into between-root variation, within-root Gaussian variance and a mean-shift correction (Corollary~\ref{cor:global-variance}).

For each fixed independent with-replacement batch size $B$, the matrices $C_j,\bar S_j$ and coefficients $b_j,k_{f,j}$ scale by $1/B$.
Thus fluctuations scale as $\sqrt{a/B}$ and leading mean shifts as $a/B$.
\paragraph{Computing predictions.}
Evaluate $A_j,C_j$ at each relevant root, solve the Lyapunov equation on its stable subspace, and propagate covariance through observable derivatives; second derivatives give $b_j,k_{f,j}$.
These coefficients give leading conditional predictions; mixture aggregation also uses invariant-law weights from a global model or ergodically justified time occupations (Appendix~\ref{app:metastability}).
Output intervals quantify algorithmic variability about fixed-root predictions.

\section{Noise tails and smooth accuracy}\label{sec:rate}
The preceding results justify moments and qualitative $W_2$ accuracy.
Additional noise structure gives explicit rates for bounded smooth tests of the root-scaled law.
For $0<\delta\le1$, let $\mathcal H_{2,\delta}$ contain real functions satisfying
$\norm h_\infty+\norm{\nabla h}_\infty+\norm{D^2h}_\infty+[D^2h]_\delta\le1$.
The last term is the global H\"older seminorm, and $\Gamma_jh=\int h(z)\,\Gamma_j(\mathrm dz)$.
At a stable root, measure the unnormalized local error by
\begin{equation}\label{eq:local-error}
\mathcal E_{a,j}=\sup_{h\in\mathcal H_{2,\delta}}
\left|\E[\one_{\mathcal C_j}(X)h((X-m_j)/\sqrt a)]-w_{a,j}\Gamma_jh\right|.
\end{equation}
\begin{corollary}[Additive-noise accuracy]\label{cor:tail-summary}
Under Assumption~\ref{ass:standing}, centered additive innovations with a finite $(2+\delta)$ moment give $\mathcal E_{a,j}=O(a^{\delta/2})$ at every stable root.
For a fixed invertible linear transform of $d$ independent standardized $t_3$ coordinates, $\mathcal E_{a,j}=O(\sqrt a)$ on $\mathcal H_{2,1}$.
\end{corollary}
Despite infinite third absolute moments, the $t_3$ bound follows from symmetry and OU smoothing; invertibility ensures $C_j\succ0$ (Corollary~\ref{cor:symmetric-rate}).
In the scalar linear example, the same cosine test has error $O(a)$ for Gaussian or Rademacher innovations and order $\sqrt a$ for $t_3$ noise (Proposition~\ref{prop:symmetric-lower}).
For state-dependent noise, a uniformly bounded local $(2+\delta)$ moment and a locally $\beta$-H\"older covariance give
$\mathcal E_{a,j}=O(a^{\delta/2}+a^{\beta/2}+\sqrt a)$ at stable roots, for $0<\delta,\beta\le1$ (Theorem~\ref{thm:rate}).
Appendices~\ref{app:rate}--\ref{app:symmetric} derive these smooth-test rates, tail moduli and quantitative exclusion of noise-excited saddles.

\section{Numerical studies}\label{sec:numerics}\label{sec:experiments}
We first test root-specific moments and transport, then examine covariance, loss and mean predictions in controlled SGD. Heavy-tail and communication studies connect these measurements to the noise assumptions and occupation weights.
Predictions use roots, derivatives and noise covariances evaluated independently of the measured trajectories.
Covariance and loss settings use 16 independent paths; mean confirmation uses 32 per step.
Error bars denote one standard error across paths unless stated otherwise.
Percentage deviations are $100(\widehat\theta/\theta_{\rm pred}-1)$ for an estimate $\widehat\theta$ and its prediction $\theta_{\rm pred}$.

\subsection{Root-specific moments and transport}
The two-dimensional example of Proposition~\ref{prop:counter} has invariant rectangles surrounding roots with
$S_-=\diag(1/75,1/50)$ and $S_+=\diag(1/60,1/50)$.
At $a=0.0025$, measured scaled first-coordinate variances are $0.013363\pm0.000054$ and $0.016691\pm0.000104$.
A fixed equal-weight pooled reference assigns $0.015$ to both roots, giving limiting relative errors of $+12.5\%$ and $-10\%$.
Root labels therefore matter even when an aggregate appears accurate (Figure~\ref{fig:root-moments}a).
Proposition~\ref{prop:contract} bounds the scaled initialization error by $10^{-6}$.
We also compute the invariant product laws by a refined one-dimensional transfer discretization.
For the same-label cost $D_a$ in~\eqref{eq:same-label-cost}, select invariant mixtures whose root weights alternate between $1/4$ and $3/4$.
As $a$ decreases from $0.04$ to $0.0025$, $D_a$ falls from $0.00310$ to $0.00028$, while the fixed pooled cost remains near $0.0069$ (Figure~\ref{fig:root-moments}c).
Even an optimally chosen shared covariance has a positive limiting same-label cost of $0.005902$ along this sequence (Proposition~\ref{prop:shared-covariance}).
Appendix~\ref{app:transport-study} gives the quantile calculation and grid refinement.

\begin{figure}[t]
\centering\includegraphics[width=\linewidth]{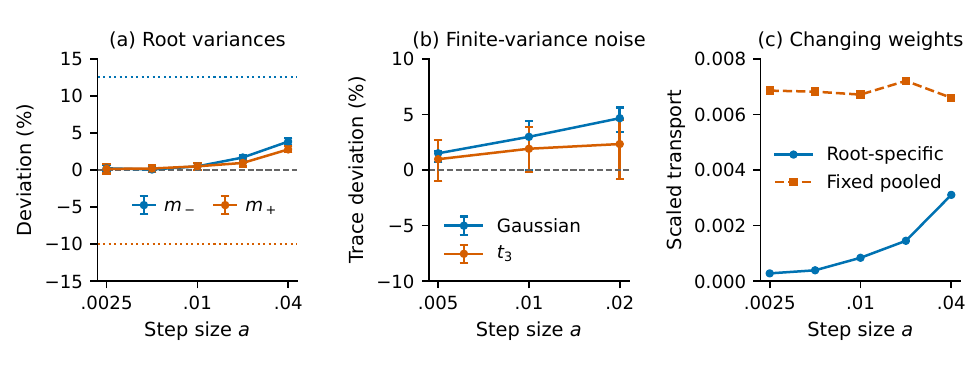}
\caption{\textbf{Root-specific moments and transport.} (a) Scaled variance deviations from each root's OU target; dotted lines show the fixed equal-weight pooled reference's limiting errors. Bars are one standard error across 16 runs. (b) Second-moment trace deviations in the nongradient model; medians and interquartile ranges retain $t_3$ dispersion. (c) Deterministic same-label scaled transport with alternating weights, using root-specific covariances or the same fixed pooled target.}
\label{fig:root-moments}
\end{figure}

\subsection{Neural SGD: prediction accuracy and measurement variability}
We train 64 input weights in a four-unit $\tanh$ network with fixed output weights on the UCI optical digits 3/8 task.
The writer-disjoint split has 769 training and 357 test examples, represented by 16 training-fitted principal components \citep{AlpaydinKaynak1998}.
Ridge $0.02$ confines the model; a fixed random tilt $10^{-4}r^Tx$ makes roots generically nondegenerate, and bounded dither with covariance $0.015^2I$ supplies noise excitation (Proposition~\ref{prop:neural}).
Twenty optimization starts locate one numerical root $m$, with Hessian $H=D^2F(m)$ having spectrum $[0.02435,0.24009]$.
We freeze this root and 12 test-logit probes before measuring fluctuations.

Per-run mean absolute relative error (MARE) averages absolute variance errors relative to the 12 probe predictions.
Aggregate MARE averages each probe's variance across paths before taking absolute relative errors.
At the smallest tested normalized step, $a\lambda_{\max}(H)=0.000625$, per-run MARE is $6.1$--$7.0\%$ and aggregate MARE is $0.9$--$2.0\%$; the $B=1$ mean variance and objective-gap ratios are $0.986$ and $1.029$.
Stationary OU paths with the same save schedule give per-run MARE $5.89\pm0.19\%$, showing the contribution of finite-trajectory variability (Figure~\ref{fig:neural-range}b).
Fourfold longer exposure reduces SGD MARE to $3.0$--$3.2\%$ for $B=1,16$, alongside OU $3.34\pm0.10\%$.
The $B=1$ objective-gap ratio remains $1.028\pm0.004$ (Figure~\ref{fig:exposure}).

\begin{figure}[t]
\centering\includegraphics[width=\linewidth]{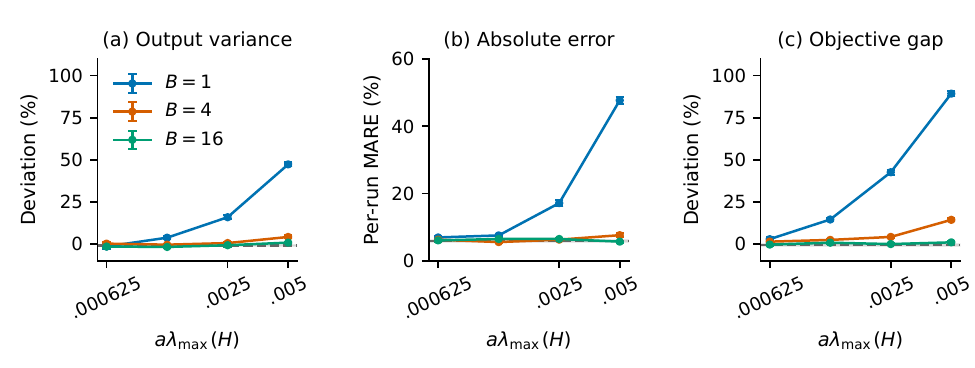}
\caption{\textbf{Neural predictions across step sizes.} Output-variance and actual objective-gap deviations share the $[-10,110]\%$ range; MARE starts at zero. All four steps, three batches and 16 paths per setting are included. Dashed lines and bands show the mean and one standard error of 64 matched stationary OU paths, whose loss is quadratic.}
\label{fig:neural-range}
\end{figure}

The full covariance is important: small-step logit MARE is $6.1$--$7.0\%$, compared with $228$--$233\%$ for a diagonal-noise reference and $185$--$190\%$ for an isotropic reference (Table~\ref{tab:revision-baselines}).
At the largest normalized step, $a\lambda_{\max}(H)=0.005$, with $B=1$, output variances and objective gaps exceed their predictions by $47.3\%$ and $89.3\%$.
Replacing the OU covariance by the exact discrete-linear reference changes logit MARE from $47.65\%$ to $47.41\%$ at this setting.
The frozen quadratic model's relative loss correction is at most $0.251\%$ (Proposition~\ref{prop:quadratic-discrete}), far below the observed objective-gap deviation.
Paired-state diagnostics, augmentation ablations and longer exposures characterize these finite-step departures (Appendix~\ref{app:moment-studies}).

\paragraph{Model dimension and data.}
An independent 192-path study covers 256/1024-parameter digits networks and a 256-parameter Wisconsin Diagnostic Breast Cancer (WDBC) classifier.
Small-step MARE is $2.74$--$3.69\%$, alongside OU $2.91$--$3.15\%$, with residual mean-ratio biases of roughly $2$--$3\%$ for digits-1024 (Table~\ref{tab:main-neural-comparison}).
Appendix~\ref{app:dimension} reports all steps and batches.

\begin{table}[t]
\centering\small
\caption{Small-step prediction ratios across models and batches at $a\lambda_{\max}(H)=0.000625$. Entries are means $\pm$ one standard error across 16 independent paths; all use physical exposure 8000. Output ratios average the 12 fixed probes within each path.}
\label{tab:main-neural-comparison}
\begin{tabular}{lcccc}
\toprule
Model & \shortstack{Output variance\\$B=1$} & \shortstack{Objective gap\\$B=1$} & \shortstack{Output variance\\$B=16$} & \shortstack{Objective gap\\$B=16$}\\
\midrule
Digits-256 & $1.012\pm0.005$ & $1.030\pm0.005$ & $0.998\pm0.003$ & $0.997\pm0.003$\\
Digits-1024 & $1.024\pm0.003$ & $1.027\pm0.003$ & $0.996\pm0.005$ & $0.999\pm0.003$\\
WDBC-256 & $1.003\pm0.003$ & $1.012\pm0.005$ & $0.997\pm0.003$ & $1.003\pm0.004$\\
\bottomrule
\end{tabular}
\end{table}

\subsection{Nonlinear observables, objective gaps and mean shifts}
Consider uniform sampling from four component losses,
\begin{equation}\label{eq:finite-sum}
\begin{gathered}
F(x)=\tfrac12\norm x^2-2\log\cosh x_1-1.5\log\cosh x_2,\\
\ell_{s,t}(x)=F(x)+s b^Tx+0.6t\log\cosh(v^Tx),\quad
b=(0.5,0.2)^T,\quad v=(-0.3,0.7)^T,
\end{gathered}
\end{equation}
where $s,t\in\{-1,1\}$ are independent.
The single-sample covariance is $C_1(x)=bb^T+0.36\tanh^2(v^Tx)vv^T$. At the positive root $m$, write $H=D^2F(m)$ and let $S_1$ solve $HS_1+S_1H=C_1(m)$.
For $f_2(x)=\tanh(x_1-x_2/2)$ and $a=0.005$, variance ratios are $1.009\pm0.007$, $1.003\pm0.005$ and $0.999\pm0.005$ for $B=1,4,16$.
The actual mean objective-gap ratios are $1.010\pm0.006$, $1.000\pm0.005$ and $1.005\pm0.004$ (Figure~\ref{fig:observable-moments}a,b).
The generalized chi-square prediction also describes the loss distribution (Figure~\ref{fig:revision-e3}).

For $f=(x_1+x_2,\tanh(x_1-x_2/2))$, Theorem~\ref{thm:bias} predicts $k_f=(-0.055598,-0.027255)^T$ at the positive root for $B=1$.
Writing $T=aN$ for the physical exposure of $N$ retained updates, a 96-path measurement fixes $aT=1600$, with 32 independent paths at each of $a=0.04,0.02,0.01$.
At the smallest step, the signed estimates are $-0.05470\pm0.00429$ and $-0.027249\pm0.000498$ (Figure~\ref{fig:observable-moments}c).
Appendix~\ref{app:mean-measurements} also reports an independent 48-path measurement with larger discrepancies, paired burn-in windows and an empirical drift decomposition.

\begin{figure}[t]
\centering\includegraphics[width=\linewidth]{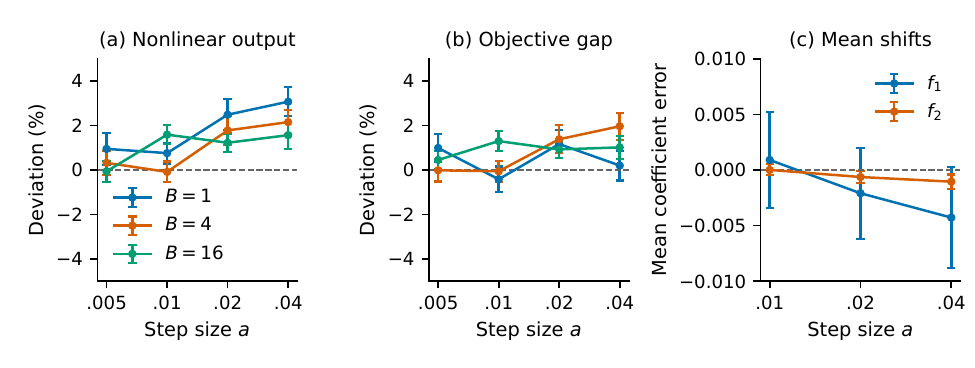}
\caption{\textbf{Predictions beyond parameter covariance.} Percentage deviations of (a) nonlinear output variance from $(a/B)Df_2(m)S_1Df_2(m)^T$ and (b) actual objective gap from $(a/4B)\tr C_1(m)$ share the range $[-5,5]\%$. (c) Signed mean-coefficient errors, with $T=1600/a$. Bars are one standard error across 16 paths in (a,b) and 32 in (c); the two outputs share paths. Zero denotes the prediction.}
\label{fig:observable-moments}
\end{figure}

\subsection{Heavy-tailed fluctuations and global occupation}
The nonlinear nongradient model~\eqref{eq:fourwell} compares matched Gaussian and $t_3$ innovations across three steps and 96 paths, labeling all nine roots.
At $a=0.005$, median stable root-centered trace ratios are $1.015$ and $1.010$, respectively (Figure~\ref{fig:root-moments}b).
Figure~\ref{fig:finite-second} reports the broader $t_3$ dispersion and all-root truncated moments; Figure~\ref{fig:symmetric-endpoint} gives the separate analytic smooth-test comparison.

Global communication supplies occupation information in the separate scalar model $g(x)=x+5\cos x$.
Long paths give 388 and 413 transitions at $a=0.04$ and $0.02$, with occupation and directional rates agreeing with the heavy-tail proxy at the measurement resolution (Appendix~\ref{app:long-communication}).

\section{Discussion}\label{sec:discussion}
Stationary energy matching turns local Gaussian limits into reliable moment asymptotics across multiple equilibria.
It gives covariance, loss and mean-shift formulas a common foundation under second innovation moments, separating local coefficients from global occupation weights.

The heterogeneous-root comparison links this distinction directly to transport: even the best common covariance retains a positive limiting same-label cost.
In neural SGD, matched OU paths calibrate finite-trajectory variability, while step-size and batch-size comparisons identify departures from the local moment predictions.

The present theory concerns stationary laws near finite hyperbolic roots.
Extending energy matching to equilibrium manifolds would connect normal moment control with tangential dynamics \citep{LiWangArora2022}.
Quantitative $W_2$ bounds and equilibration estimates would connect these stationary predictions to prescribed training horizons.

\section*{Reproducibility statement}
The appendices contain complete proofs, model-specific certificates and experimental protocols.

\section*{AI use statement}
In this work, we used generative AI tools to assist with manuscript writing and editing, literature search and identification, drafting portions of the manuscript, experimental design and code development, exploration of new research directions based on existing results, and the development and verification of mathematical arguments and proofs.

We have reviewed all AI-assisted work. Suggested references were checked against original sources, experimental code and numerical results were executed and verified by the authors, and mathematical claims, derivations, and proofs were independently checked for correctness.

We take responsibility for the final content of this work, including all text, claims, proofs, code, experiments, and other artifacts produced with the aid of generative AI.
\bibliography{references}
\bibliographystyle{iclr2027_conference}
\clearpage
\appendix
\section*{Guide to the appendices}
The proofs follow the dependence of the results: concentration and stationary innovation tails precede local OU limits, energy matching, transport and mean expansions.
The experimental appendices give the model certificates and measurement protocols behind the main figures.

\begin{center}\small
\begin{tabular}{p{0.69\linewidth}p{0.23\linewidth}}
\toprule
Mathematical question & Location\\
\midrule
Existence, concentration and stationary innovation tails & Appendix~\ref{app:concentration}\\
Finite OU measures and generalized unstable directions & Appendix~\ref{app:ou}\\
Root localization, energy matching and Gaussian transport & Appendices~\ref{app:mixture}--\ref{app:energy}\\
Rootwise mean shifts and global observable variance & Appendix~\ref{app:bias}\\
Smooth-test rates, unstable mass and the symmetric $t_3$ endpoint & Appendices~\ref{app:rate}--\ref{app:symmetric}\\
Observable and loss consequences & Appendix~\ref{app:observables}\\
Exact saddle components, selected weights and local equilibration & Appendices~\ref{app:unexcited}--\ref{app:counter}\\
Root continuity without fixed-neighborhood uniform integrability & Appendix~\ref{app:root-continuity-example}\\
\midrule
Experimental question & Location\\
\midrule
Same-label invariant-law transport and numerical resolution & Appendix~\ref{app:transport-study}\\
Structural assumptions for the augmented neural model & Appendix~\ref{app:neural}\\
Recursions, deterministic references and independent replication & Appendix~\ref{app:revision-numerics}\\
OU measurement references, neural augmentation and observation windows & Appendix~\ref{app:moment-studies}\\
Independent mean measurements and drift decomposition & Appendix~\ref{app:mean-measurements}\\
Neural dimension and data comparisons & Appendix~\ref{app:dimension}\\
Occupation and directional transition rates at long horizons & Appendix~\ref{app:long-communication}\\
\bottomrule
\end{tabular}
\end{center}

\section{Existence, energy and localization}\label{app:concentration}

We prove Theorem~\ref{thm:concentration}. Constants $K$ may change between lines; all bounds are uniform over invariant laws and sufficiently small $a$.
Let $V(x)=\norm{x}^2$. Conditional centering gives the exact identity
\begin{equation}\label{eq:foster-exact}
P_aV(x)-V(x)=-2a\ip{x}{g(x)}+a^2\{\norm{g(x)}^2+\tr C(x)\}.
\end{equation}
Global Lipschitz continuity implies $\norm{g(x)}^2\le K(1+\norm{x}^2)$.
Together with~\eqref{eq:cov-growth} and~\eqref{eq:dissipativity}, this gives, after reducing the step-size upper bound,
\begin{equation}\label{eq:foster}
P_aV(x)\le(1-\kappa a)V(x)+Ka.
\end{equation}
Starting at any fixed point, iteration yields $\sup_n\E V(X_n)\le V(X_0)+K/\kappa$.
The Ces\`aro averages of the laws are tight.
Any weak limit is invariant: for bounded continuous $f$, the telescoping difference between the average law and its image under $P_a$ tends to zero, and the Feller property permits passage to the weak limit.
Lower semicontinuity of $V$ gives a finite second moment.

To cover every invariant probability, we use a bounded truncation before integrating the drift inequality.
\begin{lemma}[Truncated Foster bound]\label{lem:truncated-foster}
Let $V$ be finite and nonnegative. If $PV\le\rho V+b$ with $0<\rho<1$ and $0\le b<\infty$, every invariant probability satisfies $\pi V\le b/(1-\rho)$.
\end{lemma}
\begin{proof}
Set $V_N=V\wedge N$. Concavity and monotonicity of truncation give
$PV_N\le(PV)\wedge N\le(\rho V+b)\wedge N$.
Invariance for the bounded $V_N$ implies
\[
\int\{V\wedge N-(\rho V+b)\wedge N\}\,d\pi\le0.
\]
The integrand plus $b$ is nonnegative and converges pointwise to $(1-\rho)V$.
Fatou's lemma yields $(1-\rho)\pi V\le b$, without assuming integrability in advance.
\end{proof}
Apply Lemma~\ref{lem:truncated-foster} to~\eqref{eq:foster}; every invariant probability has $\E V(X)\le K/\kappa$.

Write $\Delta=X^+-X=a\{\xi(X,\zeta_1)-g(X)\}$.
Taylor's upper bound for $U$ and conditional centering imply
\[
P_aU(x)-U(x)\le -ac_U\norm{g(x)}^2+\tfrac12L_Ua^2\{c_0+(1+c_1)\norm{g(x)}^2\}.
\]
All terms are integrable: $U$ has quadratic growth, $\nabla U$ has at most linear growth, and the covariance and drift bounds control the second increment moment.
Stationarity therefore gives, for $2c_U>L_U(1+c_1)a$,
\begin{equation}\label{eq:energy-explicit}
\E\norm{g(X)}^2\le\frac{L_Uc_0a}{2c_U-L_U(1+c_1)a}.
\end{equation}
It also proves
\begin{equation}\label{eq:increment-second}
\E\norm{\Delta}^2=a^2\E\{\tr C(X)+\norm{g(X)}^2\}\le Ka^2.
\end{equation}

We next justify the global error bound used in the main text.
Since $A_j$ is invertible, Taylor expansion and its smallest singular value give $\norm{g(x)}\ge c_j\norm{x-m_j}$ in a sufficiently small ball about $m_j$.
Outside a large ball, Cauchy--Schwarz and dissipativity give
$\norm{g(x)}\ge\kappa\norm{x}-b/\norm{x}\ge(\kappa/2)\norm{x}$.
There $\dist(x,\Mcal)\le\norm{x}+\max_j\norm{m_j}\le K\norm{x}$.
On the remaining compact set separated from every root, the continuous positive ratio $\norm{g(x)}/\dist(x,\Mcal)$ has a positive minimum.
Combining these three regions gives
\begin{equation}\label{eq:global-error}
\norm{g(x)}\ge c_g\dist(x,\Mcal)\qquad(x\in\R^d).
\end{equation}
The last estimate in~\eqref{eq:concentration} follows.
In particular, every measurable set at distance at least $r>0$ from $\Mcal$ has stationary mass at most $Ka/r^2$.

\subsection{Root continuity and stationary innovation tails}\label{app:stationary-tails}
The concentration estimates above control the mass away from the root set.
We next combine them with continuity at the roots to control large innovations under stationary sampling.
The argument precedes, and does not use, uniform integrability of the rescaled displacement.

\begin{lemma}[Innovation laws near the roots]\label{lem:root-noise-continuity}
Under Assumption~\ref{ass:standing}, as $x\to m_j$, the law of $\xi(x,\zeta_1)$ converges weakly to that of $\xi(m_j,\zeta_1)$ and its second moment converges.
In particular,
\begin{equation}\label{eq:shrinking-noise-tails}
\lim_{R\to\infty}\lim_{r\downarrow0}
\sup_{\dist(x,\Mcal)\le r}
\E[\norm{\xi(x,\zeta_1)}^2\one\{\norm{\xi(x,\zeta_1)}>R\}]=0.
\end{equation}
\end{lemma}
\begin{proof}
Fix a sufficiently small $a_*>0$ for which $P_{a_*}$ is Feller.
For any sequence $x_n\to m_j$, the next-state laws $P_{a_*}(x_n,\cdot)$ converge weakly to $P_{a_*}(m_j,\cdot)$.
If Feller is defined on $C_0$, it first gives vague convergence; the probability limit then gives weak convergence and tightness.
The maps
\[
F_n(y)=(y-x_n)/a_*+g(x_n)
\]
converge uniformly on compact sets to $F(y)=(y-m_j)/a_*+g(m_j)$.
Tightness and this local uniform convergence imply convergence of the image laws under $F_n$ to the image law under $F$.
These are precisely the innovation laws, because the mechanism $\xi$ is independent of the step size.
Continuity of $C$ at $m_j$ gives convergence of their second moments to $\tr C_j$.

For completeness, set $h_R(u)=\min\{\norm u^2,R^2\}$.
This is bounded and continuous, and
\[
\E[\norm\xi^2\one\{\norm\xi>\sqrt2R\}]
\le2\{\tr C(x)-\E h_R(\xi(x,\zeta_1))\}.
\]
Both terms in braces are continuous at $m_j$.
Their difference there is $\E(\norm{\xi(m_j,\zeta_1)}^2-R^2)_+$, which tends to zero as $R\to\infty$.
Taking the finite maximum over roots proves~\eqref{eq:shrinking-noise-tails}.
The order of the two limits is part of the statement.
\end{proof}

\begin{lemma}[Stationary innovation and increment tails]\label{lem:stationary-noise-tails}\label{lem:increment-tail}
Under Assumption~\ref{ass:standing},
\begin{equation}\label{eq:stationary-noise-ui}
\lim_{R\to\infty}\limsup_{a\downarrow0}\sup_{\pi\in\mathcal I_a}
\E_\pi[\norm{\xi(X,\zeta_1)}^2\one\{\norm{\xi(X,\zeta_1)}>R\}]=0.
\end{equation}
The same assertion holds with $Y_a=\xi(X,\zeta_1)-g(X)$ in place of $\xi(X,\zeta_1)$.
Consequently, for each fixed $\rho>0$ and $\Delta=aY_a$,
\begin{equation}\label{eq:increment-tail-refined}
\E[\norm\Delta^2\one\{\norm\Delta>\rho\}]=o(a^2)
\end{equation}
uniformly over invariant-law selections.
\end{lemma}
\begin{proof}
Let $K_r=\{x:\dist(x,\Mcal)\le r\}$.
The covariance bound and concentration imply
\begin{align*}
\E_\pi[\norm\xi^2\one_{K_r^c}(X)]
&=\E_\pi[\tr C(X)\one_{K_r^c}(X)]\\
&\le c_0\pi(K_r^c)+c_1\E_\pi\norm{g(X)}^2
\le Ka/r^2+Ka.
\end{align*}
Inside $K_r$, the tail expectation is bounded by the supremum in~\eqref{eq:shrinking-noise-tails}.
For fixed $R,r$, take $\limsup_{a\downarrow0}\sup_{\pi\in\mathcal I_a}$, then send $r\downarrow0$, and finally $R\to\infty$.
This proves~\eqref{eq:stationary-noise-ui} in its stated order.

For the increment, use the pointwise inequality
\begin{align*}
\norm{u-v}^2\one\{\norm{u-v}>R\}
\le{}&4\norm u^2\one\{\norm u>R/2\}\\
&+4\norm v^2\one\{\norm v>R/2\}.
\end{align*}
The innovation term has just been controlled, and the drift contribution is bounded by $4\E\norm{g(X)}^2=O(a)$.
Thus the same asymptotic tail assertion holds for $Y_a$.
At the diverging threshold $\rho/a$, it gives
$\E[\norm{Y_a}^2\one\{\norm{Y_a}>\rho/a\}]=o(1)$.
Multiplication by $a^2$ proves~\eqref{eq:increment-tail-refined}.
\end{proof}

These are tail statements under stationary sampling as $a\downarrow0$.
Appendix~\ref{app:root-continuity-example} shows that this stationary control is strictly weaker than uniform integrability on a fixed root neighborhood.

\section{Identification of finite stationary OU measures}\label{app:ou}

\begin{lemma}\label{lem:OU}
Let $C$ be symmetric nonnegative and let $\nu$ be a finite nonnegative measure with $\int\norm{z}^2\nu(\dd z)<\infty$.
Suppose
\begin{equation}\label{eq:OU-equation}
\int\mathcal L f\dd\nu=0\quad(f\in C_c^2(\R^d)),\qquad
\mathcal L f=-\ip{Az}{\nabla f}+\tfrac12\tr(CD^2f).
\end{equation}
Write $q=\nu(\R^d)$ and $M=\int zz^\trans\nu(\dd z)$.
Then $AM+MA^\trans=qC$.
If $A$ is positive stable, $\nu=q\N(0,S)$ with $AS+SA^\trans=C$.
If $A^\trans v=\lambda v$, $\Re\lambda<0$ and $v^*Cv>0$, then $q=0$.
\end{lemma}

\begin{proof}
Choose a smooth cutoff $\psi_R$ equal to one on $B(0,R)$ and zero outside $B(0,2R)$, with derivative bounds $\norm{D^k\psi_R}_\infty\le K R^{-k}$ for $k=1,2$.
Apply the stationary equation to $z_rz_s\psi_R(z)$.
The additional derivative terms are supported on $R\le\norm{z}\le2R$ and bounded by $K(1+\norm{z}^2)$ there, uniformly for $R\ge1$.
Their integrals vanish as $R\to\infty$ by the finite second moment.
The remaining terms converge by dominated convergence and give $AM+MA^\trans=qC$.
Hence
\[
2\Re\lambda\,v^*Mv=qv^*Cv.
\]
The left side is nonpositive and the right side is strictly positive if $q>0$, proving the last assertion, including complex eigendirections.

For the stable case, define $\phi(t)=\int e^{it^\trans z}\nu(\dd z)$.
Apply the stationary equation to the real and imaginary parts of $e^{it^\trans z}\psi_R(z)$.
Cutoff terms are bounded by $K_t(1+\norm{z})$ on the same annulus; passage to the limit is again justified.
Differentiation of $\phi$ is valid by its finite first moment, and the equation becomes
\[
(A^\trans t)\cdot\nabla\phi(t)=-\tfrac12t^\trans Ct\,\phi(t).
\]
For $t_s=e^{-A^\trans s}t$, it follows that
$\frac{\mathrm d}{\mathrm ds}\phi(t_s)=\tfrac12t_s^\trans Ct_s\,\phi(t_s)$.
Solving this scalar linear equation, without dividing by $\phi$, gives
\[
\phi(t)=\phi(t_s)\exp\left\{-\tfrac12\int_0^s t^\trans e^{-Au}Ce^{-A^\trans u}t\dd u\right\}.
\]
Stability implies $t_s\to0$ and $\phi(t_s)\to q$.
Thus $\phi(t)=q\exp(-t^\trans St/2)$.
Uniqueness of Fourier transforms of finite measures completes the proof; no nondegeneracy of $C$ is required.
\end{proof}

\begin{lemma}[Hyperbolic OU classification]\label{lem:hyperbolic-OU}
Let $A$ be hyperbolic and let $\nu$ satisfy Lemma~\ref{lem:OU}, with mass $q$.
Define $P^s,P^u$ and $\bar S$ as in~\eqref{eq:all-covariance}.
If $q>0$, then $P^uC(P^u)^T=0$ and $\nu=q\N(0,\bar S)$.
Conversely, this Gaussian is stationary whenever $P^uC(P^u)^T=0$.
In particular, $P^uC(P^u)^T\ne0$ forces $q=0$.
\end{lemma}
\begin{proof}
Choose a real invertible matrix $T$ such that
$T^{-1}AT=\diag(A_s,A_u)$, with spectra in the positive and negative half-planes, respectively.
This is a splitting into generalized invariant subspaces and does not presume diagonalizability.
Transform the measure by $T^{-1}$ and set
$\widetilde C=T^{-1}CT^{-T}$ and $\widetilde M=T^{-1}MT^{-T}$.
These are congruences, so both matrices remain nonnegative.
The moment identity in Lemma~\ref{lem:OU} gives
\[
A_u\widetilde M_{uu}+\widetilde M_{uu}A_u^T=q\widetilde C_{uu}.
\]
Since $A_u$ has negative-real-part spectrum,
$H=\int_0^\infty e^{A_u^Tt}e^{A_ut}\dd t$ is positive definite and satisfies $A_u^TH+HA_u=-I$.
Taking a trace pairing yields
\[
-\tr\widetilde M_{uu}=q\tr(H\widetilde C_{uu}).
\]
For $q>0$, the opposite signs and
$\tr(H\widetilde C_{uu})\ge\lambda_{\min}(H)\tr\widetilde C_{uu}$ imply
$\widetilde M_{uu}=\widetilde C_{uu}=0$.
Nonnegativity of $\widetilde C$ forces $\widetilde C_{su}=0$: evaluate its quadratic form at $(x,ty)$ for every real $t$.
Also $\int\norm{z_u}^2\dd\widetilde\nu=0$, so $\widetilde\nu$ is supported on $z_u=0$.

To obtain the stationary equation for its stable marginal, multiply a compactly supported stable test by a compact unstable cutoff equal to one near zero.
The unstable derivative terms vanish on the support of the measure.
The stable case of Lemma~\ref{lem:OU} identifies this marginal as the Gaussian with covariance
$\int_0^\infty e^{-A_st}\widetilde C_{ss}e^{-A_s^Tt}\dd t$.
Transforming back gives $q\N(0,\bar S)$.
If the stable subspace is empty, the measure is $q\delta_0$; if the unstable subspace is empty, the conclusion is precisely the stable OU lemma.
If $q=0$, nonnegativity gives $\nu=0$.

Finally, positivity and $P^uC(P^u)^T=0$ imply $P^uC=C(P^u)^T=0$.
Thus $P^sC(P^s)^T=C$, and the Gaussian supported on the stable subspace indeed solves the full stationary equation.
Without this compatibility condition, $A\bar S+\bar SA^T=P^sC(P^s)^T$ still holds, but its right side need not equal $C$.
\end{proof}

The projection condition can exclude roots not covered by eigenvector excitation.
For $A=\left(\begin{smallmatrix}-1&1\\0&-1\end{smallmatrix}\right)$ and $C=\diag(1,0)$,
every left eigenvector is proportional to $(0,1)^T$, whereas $P^uC(P^u)^T=C\ne0$.
The classification therefore excludes positive local mass even though~\eqref{eq:eigen-excitation} fails.
This matrix example isolates the local excitation criterion; global confinement is supplied separately by Assumption~\ref{ass:standing}.

\section{Localization at all roots}\label{app:mixture}
\begin{lemma}[All-root local limits]\label{lem:all-root}
Under Assumption~\ref{ass:standing}, every sequence $a\downarrow0$ and every invariant-law selection admit a subsequence such that
\begin{equation}\label{eq:all-root-limits}
(J_a,Z_a)\weak\sum_{j=1}^{\ell}\delta_j\otimes\nu_j,
\quad q_j:=\nu_j(\R^d)=\lim w_{a,j},\quad\sum_jq_j=1.
\end{equation}
The measures have finite second moments $M_j=\int zz^T\nu_j(dz)$ and satisfy
$A_jM_j+M_jA_j^T=q_jC_j$. At every root, $\nu_j=q_j\bar\Gamma_j$, with $q_j=0$ whenever $P_j^uC_j(P_j^u)^T\ne0$.
\end{lemma}
\begin{lemma}[Marked weak approximation]\label{lem:weak}
Under Assumptions~\ref{ass:standing} and~\ref{ass:excitation},
\begin{equation}\label{eq:mixture}
\sum_{j\in\Ucal}w_{a,j}\to0,\qquad
d_{\BL}\!\left(\Law(X_{\Mcal},Z_a),\sum_{j\in\Scal}w_{a,j}(\delta_{m_j}\otimes\Gamma_j)\right)\to0.
\end{equation}
Along every subsequence with convergent weights,
\begin{equation}\label{eq:weak-limit}
(J_a,Z_a)\weak\sum_{j\in\Scal}w_j(\delta_j\otimes\Gamma_j).
\end{equation}
\end{lemma}
We first prove Lemma~\ref{lem:all-root}, and then apply excitation to obtain Lemma~\ref{lem:weak}.
For finite measures define $d_{\BL}(\mu,\nu)=\sup_{\norm{h}_\infty+\operatorname{Lip}(h)\le1}|\mu h-\nu h|$. Choose $r>0$ sufficiently small that the closed balls $\overline B(m_j,4r)$ are disjoint and have their nearest root equal to $m_j$.
Choose $\chi_j\in C_c^\infty$ with $0\le\chi_j\le1$, equal to one on $B(m_j,r)$ and zero outside $B(m_j,2r)$.
Define the finite localized measure and its mass by
\begin{equation}\label{eq:localized}
\nu_{a,j}(D)=\E[\chi_j(X)\one\{(X-m_j)/\sqrt a\in D\}],\qquad q_{a,j}=\E\chi_j(X).
\end{equation}
Theorem~\ref{thm:concentration} gives
\begin{equation}\label{eq:cutoff-moments}
\int\norm{z}^2\nu_{a,j}(\dd z)\le K,\qquad
0\le w_{a,j}-q_{a,j}\le Ka.
\end{equation}
Indeed, on the cutoff support $\norm{X-m_j}=\dist(X,\Mcal)$; the difference between the cell indicator and cutoff is supported where $\dist(X,\Mcal)\ge r$.
These finite measures are uniformly tight.

\paragraph{Cutoff change.}
Let $D_j=\supp\nabla\chi_j$ and choose $\rho>0$ so the closed $\rho$-enlargement $D_j^\rho$ is compact and separated from all roots.
On $\{\norm{\Delta}\le\rho\}$ a nonzero cutoff change implies $X\in D_j^\rho$, by the fundamental theorem of calculus along the increment segment.
On this compact set,
$\E[\norm{\Delta}\mid X=x]\le a(\norm{g(x)}+\sqrt{\tr C(x)})\le Ka$.
The small-increment contribution to $\E|\chi_j(X^+)-\chi_j(X)|$ is at most
$Ka\pi_a(D_j^\rho)=O(a^2)$.
The large-increment contribution is at most
$\Pp(\norm{\Delta}>\rho)\le\rho^{-2}\E\norm{\Delta}^2=O(a^2)$.
Consequently,
\begin{equation}\label{eq:cutoff-change}
\E|\chi_j(X^+)-\chi_j(X)|\le Ka^2.
\end{equation}

\paragraph{Localized generator equation.}
Fix $f\in C_c^2$ and abbreviate $Z=(X-m_j)/\sqrt a$, $H=\Delta/\sqrt a$.
Stationarity of $\chi_j(x)f((x-m_j)/\sqrt a)$ and~\eqref{eq:cutoff-change} give
\begin{equation}\label{eq:local-stationarity}
a^{-1}\E\{\chi_j(X)[f(Z+H)-f(Z)]\}=O(a).
\end{equation}
Let $\omega_f$ be a bounded modulus of continuity of $D^2f$, tending to zero at zero.
Taylor's formula gives
\[
f(Z+H)-f(Z)=\nabla f(Z)^\trans H+\tfrac12H^\trans D^2f(Z)H+R,
\qquad |R|\le K\norm{H}^2\omega_f(\norm{H}).
\]
Write $Y_a=\xi(X,\zeta_1)-g(X)$, so $H=\sqrt aY_a$.
The divided Taylor remainder is bounded by
\[
a^{-1}\E[\chi_j|R|]\le K\E[\chi_j(X)\norm{Y_a}^2\omega_f(\sqrt a\norm{Y_a})].
\]
On $\norm{Y_a}\le R_0$, it tends to zero uniformly as $a\downarrow0$, using bounded second moments.
On the complement, boundedness of $\omega_f$ and Lemma~\ref{lem:stationary-noise-tails} control the expectation uniformly.
Let $a\downarrow0$ and then $R_0\to\infty$ to obtain $a^{-1}\E[\chi_j|R|]\to0$.

Conditional centering gives $\E[H\mid X]=-\sqrt a g(X)$ and
$\E[HH^\trans\mid X]=a\{C(X)+g(X)g(X)^\trans\}$.
The drift-squared contribution is $O(a)$ by~\eqref{eq:concentration}.
Local $C^2$ regularity and~\eqref{eq:cutoff-moments} give
\begin{equation}\label{eq:linearization}
\E\left[\chi_j(X)\norm{g(X)/\sqrt a-A_jZ}\right]
\le K\sqrt a\E[\chi_j(X)\norm{Z}^2]\le K\sqrt a.
\end{equation}
Finally, $C$ is bounded on the cutoff support by~\eqref{eq:cov-growth}.
Its continuity at $m_j$, together with concentration outside smaller root balls, implies
$\E[\chi_j(X)\norm{C(X)-C_j}]\to0$.
Substitution in~\eqref{eq:local-stationarity} yields
\[
\int\left[-\ip{A_jz}{\nabla f(z)}+\tfrac12\tr(C_jD^2f(z))\right]\nu_{a,j}(\dd z)\longrightarrow0.
\]

\paragraph{Identification and removal of subsequences.}
Along any sequence $a_n\downarrow0$, extract a common subsequence on which all finite measures $\nu_{a_n,j}$ converge weakly.
Their limits $\nu_j$ have finite second moment by lower semicontinuity, and their masses are the limits of $q_{a_n,j}$.
The integrand in the last display is bounded continuous because $f$ has compact support, so each limit satisfies~\eqref{eq:OU-equation}.
Lemma~\ref{lem:OU} gives the matrix identity for every root; Lemma~\ref{lem:hyperbolic-OU} identifies $\nu_j=q_j\bar\Gamma_j$.
The total lost cutoff mass is $O(a)$, so $\sum_jq_j=1$ and $q_j=\lim w_{a,j}$.
Replacing cutoffs by cells costs $O(a)$ for bounded tests, proving~\eqref{eq:all-root-limits} and Lemma~\ref{lem:all-root}.

Under excitation, Lemma~\ref{lem:hyperbolic-OU} gives $q_j=0$ for every unstable root.
The marked law and the stable comparison measure then have the same subsequential weak limit.
The bounded-Lipschitz metric metrizes weak convergence for these tight finite measures.
If~\eqref{eq:mixture} failed uniformly, a violating sequence of step sizes and invariant laws would have such a subsequence, a contradiction.
This proves Lemma~\ref{lem:weak}.

\section{Exact energy, uniform integrability and \texorpdfstring{$W_2$}{W2}}\label{app:energy}
Write $D=\langle\nabla U,g\rangle$ and $E_j=D^2U(m_j)$. The preceding concentration theorem also gives $\E\norm{Z_a}^2\le K$.
\begin{lemma}[Quadratic behavior of the descent dissipation]\label{lem:D}
At every root, $\nabla U(m_j)=0$, and, as $y\to0$,
\begin{equation}\label{eq:Dlocal}
 D(m_j+y)=y^TE_jA_jy+o(\norm y^2).
\end{equation}
Moreover $B_j=(E_jA_j+A_j^TE_j)/2\succeq c_UA_j^TA_j\succ0$.
\end{lemma}
\begin{proof}
Since $D\ge0$ and $D(m_j)=0$, differentiability gives $\nabla D(m_j)=0$. The product rule gives $\nabla D(m_j)=A_j^T\nabla U(m_j)$, because $g(m_j)=0$. Hyperbolicity implies invertibility of $A_j$, hence $\nabla U(m_j)=0$. Now $\nabla U(m_j+y)=E_jy+o(\norm y)$ and $g(m_j+y)=A_jy+O(\norm y^2)$. Multiplication proves \eqref{eq:Dlocal}. Applying \eqref{eq:lyapunov} at $m_j+ty$ and dividing by $t^2$ proves the matrix inequality. In particular this argument does not presume that $E_j$ is positive definite at unstable roots.
\end{proof}

\begin{lemma}[Exact stationary energy asymptotics]\label{lem:energy}
Under Assumption~\ref{ass:standing}, without excitation,
\begin{equation}\label{eq:energy}
 \frac{\E D(X)}a=\frac12\sum_{j=1}^{\ell}w_{a,j}\tr(E_jC(m_j))+o(1).
\end{equation}
The remainder is uniform over selected invariant laws.
\end{lemma}
\begin{proof}
Put $Y=\xi(X,\zeta)-g(X)$, so $\Delta=aY$. Taylor expansion with an integral remainder gives
\[
 U(X+aY)-U(X)=a\nabla U(X)^TY+
 \frac{a^2}{2}Y^TD^2U(X)Y+a^2 r_a(X,Y),
\]
where, with $L_U=\sup_x\norm{D^2U(x)}$,
\begin{align*}
 |r_a(x,y)|&\le \norm y^2\int_0^1(1-t)
 \norm{D^2U(x+tay)-D^2U(x)}\,dt,\\
 |r_a(x,y)|&\le L_U\norm y^2.
\end{align*}
We show $\E|r_a(X,Y)|=o(1)$; this is the step for which a bare second-moment bound would not suffice.

Choose a fixed compact union $K_0$ of small closed root balls.
For any fixed $R$, continuity of $D^2U$ gives uniform continuity on a compact enlargement of $K_0$, and hence
\[
\sup_{x\in K_0,\,\norm y\le R}|r_a(x,y)|\longrightarrow0.
\]
The expectation on $\norm Y>R$ is bounded by
$L_U\E[\norm Y^2\one\{\norm Y>R\}]$.
Lemma~\ref{lem:stationary-noise-tails} makes this contribution vanish after $a\downarrow0$ and then $R\to\infty$.
Thus the required remainder limit holds inside $K_0$.

Outside $K_0$, concentration gives $\pi_a(K_0^c)=O(a)$. Conditional centering and the covariance growth assumption imply
\begin{align*}
 \E[\norm Y^2\one_{\{X\notin K_0\}}]
 &=\E[\{\tr C(X)+\norm{g(X)}^2\}\one_{\{X\notin K_0\}}]\\
 &\le c_0\pi_a(K_0^c)+(1+c_1)\E\norm{g(X)}^2=O(a).
\end{align*}
The global bounded-Hessian estimate therefore makes the outside remainder vanish as well.

All terms in the Taylor identity are integrable: $U$, $\nabla U^Tg$, and $g^2$ have at most quadratic growth and the covariance is controlled. Stationarity and conditional centering consequently give
\[
 \frac{\E D(X)}a
 =\frac12\E\tr\bigl[D^2U(X)\{C(X)+g(X)g(X)^T\}\bigr]+o(1).
\]
The drift-squared term is $O(a)$ by \eqref{eq:energy-explicit}. To replace the covariance term by root values, first restrict to small root balls. There, continuity of $D^2U$ and of $C$ at each root controls the error as the radii decrease. Outside these balls its absolute expectation is bounded by a constant times
$c_0\pi_a(\text{outside})+c_1\E\norm{g(X)}^2=O(a)$ for each fixed radius. The difference between ball mass and Voronoi mass is also $O(a)$ for fixed radii. Taking $a\to0$ and then shrinking the radii proves \eqref{eq:energy}. All estimates are law-uniform.
\end{proof}

\begin{proof}[Proof of Theorem~\ref{thm:local-moments}]
Take any sequence of steps and invariant laws, and extract the common all-root subsequence of Lemma~\ref{lem:all-root}.
Write $R_a=D(X)/a$. Lemma~\ref{lem:D} identifies $R_a\Rightarrow R$, where $R$ is the image of the marked limit under $(j,z)\mapsto z^TE_jA_jz$.
Indeed, on $\norm{Z_a}\le R_0$ the divided Taylor error is uniformly $o(1)$, and tightness controls the complement.
Both variables are nonnegative. Symmetry of $E_j,M_j$ and the matrix moment equation imply
\[
\E R=\sum_j\tr(E_jA_jM_j)
=\tfrac12\sum_j\tr\{E_j(A_jM_j+M_jA_j^T)\}
=\tfrac12\sum_jq_j\tr(E_jC_j).
\]
The two traces in the middle sum agree by transposition and cyclicity.
This identity applies at unstable roots without identifying their measures as Gaussians.
Lemma~\ref{lem:energy} and $w_{a,j}\to q_j$ give $\E R_a\to\E R$.

For fixed $T>0$, bounded continuous truncation gives
\[
\E(R_a-T)_+=\E R_a-\E(R_a\wedge T)\longrightarrow\E(R-T)_+.
\]
Since $r\one\{r>2T\}\le2(r-T)_+$, the variables $R_a$ are asymptotically uniformly integrable along this subsequence.
The domination $R_a\ge c_Uc_g^2\norm{Z_a}^2$ transfers the property to $\norm{Z_a}^2$.
If~\eqref{eq:square-ui} failed, there would be $a_n\downarrow0$, invariant laws $\pi_n$ and $R_n\to\infty$ with tail expectations bounded away from zero.
The preceding subsequence argument contradicts this, proving the law-uniform assertion.

At every root, Lemma~\ref{lem:hyperbolic-OU} gives $\nu_j=q_j\bar\Gamma_j$; incompatible roots have $q_j=0$.
Uniform integrability permits linear and quadratic tests multiplied by the discrete root label, proving~\eqref{eq:second} on every extracted subsequence and hence on the full family.
For weights bounded below, division gives conditional weak convergence and square uniform integrability.
The weak-plus-second-moment criterion for $W_2$ proves~\eqref{eq:conditional}; subtracting conditional means gives its covariance assertion.
\end{proof}

\paragraph{Coupling at a common root.}
We prove the same-label assertion of Theorem~\ref{thm:W2}:
\begin{equation}\label{eq:same-label}
\sup_{\pi\in\mathcal I_a}\sum_{j:w_{a,j}>0}w_{a,j}
W_2^2(\Law(Z_a\mid J_a=j),\bar\Gamma_j)\longrightarrow0.
\end{equation}
For the marked cost $d((j,z),(k,y))^2=\one\{j\ne k\}+\norm{z-y}^2$, it implies
\begin{equation}\label{eq:W2}
W_2\!\left(\Law(J_a,Z_a),\sum_jw_{a,j}(\delta_j\otimes\bar\Gamma_j)\right)\to0
\end{equation}
uniformly over invariant-law selections.

\begin{proof}[Proof of Theorem~\ref{thm:W2}]
Take any sequence $a_n\downarrow0$ and invariant laws $\pi_n$.
Extract a subsequence on which all weights converge to $q_j$ and the marked laws have the limits in Lemma~\ref{lem:all-root}.
Lemma~\ref{lem:hyperbolic-OU} identifies each limit as $q_j\bar\Gamma_j$.
If $q_j>0$, conditional weak convergence and the square uniform integrability in Theorem~\ref{thm:local-moments} yield conditional $W_2$ convergence \citep{Villani2009}.
If $q_j=0$, an independent coupling gives, whenever $w_{a,j}>0$,
\[
w_{a,j}W_2^2(\Law(Z_a\mid J_a=j),\bar\Gamma_j)
\le2\E[\one\{J_a=j\}\norm{Z_a}^2]+2w_{a,j}\tr\bar S_j.
\]
For any fixed $R$, the first expectation is bounded by
$R^2w_{a,j}+\E[\norm{Z_a}^2\one\{\norm{Z_a}>R\}]$.
First let $n\to\infty$, then $R\to\infty$. The bound tends to zero by~\eqref{eq:square-ui}.
Terms with $w_{a,j}=0$ contribute zero without defining a conditional law.
Summing over the finite root set proves~\eqref{eq:same-label} along the extracted subsequence.
Failure of uniform convergence would give a violating sequence with such a subsequence, a contradiction.

Choose optimal conditional couplings of $Z_a$ and $G_j\sim\bar\Gamma_j$, and mix them with the exact weights $w_{a,j}$.
Matching labels gives zero label cost and proves~\eqref{eq:W2}.
Mapping each pair to $(m_j+\sqrt aZ_a,m_j+\sqrt aG_j)$ produces marginals $\pi_a$ and $\bar G_{a,\pi}$.
Its squared cost is $a$ times the sum in~\eqref{eq:same-label}, proving~\eqref{eq:original-W2}.
For an incompatible root, positive limiting weight would contradict Lemma~\ref{lem:hyperbolic-OU}; the same subsequence argument makes the exclusion uniform.
\end{proof}

\begin{proof}[Proof of Corollary~\ref{cor:stable-mixture}]
Excitation and Lemma~\ref{lem:hyperbolic-OU} exclude every unstable component, so $s_a\to1$ uniformly.
Along any weight-convergent subsequence, the marked law and $Q_a$ have the same stable Gaussian weak limit.
Equation~\eqref{eq:square-ui} upgrades the first convergence to $W_2$; the finite Gaussian mixtures $Q_a$ also converge in $W_2$.
The triangle inequality and the subsequence argument prove~\eqref{eq:target}.
For each unstable root and fixed $R$, its second-moment contribution is at most
$R^2w_{a,j}+\E[\norm{Z_a}^2\one\{\norm{Z_a}>R\}]$.
Letting $a\to0$ and then $R\to\infty$ proves the last assertion.
\end{proof}

\section{Rootwise stationary means}\label{app:bias}
We prove Theorem~\ref{thm:bias}. All estimates in this section use Assumption~\ref{ass:standing} alone and are uniform over invariant-law selections.

\begin{lemma}[Fixed-neighborhood tails]\label{lem:outside-refined}
For every fixed $r>0$,
\[
\E[\dist(X,\Mcal)^2\one\{\dist(X,\Mcal)\ge r\}]=o(a),\qquad
\Pp\{\dist(X,\Mcal)\ge r\}=o(a).
\]
\end{lemma}
\begin{proof}
The first expectation divided by $a$ is
$\E[\norm{Z_a}^2\one\{\norm{Z_a}\ge r/\sqrt a\}]$, which tends to zero by~\eqref{eq:square-ui}.
The probability is at most the first expectation divided by $r^2$.
\end{proof}

\begin{lemma}[Original-scale cutoff flux]\label{lem:linear-flux}
For the cutoffs in Appendix~\ref{app:mixture},
\begin{equation}\label{eq:linear-flux}
\E[|\chi_j(X^+)-\chi_j(X)|(1+\norm{X^+-m_j})]=o(a^2).
\end{equation}
Moreover, $\E[\one_{\mathcal C_j}(X)g(X)]=o(a)$ for every root.
\end{lemma}
\begin{proof}
Choose $\rho>0$ so the $\rho$-enlargement of $\supp\nabla\chi_j$ is compact and separated from every root.
If $\norm\Delta\le\rho$, a nonzero cutoff change places both endpoints in a fixed compact set and $X$ in this enlargement.
The integrand in~\eqref{eq:linear-flux} is then bounded by $K\norm\Delta$.
On the enlargement, $\E[\norm\Delta\mid X]\le Ka$ by boundedness of the drift and conditional second moments. Its stationary mass is $o(a)$ by Lemma~\ref{lem:outside-refined}.
This contribution is $o(a^2)$.

If $\norm\Delta>\rho$ and the cutoff changes, at least one endpoint lies in the compact cutoff support.
Hence $\norm{X^+-m_j}\le K+\norm\Delta$, and the remaining contribution is bounded by
\[
K\E[(1+\norm\Delta)\one\{\norm\Delta>\rho\}]
\le K_\rho\E[\norm\Delta^2\one\{\norm\Delta>\rho\}]=o(a^2).
\]
This proves the flux estimate using Lemma~\ref{lem:increment-tail}.

Apply stationarity componentwise to the bounded function $\chi_j(x)(x-m_j)$.
Expanding its increment and using conditional centering yields
\[
a\E[\chi_j(X)g(X)]
=\E[(\chi_j(X^+)-\chi_j(X))(X^+-m_j)]=o(a^2).
\]
On $\mathcal C_j$, $\norm{x-m_j}=\dist(x,\Mcal)$ and $\norm{g(x)}\le L\norm{x-m_j}$.
Where $\one_{\mathcal C_j}-\chi_j$ is nonzero, this distance is at least a fixed $r>0$.
Thus
\[
\E[(\one_{\mathcal C_j}-\chi_j)\norm{g(X)}]
\le (L/r)\E[\dist(X,\Mcal)^2\one\{\dist(X,\Mcal)\ge r\}]=o(a).
\]
Replacing the cutoff proves the drift balance. The argument uses stationarity of the original chain, without assigning stationarity to a cell-conditioned law.
\end{proof}

\begin{proof}[Proof of Theorem~\ref{thm:bias}]
Fix any root and write $y=x-m_j$. Local $C^2$ regularity gives
\[
g(m_j+y)=A_jy+\tfrac12(D^2g(m_j):yy^T)+r_j(y),\qquad \norm{r_j(y)}=o(\norm y^2).
\]
For $\epsilon>0$, choose $r>0$ so $\norm{r_j(y)}\le\epsilon\norm y^2$ when $\norm y\le r$.
Outside this ball, global Lipschitz continuity and the fixed Taylor coefficients give $\norm{r_j(y)}\le K_r\norm y^2$.
Consequently,
\[
\frac1a\E[\one_{\mathcal C_j}\norm{r_j(X-m_j)}]
\le\epsilon\E\norm{Z_a}^2+K_r\E[\norm{Z_a}^2\one\{\norm{Z_a}>r/\sqrt a\}].
\]
Let $a\downarrow0$ and then $\epsilon\downarrow0$. Theorem~\ref{thm:local-moments} makes the remainder $o(a)$ and also gives
$\E[\one_{\mathcal C_j}(X-m_j)(X-m_j)^T]=a w_{a,j}\bar S_j+o(a)$.
Combine these facts with Lemma~\ref{lem:linear-flux} to obtain
\[
A_j\E[\one_{\mathcal C_j}(X-m_j)]+\tfrac a2w_{a,j}t_j=o(a).
\]
Invertibility of $A_j$ proves~\eqref{eq:mean-bias}.

For a $C^2$ observable with coordinatewise bounded Hessian, the same truncation argument controls its second-order Taylor remainder by $o(a)$.
Its linear term uses~\eqref{eq:mean-bias}, and its quadratic term uses~\eqref{eq:second}, yielding~\eqref{eq:observable-bias}.
Bounded Hessians ensure at most quadratic growth and integrability.
Dividing by weights bounded below proves both conditional expansions.
\end{proof}

For gradient drift and $f=F$, the linear term vanishes and $k_{F,j}=\tfrac12\tr(H_j\bar S_j)$.
At a compatible root, $H_j\bar S_j+\bar S_jH_j=C_j$, giving $k_{F,j}=\tfrac14\tr C_j$.
Incompatible roots have vanishing weights, so their coefficient difference does not affect the weighted limit~\eqref{eq:weighted-loss}.
For a fixed evaluation loss, the gradient term in~\eqref{eq:test-bias} generally remains.
The expansion concerns algorithmic variation of this fixed loss, rather than uncertainty in a population risk.

\subsection{Mean offsets and root-centered Gaussian accuracy}\label{app:mean-offset}
For an invariant-law family $\pi_a$ with $\underline w_j:=\liminf_{a\downarrow0}w_{a,j}>0$ and $b_j\ne0$, the mean-difference bound for $W_2$ gives
\[
D_a(\pi_a)\ge\sqrt{w_{a,j}}\,\norm{\E[Z_a\mid J_a=j]},
\qquad
\liminf_{a\downarrow0}\frac{D_a(\pi_a)}{\sqrt a}\ge\sqrt{\underline w_j}\,\norm{b_j}.
\]
Indeed, Theorem~\ref{thm:bias} gives $\E[Z_a\mid J_a=j]=\sqrt a\,b_j+o(\sqrt a)$, proving the second inequality.
This mean-offset contribution to the same-label error is consistent with the scalar nonlinear lower bound of \citet[Proposition~3.9]{WangEtAl2026}.

\subsection{Between-root and within-root observable variance}
The weighted expansions also describe global variability, including the contribution of different root predictions.
\begin{corollary}[Global observable variance]\label{cor:global-variance}
Under Assumption~\ref{ass:standing}, let $f:\R^d\to\R$ be $C^2$, globally Lipschitz, and have bounded Hessian.
Set $f_j=f(m_j)$, $v_{f,j}=Df(m_j)\bar S_jDf(m_j)^T$, and $\bar f_a=\sum_jw_{a,j}f_j$.
With $k_{f,j}$ from Theorem~\ref{thm:bias}, uniformly over invariant-law selections,
\begin{equation}\label{eq:global-variance}
\operatorname{Var}(f(X))
=\sum_jw_{a,j}(f_j-\bar f_a)^2
+a\sum_jw_{a,j}\{v_{f,j}+2(f_j-\bar f_a)k_{f,j}\}+o(a).
\end{equation}
The weights are the actual finite-$a$ weights; no convergence or positive lower bound is required.
\end{corollary}
\begin{proof}
Write $\Delta_j f=f(X)-f_j$ on $\mathcal C_j$.
Theorem~\ref{thm:bias} gives $\E[\one_{\mathcal C_j}\Delta_j f]=aw_{a,j}k_{f,j}+o(a)$.
We first show
\begin{equation}\label{eq:observable-increment-second}
\E[\one_{\mathcal C_j}(\Delta_j f)^2]=aw_{a,j}v_{f,j}+o(a).
\end{equation}
For $y=x-m_j$, differentiability implies
$(f(m_j+y)-f_j)^2-(Df(m_j)y)^2=o(\norm y^2)$ as $y\to0$.
Global Lipschitz continuity bounds the absolute difference by $K\norm y^2$ everywhere.
For any $\eta>0$, choose $r>0$ so the local bound is $\eta\norm y^2$.
After division by $a$, the expected absolute difference on $\mathcal C_j$ is at most
\[
\eta\E\norm{Z_a}^2+K\E[\norm{Z_a}^2\one\{\norm{Z_a}>r/\sqrt a\}].
\]
Theorem~\ref{thm:local-moments}, followed by $\eta\downarrow0$, makes this remainder $o(1)$ uniformly.
The rootwise second-moment formula~\eqref{eq:second} then proves~\eqref{eq:observable-increment-second}.

Summing over the finite root set yields
\begin{align*}
\E f(X)&=\bar f_a+a\sum_jw_{a,j}k_{f,j}+o(a),\\
\E f(X)^2&=\sum_jw_{a,j}f_j^2+a\sum_jw_{a,j}(v_{f,j}+2f_jk_{f,j})+o(a).
\end{align*}
The coefficients and $\bar f_a$ are uniformly bounded because the root set is finite.
Squaring the first expansion and subtracting it from the second proves~\eqref{eq:global-variance}.
Every estimate is weighted before summation, so vanishing component weights cause no division by small probabilities.
\end{proof}
The leading term is between-root variability and can remain of order one.
The order-$a$ terms combine within-root Gaussian variance with a signed correction from the shifted component means.
The argument uses only second moments: it controls squared observable increments directly, without applying the mean expansion to $f^2$.

\section{Weighted localization and local smooth-test rates}\label{app:rate}
\begin{theorem}[Local smooth-test rate]\label{thm:rate}
Assume Assumption~\ref{ass:standing} and fix $j\in\Scal$.
For centered additive innovations with $\E\norm{\varepsilon}^{2+\delta}<\infty$,
\begin{equation}\label{eq:additive-rate}
\mathcal E_{a,j}\le K_j a^{\delta/2}.
\end{equation}
For state-dependent noise, suppose in a fixed neighborhood of $m_j$ that
\begin{equation}\label{eq:rate-regularity}
\sup_x\E\norm{\xi(x,\zeta_1)}^{2+\delta}<\infty,
\qquad \norm{C(x)-C_j}_{\op}\le K\norm{x-m_j}^\beta,
\qquad 0<\beta\le1.
\end{equation}
Then $\mathcal E_{a,j}\le K_j(a^{\delta/2}+a^{\beta/2}+\sqrt a)$.
Whenever $w_{a,j}>0$, the corresponding conditional smooth-test error is bounded by $\mathcal E_{a,j}/w_{a,j}$.
No higher stationary moment, positive-definite covariance, or unstable-root excitation is needed for this stable-root statement.
\end{theorem}

\begin{lemma}[Weighted cutoff flux]\label{lem:weighted}
For the cutoffs in Appendix~\ref{app:mixture}, let $Z_j^+=(X^+-m_j)/\sqrt a$.
Under Assumption~\ref{ass:standing}, uniformly over invariant-law selections,
\begin{equation}\label{eq:weighted-flux}
\frac1a\E[|\chi_j(X^+)-\chi_j(X)|(1+\norm{Z_j^+})]\le K\sqrt a.
\end{equation}
\end{lemma}
\begin{proof}
Use $D_j^\rho$ and $\Delta$ from that appendix.
On $\{\norm{\Delta}\le\rho\}$, a cutoff change forces both endpoints into a fixed compact set and $X\in D_j^\rho$.
For $0<a\le1$,
\[
|\chi_j(X^+)-\chi_j(X)|(1+\norm{Z_j^+})
\le Ka^{-1/2}\norm{\Delta}\one_{D_j^\rho}(X).
\]
Its expectation is at most $Ka^{-1/2}\,a\pi_a(D_j^\rho)\le Ka^{3/2}$.
On $\{\norm{\Delta}>\rho\}$, a nonzero cutoff change means at least one endpoint lies in the compact support of $\chi_j$.
Therefore $\norm{X^+-m_j}\le K+\norm{\Delta}$ and the weighted change is bounded by
\[
Ka^{-1/2}(1+\norm{\Delta})\one\{\norm{\Delta}>\rho\}.
\]
Both $\Pp(\norm{\Delta}>\rho)$ and
$\E[\norm{\Delta}\one\{\norm{\Delta}>\rho\}]$ are bounded by constants times $\E\norm{\Delta}^2\le Ka^2$.
The large-increment contribution is thus also $O(a^{3/2})$.
Adding the two bounds and dividing by $a$ proves the lemma.
\end{proof}

\paragraph{OU Poisson equation.}
Fix a stable root and write $A=A_j$, $C=C_j$, $\Gamma=\Gamma_j$.
Choose $M_A\ge1$ and $\rho_A>0$ with $\norm{e^{-At}}\le M_Ae^{-\rho_At}$.
Let $P_t^{\rm OU}$ be the OU semigroup and, for $h\in\mathcal H_{2,\delta}$, define
\begin{equation}\label{eq:poisson}
u_h(z)=-\int_0^\infty(P_t^{\rm OU}h(z)-\Gamma h)\dd t.
\end{equation}
Coupling a process starting at $z$ to a stationary process starting at $G\sim\Gamma$ with the same noise bounds their displacement by $M_Ae^{-\rho_At}\norm{z-G}$.
Hence the integral converges and
\begin{align}\label{eq:poisson-bounds}
|u_h(z)|&\le \frac{M_A}{\rho_A}(\norm{z}+\E\norm{G}),&
\norm{\nabla u_h}_\infty&\le\frac{M_A}{\rho_A},\\
\norm{D^2u_h}_\infty&\le\frac{M_A^2}{2\rho_A},&
[D^2u_h]_\delta&\le\frac{M_A^{2+\delta}}{(2+\delta)\rho_A}.\nonumber
\end{align}
The derivative bounds follow by differentiating the affine OU transition map under the integral, with the displayed exponential bounds providing domination.
They use the regularity of $h$, not elliptic smoothing, and remain valid when $C$ is singular.
Integrating $\partial_tP_t^{\rm OU}h=\mathcal L_jP_t^{\rm OU}h$ gives
$\mathcal L_ju_h=h-\Gamma h$; the same bounds justify interchange at each fixed $z$.

\begin{proof}[Proof of Theorem~\ref{thm:rate}]
Put $Z=(X-m_j)/\sqrt a$ and $H=(X^+-X)/\sqrt a$.
Shrink the cutoff support, if necessary, to the neighborhood of~\eqref{eq:rate-regularity}.
For each fixed $a$, $x\mapsto\chi_j(x)u_h((x-m_j)/\sqrt a)$ is bounded, so its stationary equation is valid without extra integrability assumptions.
Lemma~\ref{lem:weighted} and~\eqref{eq:poisson-bounds} imply
\begin{equation}\label{eq:poisson-stationarity}
\left|a^{-1}\E\{\chi_j(X)[u_h(Z+H)-u_h(Z)]\}\right|\le K\sqrt a.
\end{equation}
Taylor expansion with the H\"older Hessian has remainder $|R_h|\le K\norm{H}^{2+\delta}$.
Only $X$ inside the compact cutoff support contributes. Therefore
\[
a^{-1}\E[\chi_j|R_h|]\le Ka^{\delta/2}\E[\chi_j\norm{\xi(X,\zeta_1)-g(X)}^{2+\delta}]
\le Ka^{\delta/2}.
\]
For additive noise this uses $\E\norm{\varepsilon}^{2+\delta}<\infty$; for state-dependent noise it uses the local moment in~\eqref{eq:rate-regularity}.
It never requires a $(2+\delta)$ moment of $Z$.
The nonlinear drift error is $O(\sqrt a)$ by~\eqref{eq:linearization}, and the drift-squared second-order term is $O(a)$.
For state-dependent noise the additional covariance error is bounded by
\[
K\E[\chi_j\norm{C(X)-C_j}]\le Ka^{\beta/2}\E[\chi_j\norm{Z}^\beta]
\le Ka^{\beta/2},
\]
where the last inequality follows from~\eqref{eq:cutoff-moments} and H\"older's inequality for a measure of mass at most one.
This error is zero for additive noise.
Substituting in~\eqref{eq:poisson-stationarity} yields, uniformly in $h$,
\[
|\nu_{a,j}h-q_{a,j}\Gamma_jh|\le K(a^{\delta/2}+\sqrt a+a)
\]
in the additive case, with $Ka^{\beta/2}$ added in the state-dependent case.
Replacing the cutoff by its Voronoi cell and $q_{a,j}$ by $w_{a,j}$ costs at most $2Ka$, since $\norm{h}_\infty\le1$.
For $0<\delta\le1$ and $a\le1$, $a+\sqrt a\le2a^{\delta/2}$.
This proves the two rates; division by a positive $w_{a,j}$ proves the conditional bound.
An arbitrarily small weight is allowed in the absolute bound, but the resulting relative bound need not be informative.
\end{proof}

\section{Tail moduli and quantitative unstable mass}\label{app:tails}
This section uses the weighted cutoff estimate. Choose a fixed root cutoff $\chi_j\in C_c^\infty$ supported within the local regularity neighborhood, and write
\[
 \nu_{a,j}h=\E[\chi_j(X)h((X-m_j)/\sqrt a)],\quad q_{a,j}=\E\chi_j(X).
\]
The localization proof gives $|w_{a,j}-q_{a,j}|\le Ka$, $\int\norm z^2\nu_{a,j}(dz)\le K$, and
\eqref{eq:weighted-flux}.
With a compact neighborhood $N_j$ containing the cutoff support, define the stationary averaged modulus and covariance variation
\begin{align}
\overline T_j(a;\pi_a)&=\E[\chi_j(X)\norm{\xi(X,\zeta)}^2
\min\{\sqrt a\norm{\xi(X,\zeta)},1\}],\label{eq:T}\\
R_{C,j}(a;\pi_a)&=\E[\chi_j(X)\norm{C(X)-C_j}].\label{eq:K}
\end{align}
An optional upper bound is the fixed-neighborhood supremum
\[
T_j^{\sup}(a)=\sup_{x\in N_j}\E[\norm{\xi(x,\zeta)}^2\min\{\sqrt a\norm{\xi(x,\zeta)},1\}].
\]
Since $0\le\chi_j\le1$, $\overline T_j\le T_j^{\sup}$.
The averaged modulus tends to zero under the standing assumptions; the supremum need not do so.

\begin{lemma}[Localized generator residual]\label{lem:residual}
Suppose $f$ is $C^2$, grows at most linearly, and has bounded first and second derivatives with globally Lipschitz Hessian. With constants allowed to depend on these bounds and the fixed root,
\begin{equation}\label{eq:residual}
\begin{aligned}
 \left|\int \mathcal L_jf\,d\nu_{a,j}\right|
 &\le K\{\sqrt a+\overline T_j(a;\pi_a)+R_{C,j}(a;\pi_a)\},\\
 \mathcal L_jf(z)&=-\langle A_jz,\nabla f(z)\rangle+\tfrac12\tr(C(m_j)D^2f(z)).
\end{aligned}
\end{equation}
Only Assumption~\ref{ass:standing} is needed; the root need not be stable.
\end{lemma}
\begin{proof}
For fixed $a$, stationarity applies to the bounded function
$x\mapsto\chi_j(x)f((x-m_j)/\sqrt a)$. Put $Z=(X-m_j)/\sqrt a$ and $H=\sqrt a(\xi-g)$. Moving the cutoff change to the other side and using \eqref{eq:weighted-flux} bounds the localized one-step generator by $K\sqrt a$.

Boundedness and Lipschitz continuity of the Hessian give the Taylor remainder estimate
\[
 |f(z+h)-f(z)-\nabla f(z)^Th-\tfrac12h^TD^2f(z)h|
 \le K\norm h^2\min\{\norm h,1\}.
\]
On the compact cutoff support, $g$ is bounded. Splitting $\norm\xi$ above and below a bound for $\norm g$, and using
$u^2\min\{su,1\}$ at $2u$ being at most eight times its value at $u$, bounds the remainder after averaging with $\chi_j$ and dividing by $a$ by $K(\overline T_j(a;\pi_a)+\sqrt a)$.

Conditional centering gives the drift $-g/\sqrt a$ and conditional uncentered second moment $C(X)+g(X)g(X)^T$ in the generator. The local $C^2$ expansion of $g$ and the existing localized second moment give nonlinear drift error $K\sqrt a$. The drift-squared term is $O(a)$ by \eqref{eq:energy-explicit}; the covariance variation is bounded by a constant times \eqref{eq:K}. This proves \eqref{eq:residual} for $a\le1$.
\end{proof}

\begin{proposition}[Stable-root modulus bound]\label{prop:modulus}
Let $\mathcal H_{2,1}$ be the bounded smooth test class with Lipschitz Hessian. At any stable root,
\[
 \sup_{h\in\mathcal H_{2,1}}
 \left|\E[\one_{\mathcal C_j}(X)h((X-m_j)/\sqrt a)]-w_{a,j}\Gamma_jh\right|
 \le K\{\sqrt a+\overline T_j(a;\pi_a)+R_{C,j}(a;\pi_a)\}.
\]
Both $\overline T_j(a;\pi_a)$ and $R_{C,j}(a;\pi_a)$ tend to zero under Assumption~\ref{ass:standing}. If $C$ is locally $\beta$-H\"older, $0<\beta\le1$, then $R_{C,j}(a;\pi_a)\le Ka^{\beta/2}$; for additive noise it is zero. If local $(2+\delta)$ moments are uniformly bounded, $0<\delta\le1$, then $\overline T_j(a;\pi_a)\le Ka^{\delta/2}$.
\end{proposition}
\begin{proof}
The OU Poisson construction in Appendix~\ref{app:rate} supplies $u_h$ with $\mathcal L_ju_h=h-\Gamma_jh$, linear growth, and uniformly bounded derivatives up through a Lipschitz Hessian, even for singular $C(m_j)$. Apply Lemma \ref{lem:residual} to $u_h$ and replace cutoff mass by cell mass at cost $O(a)$.

For the averaged tail modulus, split $\norm\xi\le R$ and its complement.
The first contribution is at most $\sqrt aR\E\norm\xi^2$, and the second is at most $\E[\norm\xi^2\one\{\norm\xi>R\}]$.
Lemma~\ref{lem:stationary-noise-tails} proves convergence uniformly over invariant laws by sending $a\downarrow0$ and then $R\to\infty$.
For covariance continuity, with $\omega_j(r)=\sup_{\norm{x-m_j}\le r}\norm{C(x)-C(m_j)}$, concentration gives the useful uniform estimate
\[
 R_{C,j}(a;\pi_a)\le K\inf_{0<r\le r_j}\{\omega_j(r)+a/r^2\}.
\]
Under Hölder regularity a sharper direct estimate is
$R_{C,j}(a;\pi_a)\le Ka^{\beta/2}\E[\chi_j(X)\norm Z^\beta]\le Ka^{\beta/2}$.
Finally $\min(u,1)\le u^\delta$ proves the moment implication for the stated class $\mathcal H_{2,1}$.
\end{proof}

\begin{proposition}[An explicit endpoint for standardized $t_3$ noise]\label{prop:t3}
For scalar $\varepsilon=t_3/\sqrt3$, whose density is $2/[\pi(1+x^2)^2]$, and $s>0$,
\begin{equation}\label{eq:t3}
 \E[\varepsilon^2\min\{s|\varepsilon|,1\}]
 =\frac{2s}{\pi}\log(1+s^{-2})+\frac2\pi\arctan s.
\end{equation}
Consequently additive standardized $t_3$ innovations give a smooth-test upper bound $O(\sqrt a\log(1/a))$ for $\mathcal H_{2,1}$. The same order holds for a fixed finite-dimensional linear transform of independent such variables.
This bound uses only the general modulus; Appendix~\ref{app:symmetric} removes the logarithm when the root covariance is positive definite by using conditional symmetry.
\end{proposition}
\begin{proof}
Split the symmetric density integral at $1/s$:
\[
 \frac4\pi\left(s\int_0^{1/s}\frac{x^3}{(1+x^2)^2}\,dx
                  +\int_{1/s}^\infty\frac{x^2}{(1+x^2)^2}\,dx\right).
\]
Antiderivatives are
$\frac12[\log(1+x^2)+(1+x^2)^{-1}]$ and
$\frac12[\arctan x-x/(1+x^2)]$, respectively. The rational boundary terms cancel, giving \eqref{eq:t3}. Substitute $s=\sqrt a$.
For a fixed linear transform $L$ of $d$ such variables, set $c=\norm L_{\op}$ and $u_i=|\varepsilon_i|$.
Since $\norm{L\varepsilon}\le c\sum_i u_i\le cd\max_i u_i$ and $q_s(u)=u^2\min(su,1)$ is increasing,
\[
q_s(\norm{L\varepsilon})\le c^2d^2\sum_{i=1}^d u_i^2\min(scd u_i,1).
\]
Taking expectations and applying the scalar formula at $scd$ proves the stated order for fixed $c,d$; $c=0$ is immediate.
\end{proof}

\subsection{Noise on the unstable subspace}
\begin{theorem}[Excited unstable mass]\label{thm:saddle}
Under Assumption~\ref{ass:standing}, fix an unstable root with
$P_j^uC_j(P_j^u)^T\ne0$. Then, uniformly over invariant-law selections,
\begin{equation}\label{eq:saddlemod}
w_{a,j}\le K_j\{\sqrt a+\overline T_j(a;\pi_a)+R_{C,j}(a;\pi_a)\}.
\end{equation}
Under the local moment and covariance conditions~\eqref{eq:rate-regularity},
the bound is $K_j(a^{\delta/2}+a^{\beta/2}+\sqrt a)$.
For additive noise the covariance-variation term vanishes.
No diagonalizability or positive definiteness of the projected covariance is required.
\end{theorem}
\begin{proof}
Choose real coordinates $T^{-1}A_jT=\diag(A_s,A_u)$ on the generalized stable and unstable subspaces.
The spectrum of $A_u$ has negative real part.
Let $L_u$ be the unstable rows of $T^{-1}$, so $L_uA_j=A_uL_u$, and set
\[
B=-A_u,\qquad C_u=L_uC_jL_u^T,\qquad
H=\int_0^\infty e^{-B^Tt}e^{-Bt}\dd t.
\]
The integral converges, $H\succ0$, and $B^TH+HB=I$.
The projection condition is equivalent to $C_u\ne0$:
$P_j^uC_j(P_j^u)^T$ is the congruence by $T$ of the block matrix with only its $uu$ block $C_u$ nonzero.
Define
\[
\kappa_u=\frac1{2\lambda_{\max}(H)},\qquad
\tau=\tr(HC_u)>0,\qquad s=\tau/\kappa_u.
\]
For $y=L_uz$ and $q=y^THy$, use the fixed test
\[
f(z)=\log(1+q/s).
\]
Writing $M=H^{1/2}L_u$, the derivatives of $\log(1+\norm v^2)$ give, for every fixed $k\ge1$,
\[
\sup_z\norm{D^kf(z)}\le c_k\norm M^k s^{-k/2}.
\]
Indeed, each derivative is a rational function with a denominator power of $1+\norm v^2$ and a numerator of lower degree.
Composition with $Mz/\sqrt s$ gives the bound.
The function also grows at most linearly, so Lemma~\ref{lem:residual} applies.

On functions of $y$, the root generator is
$\mathcal L_u=(By)\cdot\nabla_y+\frac12C_u:D_y^2$.
Consequently,
\[
\mathcal L_u q=\norm y^2+\tau\ge2\kappa_u q+\tau.
\]
The nonnegative matrix $H^{1/2}C_uH^{1/2}$ has trace $\tau$, hence
$y^THC_uHy\le\tau q$.
The diffusion chain rule now yields
\begin{align*}
\mathcal L_jf(z)
&=\frac{\norm y^2+\tau}{s+q}-\frac{2y^THC_uHy}{(s+q)^2}\\
&\ge\frac{2\kappa_u q+\tau}{s+q}-\frac{2\tau q}{(s+q)^2}
=\kappa_u\frac{2t^2+t+1}{(1+t)^2}\ge\frac{\kappa_u}{2},\qquad t=q/s.
\end{align*}
For the last inequality, subtracting half the denominator from the numerator leaves $3t^2/2+1/2\ge0$.
With the cutoff mass $q_{a,j}$, Lemma~\ref{lem:residual} gives
\[
(\kappa_u/2)q_{a,j}\le\int\mathcal L_jf\dd\nu_{a,j}
\le K_j\{\sqrt a+\overline T_j(a;\pi_a)+R_{C,j}(a;\pi_a)\}.
\]
Since $0\le w_{a,j}-q_{a,j}\le Ka$, this proves~\eqref{eq:saddlemod}.
Proposition~\ref{prop:modulus} supplies the stated moment and H\"older orders.
The constants depend on the fixed root and noise; they can deteriorate as hyperbolicity or excitation is lost.
\end{proof}

\begin{corollary}[Symmetric unstable-root bound]\label{cor:symmetric-saddle}
Under the conditions of Theorem~\ref{thm:saddle}, assume also conditional symmetry of the innovation on the cutoff neighborhood.
Then
\[
w_{a,j}\le K_j\{\sqrt a+\overline T_j^{\rm sym}(a;\pi_a)+R_{C,j}(a;\pi_a)\},
\]
where $\overline T_j^{\rm sym}$ is defined in~\eqref{eq:symmetric-modulus}.
This conclusion does not require ellipticity.
\end{corollary}
\begin{proof}
The same $f$ has a bounded fourth derivative.
Apply the symmetric residual of Lemma~\ref{lem:symmetric-residual} and the positive generator bound above.
The smoothing assumption in Lemma~\ref{lem:elliptic-smoothing} is unnecessary for this explicit smooth test.
\end{proof}

The Jordan example after Lemma~\ref{lem:hyperbolic-OU} illustrates the strict gain over left-eigenvector excitation.
For its unstable block, $H=\left(\begin{smallmatrix}1/2&1/4\\1/4&3/4\end{smallmatrix}\right)$ solves the displayed Lyapunov equation.
Thus the qualitative and quantitative exclusion criteria cover the same projected-noise condition.

\begin{corollary}[Quantitative marked smooth approximation]
Under Assumptions~\ref{ass:standing} and~\ref{ass:excitation}, suppose the stable-root bound and Theorem~\ref{thm:saddle} give the same error order $r(a)$ at every root, where, for example,
$r(a)=a^{\delta/2}+a^{\beta/2}+\sqrt a$. Then, for tests $h(j,z)$ whose slice at each label belongs to $\mathcal H_{2,1}$,
\[
 \left|\E h(J_a,Z_a)-\sum_{j\in\Scal}\frac{w_{a,j}}{s_a}\Gamma_j h(j,\cdot)\right|
 \le K r(a).
\]
For the $\mathcal H_{2,\delta}$ class the same moment-based conclusion follows by combining Theorem~\ref{thm:rate} with Theorem~\ref{thm:saddle}.
\end{corollary}
\begin{proof}
Sum the absolute stable-cell errors; bound every unstable cell by its mass because the test is bounded by one; then normalize the stable weights at cost at most $1-s_a$. Finiteness of the root set completes the proof.
\end{proof}

\section{Symmetric innovations and a sharp Student endpoint}\label{app:symmetric}
Fix any root and a compact neighborhood $N_j$ containing the cutoff support.
Assume $\xi(x,\zeta)\overset d=-\xi(x,\zeta)$ for every $x\in N_j$ and define
\begin{equation}\label{eq:symmetric-modulus}
\overline T_j^{\rm sym}(a;\pi_a)
=\E[\chi_j(X)\norm{\xi(X,\zeta)}^2\min\{a\norm{\xi(X,\zeta)}^2,1\}].
\end{equation}
This modulus tends to zero uniformly by Lemma~\ref{lem:stationary-noise-tails}.
It is bounded above by the corresponding fixed-neighborhood supremum when an explicit uniform noise bound is available.
Symmetry here concerns the innovation conditional on the state, not the full increment $\xi-g$.

\begin{lemma}[Symmetric generator residual]\label{lem:symmetric-residual}
Let $f\in C^4$ have at most linear growth and bounded derivatives of orders one through four.
For the cutoff measure $\nu_{a,j}$ in~\eqref{eq:localized}, under Assumption~\ref{ass:standing} and the preceding symmetry condition,
\[
\left|\int\mathcal L_jf\dd\nu_{a,j}\right|
\le K\{\sqrt a+\overline T_j^{\rm sym}(a;\pi_a)+R_{C,j}(a;\pi_a)\}.
\]
The constant may depend on the growth and derivative bounds of $f$, but not on $a$ or the invariant-law selection.
No stability, ellipticity or fourth innovation moment is needed for this residual.
\end{lemma}
\begin{proof}
Bounded second and fourth derivatives imply, for all $v,h$,
\begin{equation}\label{eq:central-remainder}
\left|\frac{f(v+h)+f(v-h)-2f(v)}2-\frac12h^TD^2f(v)h\right|
\le K\norm h^2\min\{\norm h^2,1\}.
\end{equation}
For $\norm h\le1$, use fourth-order Taylor expansion with its integral remainder and cancellation of the odd terms.
For $\norm h>1$, use the integral second-order formula and bounded Hessian.
Put $z=(x-m_j)/\sqrt a$ and $v=z-\sqrt a g(x)$.
Conditional symmetry and~\eqref{eq:central-remainder} give
\begin{align*}
\E[f(v+\sqrt a\xi)-f(v)\mid X=x]
&=\frac a2\tr(C(x)D^2f(v))+r_a(x),\\
|r_a(x)|&\le Ka\E[\norm\xi^2\min\{a\norm\xi^2,1\}\mid X=x].
\end{align*}
Expand the deterministic difference $f(v)-f(z)$ separately to first order.
Its remainder is bounded by $Ka\norm{g(x)}^2$, whose localized expectation divided by $a$ is $O(a)$.
Replacing $D^2f(v)$ by $D^2f(z)$ costs at most
$K\sqrt a\E[\chi_j\norm{g(X)}]=O(a)$ after division by $a$; the covariance is bounded on the cutoff support.
Drift linearization costs $O(\sqrt a)$ by~\eqref{eq:linearization}, and covariance replacement costs $KR_{C,j}$.
Finally apply stationarity to $\chi_j(x)f((x-m_j)/\sqrt a)$.
The weighted cutoff estimate~\eqref{eq:weighted-flux} contributes $O(\sqrt a)$, proving the claim.
\end{proof}

\begin{lemma}[An integrable fourth derivative]\label{lem:elliptic-smoothing}
Suppose $A$ is positive stable and $C\succ0$.
For smooth mollifications $h_\epsilon$ of $h\in\mathcal H_{2,1}$, the OU Poisson solutions
$u_\epsilon=-\int_0^\infty(P_th_\epsilon-\Gamma h_\epsilon)\dd t$
have at most linear growth and bounded derivatives of orders one through four, uniformly over $h$ and $\epsilon$.
The constants depend on $A,C$ and the dimension, and need not remain bounded as $C$ becomes singular.
\end{lemma}
\begin{proof}
Mollification preserves the bounded first and second derivatives and Lipschitz Hessian.
In particular, $\norm{D^3h_\epsilon}_\infty\le K$, independently of $\epsilon$.
Write $E_t=e^{-At}$ and $Q_t=\int_0^tE_sCE_s^T\dd s$, with $G_t\sim\N(0,Q_t)$.
Differentiate $P_th_\epsilon(z)=\E h_\epsilon(E_tz+G_t)$ three times on the test function.
For the fourth derivative, Gaussian integration by parts yields, in unit directions $v_1,\ldots,v_4$,
\[
D^4P_th_\epsilon(z)[v_1,\ldots,v_4]
=\E\!\left[D^3h_\epsilon(E_tz+G_t)[E_tv_1,E_tv_2,E_tv_3]
\langle Q_t^{-1}G_t,E_tv_4\rangle\right].
\]
Cauchy--Schwarz bounds the Gaussian score by $\norm{Q_t^{-1/2}E_tv_4}$, so
\[
\norm{D^4P_th_\epsilon}_\infty\le K\norm{E_t}^4\norm{Q_t^{-1/2}}.
\]
Since $C\succ0$, $Q_t\succeq ctI$ for $0<t\le1$; the bound is $O(t^{-1/2})$ near zero.
For $t\ge1$, $Q_t\succeq Q_1\succ0$ and exponential decay of $E_t$ makes it integrable.
The same affine-semigroup differentiation gives integrable bounds for the first three derivatives.
Synchronous coupling gives linear growth as in~\eqref{eq:poisson-bounds}.
These majorants justify differentiation of the time integral and are uniform in the mollifier.
\end{proof}

\begin{corollary}[Symmetric innovations]\label{cor:symmetric-rate}
Under Assumption~\ref{ass:standing}, fix a stable root with $C_j\succ0$ and $\xi(x,\zeta)\overset d=-\xi(x,\zeta)$ on $N_j$.
For $\mathcal H_{2,1}$ tests,
\begin{equation}\label{eq:symmetric-rate}
\mathcal E_{a,j}\le K_j\{\sqrt a+\overline T_j^{\rm sym}(a;\pi_a)+R_{C,j}(a;\pi_a)\}.
\end{equation}
For additive noise obtained by a fixed invertible linear transform of $d$ independent standardized $t_3$ coordinates, $\mathcal E_{a,j}=O(\sqrt a)$.
\end{corollary}

\begin{proof}[Proof of Corollary~\ref{cor:symmetric-rate}]
Apply Lemma~\ref{lem:symmetric-residual} to the mollified Poisson solutions from Lemma~\ref{lem:elliptic-smoothing}.
Their generators equal $h_\epsilon-\Gamma_jh_\epsilon$.
Replacing cutoff mass by cell mass costs $O(a)$ for bounded tests, as in the proof of Theorem~\ref{thm:rate}.
The constants are uniform in $h$ and $\epsilon$.
Since $h$ is globally Lipschitz and bounded, $h_\epsilon\to h$ uniformly, so both finite-measure integrals converge.
This proves the displayed bound on exactly $\mathcal H_{2,1}$, without assuming classical third derivatives of every unmollified test.

For a locally Lipschitz covariance,
$R_{C,j}\le K\E[\chi_j\norm{X-m_j}]=O(\sqrt a)$.
To control the other term, a sufficient uniform tail condition is
\begin{equation}\label{eq:t3-tail-condition}
\sup_{x\in N_j}\E[\norm{\xi(x,\zeta)}^2\one\{\norm{\xi(x,\zeta)}>R\}]\le K/R,
\qquad R\ge1.
\end{equation}
Indeed, for $Y=\norm{\xi(x,\zeta)}$,
\[
\E[Y^2\min\{aY^2,1\}]
=2a\int_0^{a^{-1/2}}t\E[Y^2\one\{Y>t\}]\dd t=O(\sqrt a)
\]
uniformly over $x$.
Fixed linear transforms of independent standardized $t_3$ coordinates satisfy~\eqref{eq:t3-tail-condition}.
To see this in fixed dimension, bound $\norm{L\varepsilon}$ by $K\max_r|\varepsilon_r|$ and its squared tail by a sum of scalar squared tails.
The same argument applies to uniformly bounded state-dependent transforms when the covariance regularity and root ellipticity also hold.
Local Lipschitz covariance alone does not give an explicit rate for $\overline T_j^{\rm sym}$.
\end{proof}

\subsection{An exact finite-variance endpoint}
The density of $\varepsilon=t_3/\sqrt3$ is $2/[\pi(1+x^2)^2]$.
With $s=\sqrt a$, splitting the integral at $1/s$ gives
\begin{align}\label{eq:exact-symmetric-t3}
T^{\rm sym}(s^2)
&=\frac4\pi\left\{s^2\int_0^{1/s}\frac{x^4}{(1+x^2)^2}\dd x+
\int_{1/s}^{\infty}\frac{x^2}{(1+x^2)^2}\dd x\right\}\\
&=\frac{6s}{\pi}-3s^2+\frac2\pi(1+3s^2)\arctan s
=\frac8\pi s+O(s^2).\nonumber
\end{align}
Thus symmetry improves the general $\sqrt a\log(1/a)$ modulus to $\sqrt a$ for this noise.

\begin{proposition}[Matching smooth lower test]\label{prop:symmetric-lower}
For $X^+=(1-a)X+a\varepsilon$, $\varepsilon=t_3/\sqrt3$ and $0<a<1$, let $Z_a=X/\sqrt a$ have its invariant law and let $G\sim\N(0,1/2)$.
Then
\begin{equation}\label{eq:symmetric-lower}
\E[\cos(Z_a)/4]-\E[\cos(G)/4]=\frac{e^{-1/4}}{36}\sqrt a+O(a).
\end{equation}
Consequently the $O(\sqrt a)$ bound of Corollary~\ref{cor:symmetric-rate} is order-sharp in this subclass.
\end{proposition}
\begin{proof}
The innovation characteristic function is $(1+|t|)e^{-|t|}$, and
$Z_a=\sqrt a\sum_{k\ge0}(1-a)^k\varepsilon_{-k}$ in $L^2$.
For fixed $t$, use $\log\phi_\varepsilon(u)=-u^2/2+|u|^3/3+O(u^4)$ uniformly near zero.
Geometric summation yields
\[
\log\phi_{Z_a}(t)=-\frac{t^2}{2(2-a)}+
\frac{|t|^3\sqrt a}{3(3-3a+a^2)}+O(a),
\]
because the remainder sum is bounded by $Ka^2\sum_{k\ge0}(1-a)^{4k}=O(a)$.
Exponentiation gives
$\phi_{Z_a}(t)-e^{-t^2/4}=e^{-t^2/4}|t|^3\sqrt a/9+O(a)$.
Set $t=1$ and divide by four.
Each of the four summands defining the $\mathcal H_{2,1}$ norm of $\cos(z)/4$ equals $1/4$, so this is an admissible lower test.
\end{proof}

\paragraph{Comparison with Gaussian and Rademacher innovations.}
In the same linear recursion, take $\varepsilon\sim\N(0,1)$ or $\Pr(\varepsilon=\pm1)=1/2$.
Near zero, their log characteristic functions are $-u^2/2$ and
$\log\cos u=-u^2/2-u^4/12+O(u^6)$, respectively.
For fixed $t$, the stationary series gives
\[
\log\phi_{Z_a}(t)=-\frac{t^2}{2(2-a)}+O(a)
=-\frac{t^2}{4}+O(a),
\]
since $a^2\sum_{k\ge0}(1-a)^{4k}=O(a)$ controls the quartic remainder.
Exponentiating and taking $t=1$ shows that the same cosine-test gap is $O(a)$.
For standardized $t_3$ noise, the even but nonanalytic $|u|^3$ term instead sums to order $\sqrt a$.
Thus the difference in this example reflects tail regularity, even though all three innovations are symmetric.

\begin{figure}[t]
\centering\includegraphics[width=\linewidth]{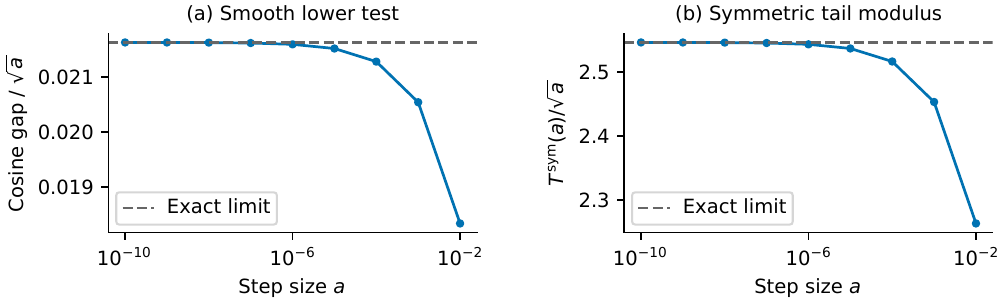}
\caption{\textbf{A sharp symmetric finite-variance endpoint.} Analytic computations for standardized $t_3$ innovations. (a) The cosine gap divided by $\sqrt a$ approaches $e^{-1/4}/36$. (b) The symmetric tail modulus divided by $\sqrt a$ approaches $8/\pi$. Dashed lines are the exact limits. The characteristic function uses 100 terms of a convergent logarithmic series at 80-digit arithmetic precision; the modulus uses~\eqref{eq:exact-symmetric-t3}. No trajectories or sampling error bars enter these calculations.}
\label{fig:symmetric-endpoint}
\end{figure}
The ellipticity condition enters only the extra smoothing step for $\mathcal H_{2,1}$.
The general singular-covariance bounds remain valid, and the specialized result concerns a smooth test metric rather than a Wasserstein rate.

\section{Observable and loss consequences}\label{app:observables}
\begin{proof}[Proof of Corollary~\ref{cor:prediction}]
Write $f(m_j+y)-f(m_j)=Df(m_j)y+r(y)$, where $\norm{r(y)}/\norm y\to0$ as $y\to0$. Within $\mathcal C_j$, $y=\sqrt a Z_a$.
For each fixed $R$, the scaled remainder on $\{\norm{Z_a}\le R\}$ is bounded by
$R\sup_{\norm y\le\sqrt a R}\norm{r(y)}/\norm y\to0$ (with ratio zero at $y=0$).
Conditional tightness makes the complementary probability arbitrarily small as $R\to\infty$.
Thus the remainder tends to zero in conditional probability, and Slutsky's theorem proves~\eqref{eq:prediction}.

If $f$ is globally Lipschitz with constant $L_f$, its squared scaled displacement is at most $L_f^2\norm{Z_a}^2$ within the cell.
More explicitly, its conditional tail above $R^2$ is bounded by
$L_f^2 w_{a,j}^{-1}\E[\norm{Z_a}^2\one\{\norm{Z_a}>R/L_f\}]$.
Theorem~\ref{thm:local-moments} and the weight lower bound make this tend to zero uniformly in small $a$; the constant-observable case is immediate.
Weak convergence and this uniform integrability give $W_2$ convergence and convergence of both first and second moments. Subtracting the limiting zero mean proves covariance convergence.
\end{proof}
\begin{proof}[Proof of Corollary~\ref{cor:loss}]
Since $\nabla F(m_j)=0$, the second-order expansion at $m_j$ and conditional tightness give \eqref{eq:losslaw}. Global Lipschitz continuity of $\nabla F$ gives a globally bounded Hessian, and hence
$|F(x)-F(m_j)|\le \frac L2\norm{x-m_j}^2$ for all $x$. The conditional uniform integrability established above thus permits expectation convergence in \eqref{eq:losslaw}. Taking the trace in $H_jS_j+S_jH_j=C(m_j)$ gives the factor $1/4$. Diagonalizing the symmetric nonnegative matrix gives the chi-square representation, including zero eigenvalues. The batch claim applies the same theorem separately for each fixed $B$, using $C_B=C_1/B$ and $S_{j,B}=S_{j,1}/B$. It is not a joint limit theorem for $B\to\infty$.
\end{proof}
For the weighted assertion~\eqref{eq:weighted-loss}, put $L_a=[F(X)-F(m_{J_a})]/a$.
The global Hessian bound gives $|L_a|\le(L/2)\norm{Z_a}^2$.
Along a weight-convergent subsequence, Taylor expansion and marked weak convergence identify the all-root quadratic limit of $L_a$; uniform integrability permits expectation convergence.
Let $M_j$ be the second moment of the all-root limit measure from Lemma~\ref{lem:all-root}.
The limiting mean is $\tfrac12\sum_j\tr(H_jM_j)$.
At every root, $H_jM_j+M_jH_j=q_jC_j$, so this mean equals $\tfrac14\sum_jq_j\tr C_j$.
The subsequence argument proves~\eqref{eq:weighted-loss} without excitation or convergent weights.
The argument also implies $W_1$ convergence of each nonnegligible conditional loss law to its quadratic limit, but does not control the loss variance.
\paragraph{Batch and discretization formulas.}
For independent sampling with replacement, covariance of a batch average is exactly $C_1(x)/B$.
The local Lyapunov equation is linear in $C$, giving $S_{j,B}=S_{j,1}/B$ for each fixed $B$.
For a batch sampled uniformly without replacement from $N$ summands, the corresponding factor is $(N-B)/(B(N-1))$, when $C_1$ is the population covariance with denominator $N$.
For a stable linear recursion the scaled stationary covariance solves~\eqref{eq:discrete-cov} whenever the spectral radius of $I-aA$ is below one.
Expansion gives $S_a=S+aS^{(1)}+O(a^2)$, where
$AS^{(1)}+S^{(1)}A^\trans=ASA^\trans$.
This follows from invertibility of the continuous Lyapunov operator and a finite-dimensional perturbation expansion, and quantifies the linear discretization correction.

\section{An exact unexcited saddle component}\label{app:unexcited}
Consider $g(x,y)=(x-2\tanh x,y)$ with additive noise $(0,\varepsilon)$, where $\E\varepsilon=0$ and $\E\varepsilon^2=1$.
The potential $U(x,y)=x^2/2-2\log\cosh x+y^2/2$ satisfies $\nabla U=g$, $\norm{D^2U}\le1$ and quadratic growth.
Moreover,
\[
\langle(x,y),g(x,y)\rangle\ge\tfrac12x^2+y^2-2.
\]
The drift is globally Lipschitz. The scalar equation $x=2\tanh x$ has roots $0,\pm m$.
On $x>0$, $2\tanh x-x$ is strictly concave, initially increasing, and eventually negative; hence it has exactly one positive root, with negative derivative there.
Thus the nonzero roots are stable, and the origin has $A=\diag(-1,1)$.
All roots are hyperbolic. Translation by additive finite-variance noise gives a Feller kernel and constant covariance, satisfying the noise conditions.
These facts verify Assumption~\ref{ass:standing}.

For $0<a<2$, each line $x\in\{-m,0,m\}$ is invariant and supports
\[
\pi_{a,x}=\delta_x\otimes\Law\!\left(a\sum_{k\ge0}(1-a)^k\varepsilon_{-k}\right).
\]
The series converges in $L^2$. On the central line every point has the origin as its nearest root, so the saddle weight is one.
Its scaled covariance is $\diag(0,1/(2-a))$, converging to $\bar S_0=\diag(0,1/2)$.
For Gaussian innovations the finite-step law is exactly Gaussian, and
\[
W_2\bigl(\pi_{a,0},\N(0,a\bar S_0)\bigr)
=\sqrt a\left|(2-a)^{-1/2}-2^{-1/2}\right|=O(a^{3/2}).
\]
This is an exact example of Theorem~\ref{thm:W2}, including a singular saddle component; its faster rate is specific to this linear coordinate.
Arbitrary convex combinations of the three invariant laws allow fixed or step-dependent component weights.
The all-root mixture remains defined even when the total stable weight is zero.

\section{An explicit family of selected invariant laws}\label{sec:global}\label{app:counter}

Absorbing decompositions are established in separable finite-sum SGD \citep{ShirokoffZaleski2025}. This example supplies explicit root covariances and a local equilibration certificate within our assumptions.
For $d\ge2$, let
\begin{equation}\label{eq:counter}
g(x)=(h(x_1),x_2,\ldots,x_d)^\trans,
\qquad h(u)=\frac{(u+1)u(u-2)}{1+u^2},
\end{equation}
and use independent Rademacher innovations of amplitude $\delta=1/5$ in every coordinate.
The roots are $m_-=(-1,0,\ldots,0)$, $m_0=0$ and $m_+=(2,0,\ldots,0)$.
The first-coordinate derivatives are $3/2,-2,6/5$, respectively.

\begin{proposition}[Arbitrary selected weights]\label{prop:counter}
For every $0<a\le1/6$, the recursion with~\eqref{eq:counter} has invariant laws supported in the two disjoint rectangles
\begin{equation}\label{eq:rectangles}
R_-=[-6/5,-4/5]\times[-1/5,1/5]^{d-1},\quad
R_+=[9/5,11/5]\times[-1/5,1/5]^{d-1}.
\end{equation}
Denote any such laws by $\pi_a^-$ and $\pi_a^+$.
For every function $c(a)\in[0,1]$, $c(a)\pi_a^-+(1-c(a))\pi_a^+$ is invariant.
Its local Gaussian covariances are
\begin{equation}\label{eq:counter-S}
S_-=\diag(\delta^2/3,\delta^2/2,\ldots),\qquad
S_+=\diag(5\delta^2/12,\delta^2/2,\ldots).
\end{equation}
An oscillating selection $c(a)$ can prevent even the marginal law of $Z_a$ from converging.
\end{proposition}

\begin{proof}[Proof of Proposition~\ref{prop:counter}]
Polynomial division gives
\[
h(u)=u-1+\frac{-3u+1}{1+u^2},\qquad
h'(u)=1+\frac{3u^2-2u-3}{(1+u^2)^2}.
\]
Thus $h$ is smooth, $h(u)-u$ is bounded and $|h'(u)|\le6$.
For example, bound the three numerator terms separately using
$u^2/(1+u^2)^2\le1/4$, $2|u|/(1+u^2)^2\le1$ and $(1+u^2)^{-2}\le1$.
Consequently $u\mapsto u-ah(u)+a\eta$ is nondecreasing whenever $a\le1/6$, for every $\eta\in[-1/5,1/5]$.
The exact endpoint values are
\[
h(-6/5)=-96/305<-1/5,\quad h(-4/5)=56/205>1/5,
\]
\[
h(9/5)=-63/265<-1/5,\qquad h(11/5)=88/365>1/5.
\]
At each left endpoint the update with the smallest innovation points inward; at each right endpoint the update with the largest innovation points inward.
Monotonicity then maps the entire interval into itself for every allowed innovation.
For the remaining coordinates, $x_r^+=(1-a)x_r+a\eta_r$ preserves $[-1/5,1/5]$ when $a\le1$.
Both rectangles in~\eqref{eq:rectangles} are therefore invariant.
The finite-noise transition kernel is Feller, so Ces\`aro averaging on either compact rectangle supplies an invariant probability measure supported there.

The drift is globally Lipschitz and dissipative because $g(x)=x+O(1)$.
It is the gradient of
$U(x)=\int_0^{x_1}h(u)\dd u+\frac12\sum_{r=2}^dx_r^2$.
This potential has at most quadratic growth, bounded Hessian and $\ip{\nabla U}{g}=\norm{g}^2$.
The three roots are $m_-,m_0,m_+$, with Jacobians
$\diag(3/2,1,\ldots)$, $\diag(-2,1,\ldots)$ and $\diag(6/5,1,\ldots)$.
The innovation covariance is $(1/25)I$, which excites the unstable first coordinate at the middle root.
All qualitative assumptions hold.

Each rectangle lies in its corresponding stable root's Voronoi cell.
Linearity of the stationary equation makes every stated convex combination invariant, and its cell weights are exactly $c(a)$ and $1-c(a)$.
Solving the diagonal Lyapunov equations gives~\eqref{eq:counter-S}.
For example, choosing $c(a)=(1+\sin(\log(1/a)))/2$ gives sequences tending to zero on which $c(a)$ is respectively zero and one.
Theorem~\ref{thm:mixture} identifies distinct marginal Gaussian limits on these sequences, because their first-coordinate variances differ.
This proves the final assertion.
\end{proof}

\begin{proposition}[Certified local equilibration]\label{prop:contract}
For Proposition~\ref{prop:counter} with $h(u)=(u+1)u(u-2)/(1+u^2)$ and coordinatewise Rademacher noise of amplitude $1/5$, each rectangle $R_-$ and $R_+$ contains a unique invariant law $\pi_a^\pm$ for $0<a\le1/6$. For $x\in R_\pm$, after $n$ updates,
\[
 W_2\left(\Law\left(\frac{X_n-m_\pm}{\sqrt a}\right),
 \Law_{\pi_a^\pm}\left(\frac{X-m_\pm}{\sqrt a}\right)\right)
 \le\frac{\operatorname{diam}(R_\pm)}{\sqrt a}(1-a/2)^n
 \le\frac{\operatorname{diam}(R_\pm)}{\sqrt a}e^{-an/2}.
\]
\end{proposition}
\begin{proof}
The preceding endpoint argument proves invariance. The derivative satisfies
\[
 h'(u)=\frac{u^4+5u^2-2u-2}{(1+u^2)^2}.
\]
On $[-6/5,-4/5]$, separate bounds on the numerator and denominator give
$h'(u)\ge[(4/5)^4+5(4/5)^2+8/5-2]/[1+(6/5)^2]^2=2006/3721>1/2$; on $[9/5,11/5]$ they give
$h'(u)\ge[(9/5)^4+5(9/5)^2-22/5-2]/[1+(11/5)^2]^2=12686/21316>1/2$. Together with the global bound $h'(u)\le6$, the first-coordinate update derivative lies in $[0,1-a/2]$. Every other coordinate has derivative $1-a$. Thus, under a common noise realization, the update contracts Euclidean distances within each convex rectangle by at most $1-a/2$. Existence follows from the compact Feller argument above; synchronous coupling implies uniqueness and
$W_2(P_a^n(x,\cdot),\pi_a^\pm)\le(1-a/2)^n\operatorname{diam}(R_\pm)$.
Scaling by $1/\sqrt a$ proves the statement.
\end{proof}
\noindent The heterogeneous-root experiment uses physical burn-in $an\ge 2\log[\operatorname{diam}(R_\pm)/(\epsilon\sqrt a)]$ to make its scaled initialization error at most $\epsilon$. Retained observations are temporally dependent; uncertainty is measured across independent trajectories.

\section{Global communication and component weights}\label{app:metastability}

Local moment accuracy determines the fluctuations of each component. Selecting its occupation weight requires additional communication and stationary-limit hypotheses, supplied in different settings by metastability theory.

\paragraph{Neighborhood and cell weights.}
For a stable root and sufficiently small fixed $r$, write
$p_{a,j}(r)=\pi_a(B(m_j,r))$.
The concentration theorem gives $0\le w_{a,j}-p_{a,j}(r)\le K_ra$.
This transfers order-one limits between neighborhood and Voronoi weights.
It cannot transfer exponential asymptotics: $p_a=e^{-1/a}$ and $w_a=e^{-1/a}+a$ differ by $O(a)$ but have different limits of $a\log(\cdot)$.
Likewise, an empirical time occupation from a chosen initialization is not either stationary weight unless an appropriate ergodic approximation is justified.

\paragraph{Light tails.}
For additive innovations, let $\Lambda(\theta)=\log\E e^{\theta^\trans\varepsilon}$ and let $\Lambda^*$ be its Legendre transform.
When an applicable path large-deviation theorem holds on physical time $t=na$, the candidate action is
\[
I_{0T}(\phi)=\int_0^T\Lambda^*(\dot\phi(t)+g(\phi(t)))\dd t
\]
for absolutely continuous paths, at speed $1/a$.
Under the required uniform path, communication and tightness hypotheses, directed minimum-cost trees built from restricted inter-attractor path costs describe logarithmic neighborhood occupations; see \citet{Kifer1990} and the SGD analyses of \citet{AzizianEtAl2024,AzizianEtAl2025}.
If $W_j$ is the resulting tree cost toward root $j$, the neighborhood exponent takes the form $-(W_j-\min_kW_k)$, with the neighborhood limits prescribed by the applicable theorem.
Equal tree costs do not identify the relative order-one weights: prefactors or another selection argument are needed.

For isotropic Gaussian innovations $\N(0,\sigma^2I)$ and $g=\nabla F$, the action satisfies
\[
I_{0T}(\phi)=\frac1{2\sigma^2}\int_0^T\norm{\dot\phi-\nabla F(\phi)}^2\dd t
+\frac2{\sigma^2}\{F(\phi(T))-F(\phi(0))\}.
\]
The familiar barrier cost follows when the relevant reverse-gradient connection exists, interpreted through arbitrarily long approximating paths.
A reverse-gradient solution started exactly at an equilibrium remains there, so finite-time attainment from that exact initial point is not implied.
The diffusion $\dd X=-\nabla F(X)\dd t+\sqrt a\sigma\dd W$ has a Gibbs stationary density proportional to $e^{-2F/(a\sigma^2)}$ under the usual integrability conditions.
Its Laplace prefactors do not automatically transfer to the discrete recursion studied here.

\paragraph{Regularly varying tails.}
For an additive innovation with regularly varying tail index $\alpha>2$, write $H(R)=\Pp(\norm{\varepsilon}>R)$ and suppose the rescaled tail measures converge to $\nu_{\rm tail}$ away from zero.
A displacement of order one in an update $a\varepsilon$ occurs with probability on the scale $H(1/a)$ per iteration.
The corresponding waiting scales are $1/H(1/a)$ iterations and $a/H(1/a)$ units of physical time.
If the deterministic attraction basins are $\mathcal B_j$, the one-jump candidate rates are
\[
q_{ij}=\nu_{\rm tail}(\mathcal B_j-m_i),\qquad i\ne j,\qquad
q_{ii}=-\sum_{j\ne i}q_{ij}.
\]
Turning this formula into a well-chain limit requires stability and return control, negligible boundary landings such as $\nu_{\rm tail}(\partial(\mathcal B_j-m_i))=0$, and the hypotheses of an appropriate discrete-time metastability theorem.
Convergence of stationary weights to the stationary vector of this chain additionally requires control of the stationary limit and, for a unique vector, irreducibility of the limiting chain.
The diffusion results of \citet{ImkellerPavlyukevich2008,HogelePavlyukevich2015} and discrete heavy-tail analyses of \citet{WangOhRhee2022,WangRhee2023} establish versions of these mechanisms in their respective settings.
The clipped dynamics in the former discrete reference need not have the same one-jump rates as unclipped updates.
Equation~\eqref{eq:heavy-weight} is the symmetric scalar $\alpha=3$ benchmark for the unclipped experiment, not a consequence of finite variance alone.

\section{A sharp-order example for smooth tests}\label{app:sharp}
\begin{proposition}[A bounded-noise $\sqrt a$ lower example]\label{prop:sharp}
Assumption~\ref{ass:standing}, even with bounded additive innovations and a globally linear drift, does not imply an $O(a)$ Gaussian smooth-test error. There is a one-root example whose error on $\mathcal H_{2,1}$ is at least a positive constant times $\sqrt a$ for sufficiently small $a$.
\end{proposition}
\begin{proof}
Take $g(x)=x$, $0<a<1$, and innovations
\[
 \Pr(\varepsilon=\sqrt2)=1/3,\qquad
 \Pr(\varepsilon=-1/\sqrt2)=2/3.
\]
They have mean zero, variance one, and third moment $m_3=1/\sqrt2$. The stationary scaled variable is
\[
 Z_a=\sqrt a\sum_{k\ge0}(1-a)^k\varepsilon_k.
\]
The Gaussian limit has variance $1/2$. For fixed real $t$, choose the continuous logarithm near zero with $\log\phi_\varepsilon(0)=0$. Boundedness of the innovations permits a uniform fourth-order remainder in the logarithm of their characteristic function near zero. Summing that expansion over $\sqrt a(1-a)^kt$ gives
\[
 \log\E e^{itZ_a}
 =-\frac{t^2}{2(2-a)}
 -\frac{i m_3t^3\sqrt a}{6(3-3a+a^2)}+O(a),
\]
since $\sum_{k\ge0}a^2(1-a)^{4k}=O(a)$. Hence
\[
 \E\sin(tZ_a)=-e^{-t^2/4}\frac{m_3t^3}{18}\sqrt a+O(a).
\]
At $t=1$, the bounded test $h(z)=\sin z/4$ belongs to $\mathcal H_{2,1}$: each of its sup-norm, gradient norm, Hessian norm, and Hessian Lipschitz seminorm is at most $1/4$. Its Gaussian expectation is zero, so it yields the claimed lower bound. This example also explains why symmetric-noise or covariance-only experiments may fail to display the worst-case $\sqrt a$ order.
\end{proof}

\section{A controlled neural finite-sum witness}\label{app:neural}
A fixed linear tilt makes nondegeneracy generic, while bounded coordinate dither excites every direction. These devices define the augmented finite-sum model used in the neural experiment.

\begin{proposition}[Generic tilt and bounded full-rank finite-sum noise]\label{prop:neural}
Fix a finite dataset $(x_i,y_i)_{i=1}^N$ with $y_i\in\{-1,1\}$, fixed output weights $c$, and a network $f_W(x)=c^T\tanh(Wx)$ in which only $W$ is trained. Write $\theta=\operatorname{vec}(W)\in\R^p$ and
\[
 \ell_i(\theta)=\log(1+e^{-y_i f_W(x_i)}),\quad
 F_0(\theta)=\frac1N\sum_i\ell_i(\theta)+\frac\lambda2\norm\theta^2,
 \qquad\lambda>0.
\]
Draw $r\in\R^p$ from an absolutely continuous distribution \emph{once}, and hold it fixed. For fixed $\tau,\gamma>0$, define
\[
 F_r(\theta)=F_0(\theta)+\tau r^T\theta,
\]
and the augmented finite sum, indexed uniformly by $i\in\{1,\ldots,N\}$, $k\in\{1,\ldots,p\}$, $\sigma\in\{-1,1\}$,
\[
 \widetilde\ell_{i,k,\sigma}(\theta)
 =\ell_i(\theta)+\frac\lambda2\norm\theta^2+\tau r^T\theta
   +\gamma\sqrt p\,\sigma\theta_k.
\]
For almost every fixed $r$, iid finite-sum SGD on these $2Np$ equally weighted summands satisfies Assumption~\ref{ass:standing} and excitation at every unstable root, in every fixed finite dimension. Its mean objective is $F_r$, and its one-sample noise covariance is
\[
 C(\theta)=\operatorname{Cov}_{i}(\nabla\ell_i(\theta))+\gamma^2 I_p.
\]
All local innovation moments exist, and the covariance is locally Lipschitz. The theorem constants may depend on the fixed tilt and root conditioning; no uniform spectral gap over random tilts is asserted.
\end{proposition}
\begin{proof}
Because the dataset and $c$ are fixed and finite, $f_W(x_i)$ and its derivatives of any fixed positive order in $\theta$ are globally bounded: derivatives of $\tanh$ of any fixed order are bounded, and each differentiation contributes fixed input factors. The logistic loss is smooth; on the bounded logit range its derivatives are bounded. Thus every $\ell_i$ has globally bounded derivatives of all fixed positive orders.

Consequently $g_r=\nabla F_r$ is smooth, globally Lipschitz, and of the form $\lambda\theta+O(1)$ for fixed $r$. It is dissipative, $F_r$ is coercive with at most quadratic growth, its Hessian is globally bounded, and $U=F_r$ satisfies $\langle\nabla U,g_r\rangle=\norm{g_r}^2$.

Apply Sard's regular-value theorem to the smooth map $G=\nabla F_0:\R^p\to\R^p$. Its critical values have Lebesgue measure zero. Since $-\tau r$ has an absolutely continuous law, it is a regular value with probability one. Every root of $g_r$ therefore has nonsingular Hessian. The Hessian is real symmetric, so these roots are hyperbolic for the gradient flow. The root set is closed, and it is bounded because $G(\theta)=\lambda\theta+O(1)$. All its points are isolated by the inverse function theorem; a compact set of such regular roots is finite (otherwise an accumulation point would contradict local invertibility). Coercivity supplies a minimizer, so the root set is nonempty.

The gradient of each augmented summand is a continuous function. Finite averaging therefore gives a Feller transition. The ridge and fixed tilt cancel when forming centered gradient noise. The remaining data gradient fluctuation and the coordinate dither are uniformly bounded, centered, and independent when their indices are sampled independently. The dither covariance is
$\gamma^2p\,\E(e_ke_k^T)=\gamma^2I_p$, so the displayed covariance formula follows. In particular, for every nonzero complex vector $v$, $v^*C(\theta)v\ge\gamma^2\norm v^2>0$. This implies the required unstable-direction excitation at every unstable root. Bounded noise supplies the covariance growth bound with $c_1=0$ and every local innovation moment. Smooth finite-sum gradients give covariance continuity and local Lipschitz regularity.
\end{proof}
\noindent The regular-value result is Sard's theorem \citep{Lee2013}. It verifies the augmented model's structural assumptions almost surely in the fixed tilt. The unaugmented arm requires separate verification, and finite-time equilibration and positive component weights remain separate questions in either model.

For a batch of size $B$, sample $B$ independent triples $(i,k,\sigma)$ and average their complete gradients, including the dither. Then $C_B=(C_{\rm data}+\gamma^2I)/B$.
\section{Local fluctuations and global communication}\label{app:legacy}
\subsection{Matched-covariance local fluctuations}

Consider
\begin{equation}\label{eq:fourwell}
\begin{gathered}
F(x)=\sum_{r=1}^2\{x_r^2/2-c_r\log\cosh x_r\},\qquad(c_1,c_2)=(2,1.5),\\
g=(I+0.8J)\nabla F,\qquad J=\begin{pmatrix}0&-1\\1&0\end{pmatrix}.
\end{gathered}
\end{equation}
There are four attracting roots $(\pm m_1,\pm m_2)$, with $m_1=1.91500805$ and $m_2=1.28783945$.
All four have the same local Jacobian and innovation covariance,
\begin{equation}\label{eq:numerical-S}
A=\begin{pmatrix}0.83362791&-0.48454958\\0.66690233&0.60568697\end{pmatrix},\quad
C=\begin{pmatrix}1&0.4\\0.4&2\end{pmatrix},\quad
S=\begin{pmatrix}0.79678347&0.33891463\\0.33891463&1.27784990\end{pmatrix}.
\end{equation}
We match $C$ across Gaussian, Rademacher and standardized Student $t_3$ innovations.
Each of five step sizes uses 16 independent runs of physical duration 5,000, with duration 50 discarded.
We assign samples to the nearest of all nine roots and pool stable marked displacements.
Pooling is valid here because their theoretical Gaussian components coincide.

\begin{figure}[t]
\centering\includegraphics[width=\linewidth]{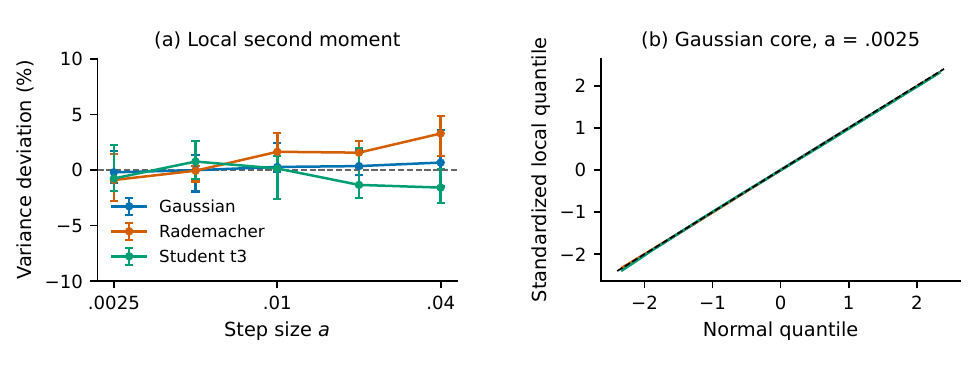}
\caption{Local fluctuations with matched innovation covariance. Variance deviations in (a) are percentages relative to the OU target on a zero-centered axis.
(a) Median variance deviation and interquartile range across 16 runs; the continuous OU prediction is zero.
(b) Pooled first-coordinate Q--Q curves at $a=0.0025$.
The intervals describe between-run dispersion.}
\label{fig:local}
\end{figure}

Figure~\ref{fig:local} shows central Q--Q curves close to the common Gaussian reference at the smallest step, alongside appreciable run-to-run variance dispersion.

\paragraph{Finite-step reference.}
An exactly linear recursion has scaled stationary covariance
\begin{equation}\label{eq:discrete-cov}
S_a=(I-aA)S_a(I-aA)^\trans+aC,
\quad AS_a+S_aA^\trans-aAS_aA^\trans=C.
\end{equation}
This reference isolates linear discretization effects; it is not the nonlinear chain's covariance.
For gradient drift, the corresponding quadratic loss admits a simple relative bound.
\begin{proposition}[Discrete correction for a frozen quadratic model]\label{prop:quadratic-discrete}
Let $H=H^T\succ0$, let $C\succeq0$ be the covariance of centered iid additive innovations, and suppose $\epsilon=a\lambda_{\max}(H)<2$.
For the stationary linear recursion $Y^+=(I-aH)Y+a\xi$, define
$L_a=\tfrac12\E(Y^THY)$ and $L_{\rm OU}=a\tr C/4$.
Then
\[
0\le L_a-L_{\rm OU}\le\frac{\epsilon}{2-\epsilon}L_{\rm OU}.
\]
If $\tr C>0$, division by $L_{\rm OU}$ gives the same bound on the relative correction.
\end{proposition}
\begin{proof}
All eigenvalues of $I-aH$ have modulus less than one, so the stationary covariance $V_a$ exists and satisfies
$HV_a+V_aH-aHV_aH=aC$.
In an orthonormal eigenbasis of $H$, write its eigenvalues as $\lambda_i$ and the diagonal entries of $C$ as $c_i\ge0$.
The diagonal covariance equations give $(V_a)_{ii}=ac_i/[\lambda_i(2-a\lambda_i)]$, hence
\[
L_a=\frac a2\sum_i\frac{c_i}{2-a\lambda_i},\qquad
L_a-L_{\rm OU}=\frac a4\sum_i c_i\frac{a\lambda_i}{2-a\lambda_i}
\le\frac{\epsilon}{2-\epsilon}\frac a4\sum_i c_i.
\]
This calculation requires no commutation between $H$ and $C$.
\end{proof}
At the largest neural normalized step, $\epsilon=0.005$, the bound is $0.2507\%$ after rounding upward.
It quantifies the loss correction within the frozen quadratic reference; the nonlinear objective is evaluated separately in the SGD experiment.

Student $t_3$ noise has an infinite fourth moment.
The Q--Q comparison and bounded smooth probes in Appendix~\ref{app:numerics} assess distributional shape without requiring fourth moments or fitting a convergence-rate exponent.

\subsection{Communication and finite-horizon occupation}

For $g(x)=x+5\cos x$, the stable roots are $m_L=-1.30644001$ and $m_R=3.83746711$, separated by $s=1.97738303$.
We again compare three innovation laws of variance one, using 12 independent runs per step and noise law.
Half start at each minimum.
For standardized $t_3$ noise, the one-large-jump prediction is
\begin{equation}\label{eq:heavy-weight}
q_{LR}\propto(s-m_L)^{-3},\quad q_{RL}\propto(m_R-s)^{-3},\quad
w_L^{\rm jump}=\frac{q_{RL}}{q_{LR}+q_{RL}}=0.84620744.
\end{equation}
We choose heavy-tail run lengths before simulation using the jump proxy, targeting 400 completed transitions in total at each step size.
Gaussian and bounded-noise comparators use two million iterations per run.
The horizons are chosen to resolve transitions and occupations under each noise law.

\begin{figure}[t]
\centering\includegraphics[width=\linewidth]{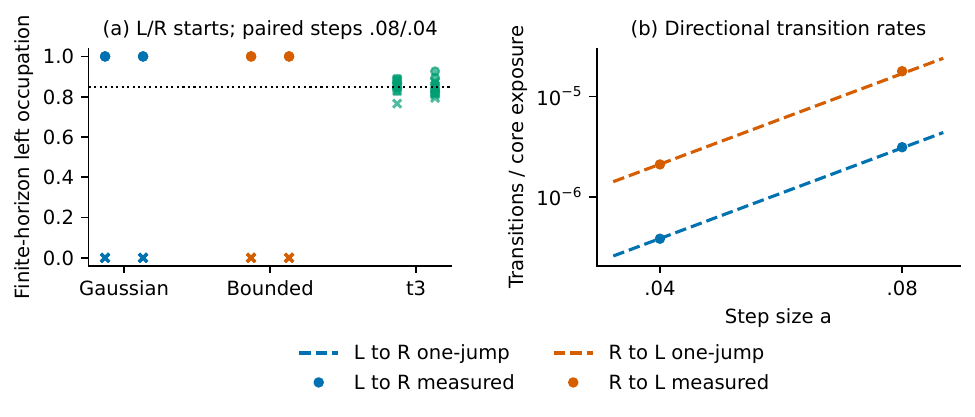}
\caption{Global occupation and transition rates.
(a) Per-run left occupation, with paired positions for $a=0.08$ and $0.04$; circles start left and crosses start right.
The dotted line is the heavy-tail limit~\eqref{eq:heavy-weight}.
(b) Heavy-tail completed transitions divided by exposure to the last visited well core, compared with one-jump probabilities per iteration.
Run lengths are reported in Table~\ref{tab:switching}.}
\label{fig:global}
\end{figure}

We distinguish separatrix crossings from completed transitions between radius-$0.25$ well cores.
Figure~\ref{fig:global} shows that local covariance matching does not imply similar finite-horizon basin exploration.
For runs without transitions, averaging equally many initializations in both wells returns $1/2$ by design.
That average does not identify a stationary weight.
Conversely, a Gaussian run without observed transitions does not establish noncommunication.
Detailed counts and horizons appear in Table~\ref{tab:switching}.

\section{Numerical methods and additional measurements}\label{app:numerics}

The supplementary local-shape, occupation and output-coverage studies comprise 240, 72 and 192 seeded trajectories, respectively. Appendix~\ref{app:revision-numerics} specifies the root-specific moment and objective-gap measurements.
The files \texttt{code/reproduce.py} and \texttt{code/simulation.py} generate these measurements directly from the recursions.
Running the first script without arguments regenerates figures and tables from the shipped data; adding \texttt{--run} also regenerates the simulations.
Per-run CSV files retain the seeds and estimators; compressed arrays retain the local Q--Q samples and running occupation curves.
Version and duration information is recorded in \texttt{data/metadata.json}.

\subsection{Local-noise experiment}

The roots of the four-well system are the nine combinations of $\{-m_1,0,m_1\}$ and $\{-m_2,0,m_2\}$.
The four combinations with nonzero coordinates are stable; the other roots are unstable.
The scalar curvatures at $m_1,m_2$ are $1-c_r\operatorname{sech}^2(m_r)$, giving~\eqref{eq:numerical-S}.
Here $U=F$ is a valid descent function because
$\ip{\nabla F}{(I+0.8J)\nabla F}=\norm{\nabla F}^2$ and
$\norm{(I+0.8J)\nabla F}^2=(1+0.8^2)\norm{\nabla F}^2$.
Moreover, $\nabla F=x+O(1)$, so $\ip{x}{g(x)}=\norm{x}^2+O(\norm{x})$.
These identities verify descent and dissipativity for this nonreversible example.

Each step in $(0.04,0.02,0.01,0.005,0.0025)$ uses 16 independent trajectories starting at the positive minimum.
The total iteration count is $\lfloor5000/a\rfloor$, the burn-in count is $\lfloor50/a\rfloor$, and retained samples are spaced by $\lfloor0.5/a\rfloor$ iterations.
Thus the target retained physical spacing is $0.5$ (it is $0.48$ at $a=0.04$ owing to integer rounding).
Uncertainty and dispersion are summarized at the independent-run level.

Draw two independent scalar innovations of variance one, either standard normal, Rademacher, or $t_3/\sqrt3$, and multiply them by
$L=\left(\begin{smallmatrix}1&0\\0.4&\sqrt{1.84}\end{smallmatrix}\right)$.
This gives $LL^\trans=C$ for all three laws.
The $t_3$ law has finite $(2+\delta)$ moments only for $\delta<1$ and has infinite third absolute and fourth moments.
The qualitative theorem applies; the moment-based additive rate theorem applies for any $0<\delta<1$, while Corollary~\ref{cor:symmetric-rate} gives the sharper $O(\sqrt a)$ smooth-test bound because this noise is symmetric and $C\succ0$.

All labeling uses the nearest of all nine roots.
Within a run, retain stable-labeled displacements $Z=(X-X_\Mcal)/\sqrt a$ and use the usual centered sample covariance with denominator $n-1$.
The retained stable fraction is at least $0.9795$ over these runs; other labels are excluded and their fractions are recorded.
The four stable components have the same $S$, so pooling them has a common local target without estimating global weights.
Figure~\ref{fig:local}(a) reports the median and interquartile range of the run-level variance estimates.
Panel (b) pools retained first-coordinate samples at the smallest step and divides by $\sqrt{S_{11}}$; the reference quantiles range from $0.01$ to $0.99$.
These are finite-horizon measurements of local fluctuations.

We also use bounded smooth distributional probes at six fixed frequencies:
\[
\Theta=\{(1,0),(0,1),(1,1),(1,-1),(2,0),(0,2)\}.
\]
\begin{equation}\label{eq:probe}
\widehat e=\max_{\theta\in\Theta}\left|
\frac1n\sum_{k=1}^n e^{i\theta^\trans Z_k}-e^{-\theta^\trans S\theta/2}\right|.
\end{equation}
Real and imaginary parts are bounded smooth tests; fixed-frequency normalization places them within a constant multiple of the theorem's test class.
This fixed-frequency diagnostic contains both Monte Carlo variation and finite-step error.
Figure~\ref{fig:additional}(a) reports run-level medians and interquartile ranges, without fitting a rate exponent.

\subsection{Communication experiment and transition budget}

For $g(x)=x+5\cos x$, use independent scalar innovations from the same three variance-one laws.
At each $a\in\{0.08,0.04\}$, use 12 runs, alternating their initial states between $m_L$ and $m_R$.
The first 100,000 iterations are discarded.
For the $t_3/\sqrt3$ innovation, define the exact one-step tail probabilities at the minima
\[
p_{LR}(a)=\Pp(t_3>\sqrt3(s-m_L)/a),\qquad
p_{RL}(a)=\Pp(t_3>\sqrt3(m_R-s)/a).
\]
Under the two-state proxy the stationary mean number of transitions per iteration is
$r(a)=2p_{LR}(a)p_{RL}(a)/(p_{LR}(a)+p_{RL}(a))$.
Before running any trajectory, set the post-burn-in length to
\[
N_{\rm post}(a)=10^5\left\lceil\frac{400}{12r(a)10^5}\right\rceil.
\]
This gives 6,500,000 and 51,400,000 post-burn-in iterations per heavy-tail run at the two steps.
The corresponding total lengths are 6,600,000 and 51,500,000.
The proxy fixes each run's horizon before observing any realized transition counts.
Gaussian and Rademacher comparisons use 2,000,000 total iterations per run, including the same burn-in.

Left occupation is the fraction of post-burn-in iterates with $X_n<s$.
Raw crossings count changes in that indicator.
Completed transitions instead track the last visited radius-$0.25$ core about a minimum: the well label changes only when the opposite core is reached.
Its state is tracked throughout burn-in; reported counts and exposure use only the retained time interval.
Directional rate estimates divide completed $i\to j$ transitions by the number of iterations whose pre-update last-core label is $i$.
They are per-iteration rates and are compared with $p_{ij}(a)$, not directly with rates in a normalized limiting clock.

\begin{table}[t]
\centering\small
\caption{Occupation experiment: 12 runs per row. $N$ includes burn-in; counts sum post-burn-in events over runs.
Heavy-tail occupation is the mean with one across-run standard error in parentheses.
$^*$The value $0.500$ is fixed by six starts in each well and no observed crossings; it is not a stationary-weight estimate.}
\label{tab:switching}
\begin{tabular}{lrrrrr}
\toprule
Noise & $a$ & $N$ per run & Left occupation & Raw & Completed \\
\midrule
Gaussian & 0.08 & 2,000,000 & $0.500^{*}$ & 0 & 0 \\
Gaussian & 0.04 & 2,000,000 & $0.500^{*}$ & 0 & 0 \\
Rademacher & 0.08 & 2,000,000 & $0.500^{*}$ & 0 & 0 \\
Rademacher & 0.04 & 2,000,000 & $0.500^{*}$ & 0 & 0 \\
Student t3 & 0.08 & 6,600,000 & $0.852 \;(0.010)$ & 477 & 413 \\
Student t3 & 0.04 & 51,500,000 & $0.845 \;(0.011)$ & 522 & 400 \\
\bottomrule
\end{tabular}
\end{table}

The heavy-tail runs produce 413 completed transitions at $a=0.08$ and 400 at $a=0.04$ (raw crossings: 477 and 522).
Their mean left occupations are $0.8517$ and $0.8448$, with across-run standard errors $0.0099$ and $0.0113$.
The smaller-step estimate is close to the asymptotic benchmark $0.8462$ at this precision.
Two step sizes and these horizons do not identify an asymptotic bias exponent.
Zero observed crossings in the other noise laws describe the sampled horizon, while Proposition~\ref{prop:counter} separately proves exact noncommunication in a different explicit bounded-noise model.

\subsection{Finite-sum SGD and prediction functions}

For the four losses in~\eqref{eq:finite-sum}, each iteration samples $B$ independent sign pairs with replacement and averages their gradients.
The simulator uses the equivalent symmetric centered-noise sign convention in~\eqref{eq:recursion}.
For each $B\in\{1,4,16\}$ and $a\in\{0.04,0.02,0.01,0.005\}$, run 16 trajectories with the same total physical duration, burn-in and thinning as the local experiment.
All start at the positive minimum and all retained samples in these runs remain in its nearest-root cell.
This is a local finite-horizon study and does not infer their global stationary weights.
For this gradient drift, the local covariance solves $HS_1+S_1H=C_1(m)$.

Fix $f_1(x)=x_1+x_2$ and $f_2(x)=\tanh(x_1-x_2/2)$ before simulation.
Let $u=(1,-1/2)^\trans$ and let $G$ have rows $(1,1)$ and $\operatorname{sech}^2(u^\trans m)u^\trans$.
The predicted covariance of $\sqrt{B/a}\{f(X)-f(m)\}$ is $GS_1G^\trans$.
For each run, measure the centered sample variance ratio against the corresponding diagonal entry and the fraction satisfying
\[
\left|\sqrt{B/a}\{f_k(X)-f_k(m)\}\right|
\le1.9599639845\sqrt{(GS_1G^\trans)_{kk}}.
\]
This last quantity is nominal Gaussian fluctuation coverage around a deterministic local value; it is not coverage of a confidence interval for a population parameter.
Neither $S_1$ nor $G$ is estimated from trajectories.
Figure~\ref{fig:additional}(b) gives the nonlinear observable's run-averaged coverage; individual-run values range from $0.9379$ to $0.9557$.
The parameter covariance, both observable variance ratios and both coverages are retained in \texttt{data/sgd.csv}.

\begin{figure}[t]
\centering\includegraphics[width=\linewidth]{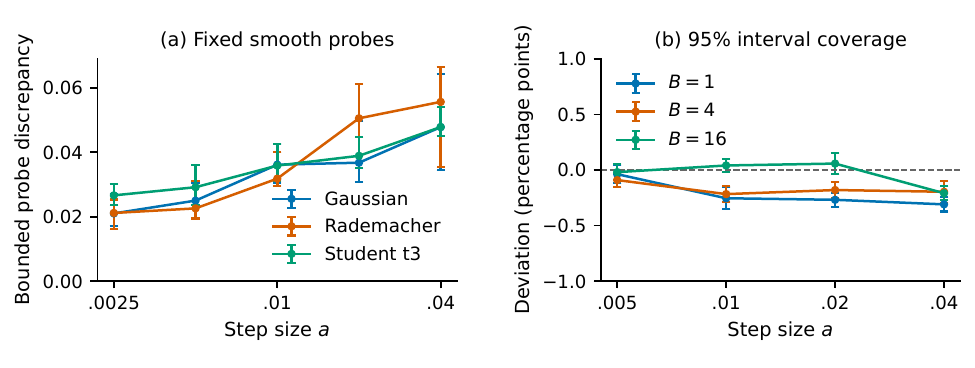}
\caption{Additional local measurements. Coverage in (b) is shown as percentage-point deviation from 95 percent on a zero-centered axis.
(a) Bounded characteristic-function discrepancy~\eqref{eq:probe}, shown as median and interquartile range across 16 runs.
(b) Nominal $95\%$ local fluctuation coverage for the nonlinear observable, shown as mean and one standard error across runs.
These finite-horizon measurements describe distributional agreement and coverage; the asymptotic rates are established by the analytic results.}
\label{fig:additional}
\end{figure}

\paragraph{Seed allocation.}
Indices start at zero in the listed parameter order.
Local seeds are $270900+10000k+100i+r$, where $k$ indexes (Gaussian, Rademacher, Student $t_3$), $i$ indexes the five descending steps and $r=0,\ldots,15$.
Occupation seeds are $370900+10000k+100i+r$, with steps $(0.08,0.04)$ and $r=0,\ldots,11$.
SGD seeds are $470900+10000b+100i+r$, where $b$ indexes batch sizes $(1,4,16)$, $i$ indexes the four descending steps and $r=0,\ldots,15$.
Random draws use NumPy's random functions within Numba-compiled loops; the shipped dependency versions are recorded to support reruns.
Every standardized $t_3$ draw is constructed as $N_0/\sqrt{N_1^2+N_2^2+N_3^2}$ from four independent standard Gaussian variables.
This has law $t_3/\sqrt3$ and variance one; the denominator is not truncated at a fixed positive threshold.

\section{Experimental protocols}\label{app:revision-numerics}

\paragraph{Estimands and independent replication.}
Covariance and loss studies use 16 independently seeded trajectories per setting, mean confirmation uses 32, and the supplementary occupation study uses 12.
For $R$ independent run-level statistics $T_r$, the reported standard error is $s(T_r)/\sqrt R$, with sample variance denominator $R-1$.
CDF bands use the 16 run-level ECDFs pointwise and are not simultaneous bands.
For infinite-fourth-moment innovations, raw quadratic estimates are displayed as independent-run distributions or medians and interquartile ranges.
The fixed neural probes are averaged within each run.
Mean absolute relative error (MARE) for path $r$ is $K^{-1}\sum_k|\widehat v_{r,k}/v_k-1|$; its reported mean averages this statistic across paths. Aggregate MARE is $K^{-1}\sum_k|R^{-1}\sum_r\widehat v_{r,k}/v_k-1|$.
The signed error averages the latter ratios without absolute values.
Root-centered second moments use the deterministic-root output instead of the trajectory mean.

\paragraph{Heterogeneous invariant components.}
The two-dimensional recursion~\eqref{eq:counter} uses independent noise coordinates $\pm0.2$, initialized at $m_-$ or $m_+$.
The step grid is $0.04,0.02,0.01,0.005,0.0025$.
The burn-in count is
\[
n_{\rm burn}=\left\lceil\frac2a\log\frac{0.4\sqrt2}{10^{-6}\sqrt a}\right\rceil.
\]
Proposition~\ref{prop:contract} bounds the scaled $W_2$ initialization error by $10^{-6}$.
The retained physical exposure is $a\lceil5000/a\rceil$; states are saved every $\max\{1,\operatorname{round}(0.5/a)\}$ updates.
The seed is $27092700+100000i+1000k+r$, for zero-based root, step and replicate indices.
The full objective is
\[
F(x)=\tfrac12x_1^2-x_1-\tfrac32\log(1+x_1^2)+\arctan x_1+\tfrac12x_2^2.
\]
Both roots satisfy $H_jS_j=0.02I_2$, so both loss limits are $0.01\chi_2^2$ with mean $0.02$.
At $a=0.0025$, scaled mean losses are $0.020089\pm0.000074$ and $0.019990\pm0.000072$.
For selected left-component weights $0.25,0.5,0.75$, weighted within-component variances are recorded separately from temporal occupation estimates.

\paragraph{Analytical shape error.}
For $X_{n+1}=(1-a)X_n+a\varepsilon_{n+1}$, the scaled stationary characteristic function is
\[
\phi_a(t)=\prod_{k\ge0}\phi_\varepsilon\bigl(t\sqrt a(1-a)^k\bigr).
\]
Gaussian, Rademacher, standardized $t_3$ and the skew law of Proposition~\ref{prop:sharp} all have variance one.
Standardized $t_3$ has density $2/[\pi(1+x^2)^2]$ and characteristic function $(1+|t|)e^{-|t|}$.
Eight steps halve from $0.08$ to $0.000625$, and $t\in\{0.5,1,2\}$, giving 96 deterministic evaluations.
Tests are divided by $1+|t|+t^2+|t|^3$ to belong to $\mathcal H_{2,1}$.
Stopping before index $K$ omits an independent centered sum of variance $(1-a)^{2K}/(2-a)$.
The characteristic-function truncation error is at most
$t^2(1-a)^{2K}/[2(2-a)]$, which is set below $10^{-12}$.
The implementation sums complex logarithms and exponentiates; the bound concerns analytic truncation.
These individual test errors differ from the supremum over the test class.
The common exact scaled covariance is $1/(2-a)$.

\begin{figure}[t]
\centering\includegraphics[width=\linewidth]{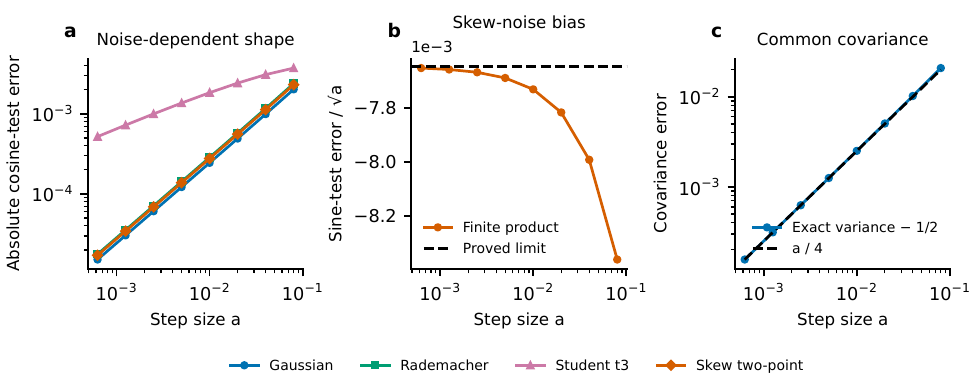}
\caption{\textbf{Distributional shape and covariance at $t=1$.} Normalized cosine discrepancies (a), skew sine bias divided by $\sqrt a$ (b), and covariance error (c). Curves are deterministic characteristic-function calculations with analytic truncation error below $10^{-12}$. The nonzero skew limit distinguishes shape error from covariance discretization.}
\label{fig:revision-e2}
\end{figure}

\paragraph{State-dependent finite-sum learning.}
At the positive root $m=(1.9150080482,1.2878394550)$ of~\eqref{eq:finite-sum},
\[
H=\diag(0.8336279122,0.6056869745),\qquad
S_1=\begin{pmatrix}0.1518851672&0.0642388947\\0.0642388947&0.0475438154\end{pmatrix}.
\]
The 192 output/loss trajectories use $B\in\{1,4,16\}$ and $a\in\{0.04,0.02,0.01,0.005\}$.
Each discards $\lceil50/a\rceil$ updates and retains physical exposure 5000 with saving interval approximately 0.5.
Seeds are $28092700+100000i+1000k+r$ in the listed batch and step orders.
Each update averages $B$ independent complete component gradients with replacement.
All nine roots enter the nearest-root label; every saved observation in this experiment remained in the positive-root cell.
For $f=(x_1+x_2,\tanh(x_1-x_2/2))$,
\[
Df(m)S_1Df(m)^T=\begin{pmatrix}0.3279067719&0.0433471556\\0.0433471556&0.0072842122\end{pmatrix}.
\]
The predicted scaled mean loss is $0.0777061923$; the quadratic-form eigenvalues are $0.0108051295$ and $0.1446072551$.
Gaussian polar coordinates give, for $q\ge0$,
\[
\Pp\{\tfrac12(\mu_1G_1^2+\mu_2G_2^2)\le q\}
=1-\frac2\pi\int_0^{\pi/2}\exp\!\left\{-\frac{q}{\mu_1\cos^2\theta+\mu_2\sin^2\theta}\right\}\,d\theta.
\]
We evaluate this expression by 256-node Gauss--Legendre quadrature.

\begin{figure}[t]
\centering\includegraphics[width=\linewidth]{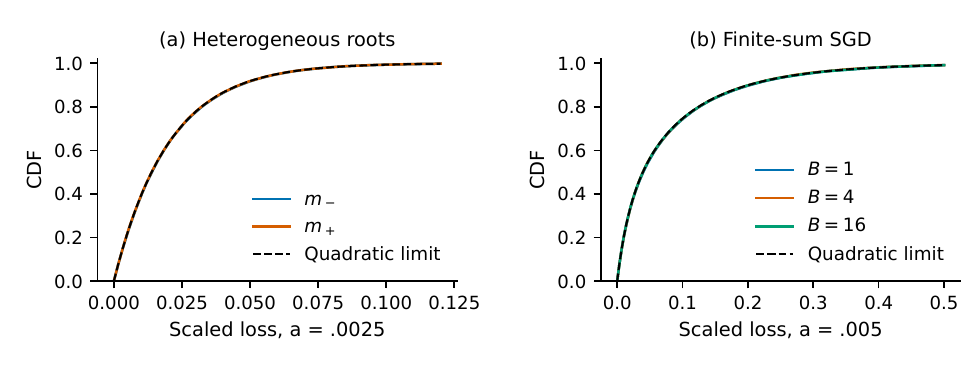}
\caption{\textbf{Loss distributions beyond their means.} (a) Actual nonlinear losses divided by $a$ in the heterogeneous rectangles at $a=0.0025$, with the shared $0.01\chi_2^2$ limit. (b) Actual finite-sum losses multiplied by $B/a$ at $a=0.005$, with the generalized chi-square limit computed by deterministic quadrature. Bands are pointwise one standard error across 16 independent runs. The corresponding moment curves appear in Figures~\ref{fig:root-moments} and~\ref{fig:observable-moments}.}
\label{fig:revision-e1}\label{fig:revision-e3}
\end{figure}

\paragraph{Neural data and root selection.}
The UCI optical digits training and test files have disjoint writers.
Pixel centering/scaling, the 16-dimensional PCA basis and whitening scales are fitted on the training subset only.
The network trains a $4\times16$ matrix without bias parameters and fixes output weights $(0.6,-0.8,1.1,-1.3)$.
The baseline uses $\lambda=0.02$, $\tau=10^{-4}$, $\gamma=0.015$ and $r\sim\N(0,I_{64}/64)$ with seed 27010001.
Root search uses 20 Gaussian initializations of scale $0.2$ (seed 27020001), BFGS and exact-Hessian Newton refinement.
Candidates require gradient norm below $10^{-8}$ and minimum Hessian eigenvalue above $10^{-6}$; the qualified candidate with smallest objective is selected.
All qualified candidates found coincide within $10^{-4}$.
At that root, training accuracy is $99.74\%$, test accuracy $97.20\%$ and test logistic loss $0.13946$.
These metrics describe the deterministic root.

The first six test examples of each class define 12 fixed probes.
The logit Jacobians are evaluated at the root, and the one-sample covariance uses the full training sum with denominator $N$, plus $\gamma^2I$.
Its trace is $0.34814885$, giving scaled mean loss $0.08703721$.
Gradient and Hessian formulas are analytic; a directional Hessian check at the root gives maximum discrepancy $7.1\times10^{-12}$.
An independent eight-trajectory pilot fixed the evaluation range before the evaluation seeds were used; its settings and records accompany the code.

\paragraph{Neural sampling.}
The normalized steps are $0.000625,0.00125,0.0025,0.005$ and batches are $1,4,16$.
Physical burn-in is 600 and retained exposure is approximately 2000.
States are saved every $\max\{1,\operatorname{round}(1/a)\}$ updates, giving 2001 saved states with exact spacing $0.9996112563352947$ at every setting.
Full objective loss is evaluated every tenth save, giving 201 values per run.
Seeds are $29092700+100000i+1000k+r$ in ascending batch/step order.
All saved states lie within Euclidean distance $0.75$ of the fixed root, documenting the local region explored during measurement.
The initialization certificate of Proposition~\ref{prop:contract} applies only to the invariant-rectangle experiment.

\paragraph{Covariance references.}
Full covariance solves $HS+SH=C_1$.
Diagonal-noise and isotropic-noise references replace $C_1$ by $\diag(C_1)$ and $(\tr C_1/d)I$, keeping the same Hessian and output Jacobian.
The discrete-linear reference solves $(I-aH)S_a(I-aH)^T+aC_1=S_a$.
Table~\ref{tab:revision-baselines} compares these references at each problem's smallest step.
At the largest neural step with $B=1$, full and discrete-linear logit MARE are $47.65\%$ and $47.41\%$.
\begin{table}[t]
\centering\small
\caption{Covariance-reference comparison. Output-variance MARE (percent), mean $\pm$ one standard error across 16 runs, at $a=0.005$ for the two-dimensional finite sum and $a\lambda_{\max}(H)=0.000625$ for the neural model. Each problem uses its fixed root Hessian and output Jacobian. The discrete-linear reference is evaluated for the neural model.}
\label{tab:revision-baselines}
\begin{tabular}{llrrrr}
\toprule
Problem & $B$ & Full & Diagonal noise & Isotropic noise & Discrete linear \\
\midrule
Finite sum & 1 & $2.32\pm0.37$ & $52.40\pm0.37$ & $34.70\pm0.30$ & -- \\
Finite sum & 4 & $1.74\pm0.23$ & $51.58\pm0.36$ & $34.03\pm0.30$ & -- \\
Finite sum & 16 & $1.34\pm0.26$ & $52.36\pm0.27$ & $34.79\pm0.23$ & -- \\
Neural & 1 & $6.97\pm0.47$ & $228.29\pm4.15$ & $184.92\pm3.49$ & $6.98\pm0.47$ \\
Neural & 4 & $6.31\pm0.36$ & $233.11\pm3.50$ & $189.61\pm2.99$ & $6.31\pm0.37$ \\
Neural & 16 & $6.09\pm0.36$ & $227.63\pm2.49$ & $185.48\pm2.16$ & $6.09\pm0.36$ \\
\bottomrule
\end{tabular}
\end{table}

The supplement provides complete configurations, per-run CSV/JSON records, selected arrays and figure-generation code. Its README documents arrays omitted for the upload-size limit and the steps to regenerate them from the original seeds and budgets.
The experiments use Python 3.12.14, NumPy 2.5.3, SciPy 1.18.1 and Numba 0.67.0 with one BLAS thread.
Dataset files, attribution and download hashes are included.
\section{Moment measurements and neural diagnostics}\label{app:moment-studies}

\begin{figure}[t]
\centering\includegraphics[width=\linewidth]{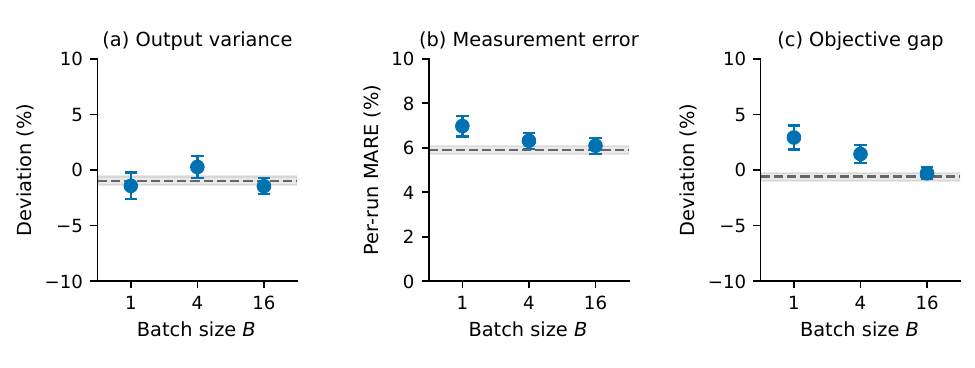}
\caption{\textbf{Small-step neural predictions.} At $a\lambda_{\max}(H)=0.000625$, (a) logit-variance and (c) actual objective-gap deviations use a common percentage scale. (b) Per-run MARE measures absolute relative error across the 12 probes. Bars summarize 16 paths; dashed lines and bands show the mean and one standard error of 64 matched stationary OU paths. The OU loss is quadratic. The main-text Figure~\ref{fig:neural-range} reports all four steps.}
\label{fig:neural-moments}
\end{figure}

\subsection{Exact OU measurement references}
For the scaled variable $Z=\sqrt{B/a}(X-m)$, the neural reference solves $HS+SH=C_1$.
At save interval $\Delta t$, the exact stationary OU skeleton is
\[
Z_{k+1}=D Z_k+\eta_k,\qquad D=e^{-H\Delta t},\qquad
\eta_k\sim\N(0,S-DSD^T),\qquad Z_0\sim\N(0,S).
\]
Each reference uses 64 independent paths, 2001 saves, $\Delta t=0.9996112563352947$, and loss evaluation every tenth save.
No initialization transient is present in this stationary reference.
The matrices and save schedules coincide across all baseline step/batch settings after scaling, so the same reference is shared across those comparisons.
The four augmentation arms each have their own reference covariance.
Seeds are $30092700+10000i$ for arm indices $i=0,1,2,3$; Gaussian streams within each arm generate 64 independent paths.
Four groups of 16 paths also provide replicates of the aggregated-variance statistic.

For a stationary scalar process with variance $v$ and lag covariance $\gamma_k$, the centered sample variance satisfies
\begin{equation}\label{eq:measurement-bias}
\E\widehat v=v-\frac{2}{n(n-1)}\sum_{k=1}^{n-1}(n-k)\gamma_k.
\end{equation}
To see this, use $\sum_i\E(Y_i-\bar Y)^2=nv-n\operatorname{Var}(\bar Y)$ and expand the variance of the sample mean.
For Gaussian samples with covariance matrix $K_n$ and $P=I-\mathbf1\mathbf1^T/n$,
$\operatorname{Var}(\widehat v)=2\tr(PK_nPK_n)/(n-1)^2$.
Using the frozen neural matrices, the 12 linearized probes have relative centering biases of $0.544$--$0.793\%$ and relative single-path variance-estimator standard deviations of $7.59$--$8.79\%$.
The simulated OU per-run MARE is $5.89\pm0.19\%$, and its full covariance relative Frobenius error is $23.54\pm0.19\%$.
These quantify finite-trajectory resolution under the Gaussian reference, rather than the nonlinear chain's approximation error.

\subsection{Nonlinear finite-second-moment experiment}
The recursion~\eqref{eq:fourwell} uses Gaussian or $t_3/\sqrt3$ coordinates transformed by
$L=\left(\begin{smallmatrix}1&0\\0.4&\sqrt{1.84}\end{smallmatrix}\right)$.
For each noise and $a=0.02,0.01,0.005$, 16 paths start at the positive stable root, discard physical time 50 and retain physical time 5000.
The save interval is 0.5 and seeds are $31092700+100000i+1000k+r$, with noise index $i$ and descending step index $k$.
All observations and their nine-root labels are saved.
Stable conditional statistics pool the four stable labels, whose root covariances coincide.
The stable root-centered second moment divides by the number of stable observations; cellwise weighted second moments instead divide by the full sample count.
All-root truncated moments include unstable labels and use radii $R=1,2,4,8,16$.

For a two-dimensional Gaussian with covariance eigenvalues $\lambda_1,\lambda_2$, set
$v(\theta)=\lambda_1\cos^2\theta+\lambda_2\sin^2\theta$.
Its truncated moment is
\[
\E\min(\norm G^2,R^2)=\frac4\pi\int_0^{\pi/2}v(\theta)
\{1-e^{-R^2/(2v(\theta))}\}\,d\theta,
\]
evaluated with 256-node Gauss--Legendre quadrature.
Bounded truncations admit independent-run standard errors; raw quadratic statistics under $t_3$ are reported with median/IQR and individual-run points.
These finite-horizon measurements assess moment and tail behavior without claiming an empirical proof of uniform integrability.

\begin{figure}[t]
\centering\includegraphics[width=\linewidth]{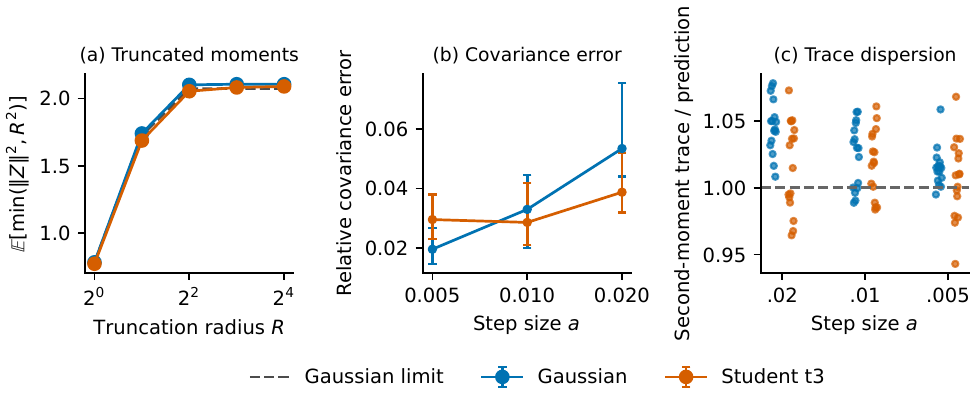}
\caption{\textbf{Finite-second-moment evidence in the nonlinear nongradient system.} (a) All-root truncated moments at $a=0.005$, with one standard error across 16 runs and a deterministic Gaussian reference. (b) Median relative covariance error and interquartile range. (c) All individual stable root-centered trace ratios; blue is Gaussian noise and orange is $t_3$. Raw quadratic measurements use dispersion summaries because $t_3$ has no fourth moment.}
\label{fig:finite-second}
\end{figure}

\subsection{Rootwise mean measurements}\label{app:mean-measurements}
At the positive root of~\eqref{eq:finite-sum}, Theorem~\ref{thm:bias} gives
\[
b=(-0.0290244628,-0.0265740280)^T,\qquad
k_f=(-0.0555984908,-0.0272547532)^T
\]
for the two preselected outputs $f=(x_1+x_2,\tanh(x_1-x_2/2))$ and $B=1$.
An independent set of 48 paths uses steps $0.04,0.02,0.01$, 16 paths per step, retained physical exposure $T=400/a$, burn-in 50, and seeds $32092700+1000k+r$.
Its smallest-step estimates are $-0.08186\pm0.00926$ and $-0.03026\pm0.00173$.
The confirmation uses 32 new independent paths per step and fixes $T=1600/a$ before execution, with seeds $34092700+1000k+r$.
Here $k=0,1,2$ indexes descending steps and $r$ starts at zero.
Every path starts at the positive root; all nine labels are recorded, and all retained states in these runs belong to that root's cell.
No stopping or retention rule depends on agreement with the coefficient.

Let $b_0=\lceil50/a\rceil$, $b_1=\lceil200/a\rceil$, $N=\lceil1600/a^2\rceil$ and $q=\operatorname{round}(0.5/a)$.
The confirmation records the post-update state $X_{i+1}$ at indices $i=b_0+kq<b_1+N$.
The primary window uses $b_0\le i<b_0+N$; the paired later-burn window uses $b_1\le i<b_1+N$ on the same grid.
Actual save intervals are $0.48,0.5,0.5$, respectively.
The record of each path includes exact first and last save times and the root-conditioned count.
Quarter, half and full primary windows form nested measurements, not independent replicates.
All 96 paths were completed, totaling $673{,}120{,}000$ updates.

\begin{table}[t]
\centering\small
\caption{Signed mean coefficients from 32 independent confirmation paths per step, with one standard error. Both outputs use the same paths.}
\label{tab:mean-confirmation}
\begin{tabular}{ccc}
\toprule
$a$ & $(\overline f_1-f_1(m))/a$ & $(\overline f_2-f_2(m))/a$\\
\midrule
$0.04$ & $-0.05988\pm0.00455$ & $-0.028306\pm0.000644$\\
$0.02$ & $-0.05769\pm0.00409$ & $-0.027889\pm0.000539$\\
$0.01$ & $-0.05470\pm0.00429$ & $-0.027249\pm0.000498$\\
Theory & $-0.05560$ & $-0.027255$\\
\bottomrule
\end{tabular}
\end{table}

For $R$ independent paths of physical exposure $T$, the OU approximation gives
\[
\operatorname{Var}\!\left(\frac B a\{\overline f-f(m)\}\right)
\approx\frac B{RTa}\operatorname{diag}(JA^{-1}C_1A^{-T}J^T),\qquad J=Df(m).
\]
The diagonal constants are $0.887625$ and $0.016139$; with $R=32,aT=1600$ the predicted standard errors are $0.00416$ and $0.000561$.
Keeping $aT$ fixed controls this leading sampling scale as $a$ decreases.
The later-minus-primary coefficient differences at $a=0.01$ are $-0.000130$ and $-0.0000191$; their uncertainty is computed from paired differences, not independent-window error bars.
The windows overlap by approximately $99.625\%$, $99.813\%$ and $99.906\%$ at the three steps, so this is a sensitivity comparison rather than independent evidence of equilibration.

\paragraph{An exact empirical decomposition.}
Write $\delta=X-m$, $M=\overline{\delta\delta^T}/a$, and $q_g=(g_1''(m_1),g_2''(m_2))^T$.
For this separable drift define
\[
r_g(X)=g(X)-A\delta-\tfrac12q_g\odot\delta^{\odot2},\qquad
r_f(X)=f(X)-f(m)-J\delta-\tfrac12(D^2f(m):\delta\delta^T).
\]
All bars in the following identity average the same saved states:
\begin{align}\label{eq:mean-decomposition}
\widehat k_f-k_f
={}&JA^{-1}\overline g/a\nonumber\\
&+\tfrac12D^2f(m):(M-S)-\tfrac12JA^{-1}\{q_g\odot\diag(M-S)\}\nonumber\\
&+(\overline r_f-JA^{-1}\overline r_g)/a.
\end{align}
This is an algebraic decomposition, without replacing the empirical gradient by zero.
For $f_1$ at $a=0.01$, the three terms averaged over the confirmation paths are $0.001277$, $-0.000305$ and $-0.0000724$.
For $f_2$, they are $0.000194$, $-0.000158$ and $-0.0000297$.
The maximum algebraic residual across all recorded windows is $2.8\times10^{-15}$.
The 48-path measurement's larger $f_1$ error at this step is likewise predominantly empirical drift: its total error is $-0.02626$, with drift contribution $-0.02600$.
For this measurement, states are recovered from the saved invertible observables: $u=\operatorname{arctanh}(f_2)$, $x_2=(f_1-u)/1.5$ and $x_1=f_1-x_2$.
These observations identify the source of measured error, without furnishing a numerical bound on the theorem's $o(a)$ remainder.

To distinguish boundary effects from innovation averaging, the confirmation also accumulates gradients and innovations at \emph{every pre-update state}.
For each complete update window, the recursion gives the exact identity
\[
\frac1N\sum_{i=b}^{b+N-1}g(X_i)
=\frac1N\sum_{i=b}^{b+N-1}\xi_i-\frac{X_{b+N}-X_b}{aN}.
\]
The recorded balance residual is at most $7.4\times10^{-17}$.
These full-update diagnostics use a different schedule from the thinned, post-update averages in~\eqref{eq:mean-decomposition} and are stored separately.

\begin{figure}[t]
\centering\includegraphics[width=\linewidth]{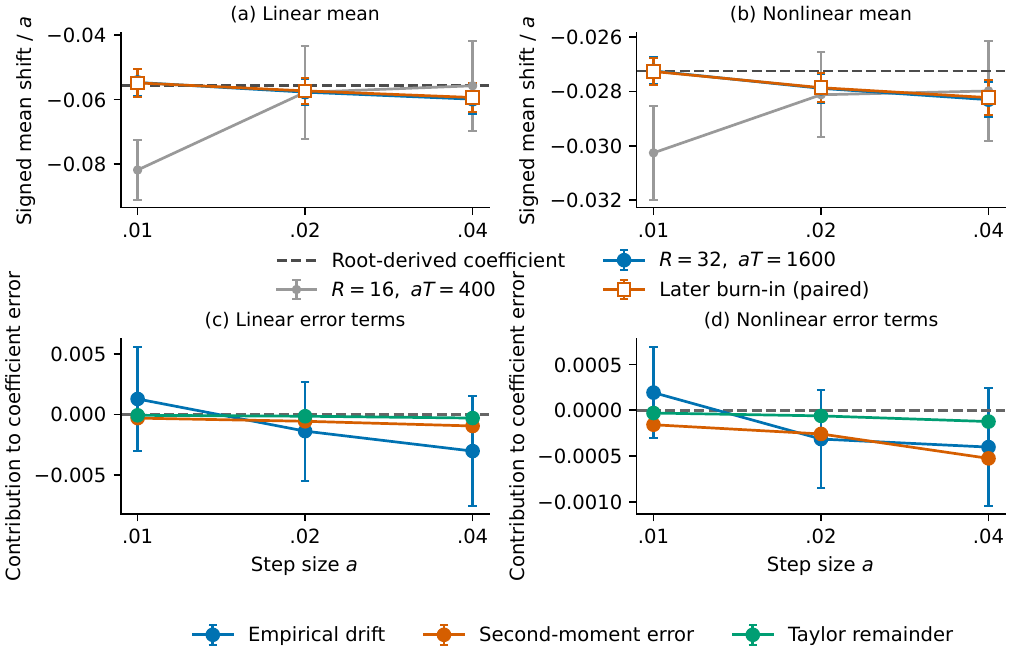}
\caption{\textbf{Independent mean measurements and error sources.} (a,b) Signed coefficients for both preselected observables. Gray paths use $R=16,aT=400$; the independent confirmation uses $R=32,aT=1600$, with paired burn-in windows. Dashed lines are root-derived coefficients. (c,d) The three terms in~\eqref{eq:mean-decomposition} for the confirmation's primary window. Bars are one standard error across independent paths; outputs and windows within a path are dependent.}
\label{fig:mean-bias}\label{fig:mean-confirmation}
\end{figure}

\subsection{Tilt, dither and local neural geometry}
\begin{table}[t]
\centering\small
\caption{Neural augmentation ablation. Mean logit-variance ratios (one standard error across 16 trajectories) use each arm's root-derived reference. The common steps are $a_s=0.002603$ and $a_l=0.020825$.}
\label{tab:augmentation}
\begin{tabular}{lcccc}
\toprule
Augmentation & $B=1,a_s$ & $B=1,a_l$ & $B=16,a_s$ & $B=16,a_l$\\
\midrule
Tilt + dither & $0.986\pm0.012$ & $1.473\pm0.011$ & $0.985\pm0.007$ & $1.009\pm0.006$\\
Tilt only & $0.986\pm0.012$ & $1.507\pm0.015$ & $0.986\pm0.007$ & $1.008\pm0.006$\\
Dither only & $0.986\pm0.012$ & $1.469\pm0.011$ & $0.985\pm0.007$ & $1.009\pm0.006$\\
Neither & $0.986\pm0.012$ & $1.502\pm0.014$ & $0.986\pm0.007$ & $1.008\pm0.006$\\
\bottomrule
\end{tabular}
\end{table}
The four arms are $(\tau,\gamma)=(10^{-4},0.015),(10^{-4},0),(0,0.015),(0,0)$.
They share the data preprocessing, ridge, output weights and test probes.
Removing dither leaves the objective and its root unchanged; only $C_1$ and $S_1$ are recomputed.
For $\tau=0$, deterministic continuation and Newton refinement give a root at distance $0.001988$ from the baseline, with minimum Hessian eigenvalue $0.024356$.
The four deterministic roots/arms have the same test accuracy to the reported precision.
Each arm defines a controlled comparison around its separately computed root.

Each arm uses common actual steps $a_s=0.0026031543133731632$ and $a_l=0.020825234506985306$, batches 1 and 16, and 16 paths per setting.
The baseline endpoint paths are also used in Figure~\ref{fig:neural-moments}; the other arms supply 192 paths.
Endpoint seed indices are paired across arms: $29092700+100000i+1000k+r$, where $i=0,2$ and $k=0,3$.
Burn-in, exposure and saving match the baseline.
All retained states are included in the moment statistics.

Because $C_{\rm dither}=\gamma^2I$ and $H$ is symmetric, its contribution is $S_{\rm dither}=\gamma^2H^{-1}/2$.
Directional fractions are $u^TS_{\rm dither}u/(u^TS_1u)$; output fractions use rows of the frozen logit Jacobian for $u$.
The resulting decomposition in Figure~\ref{fig:dither} distinguishes weak-curvature parameter directions from prediction directions.

\begin{figure}[t]
\centering\includegraphics[width=\linewidth]{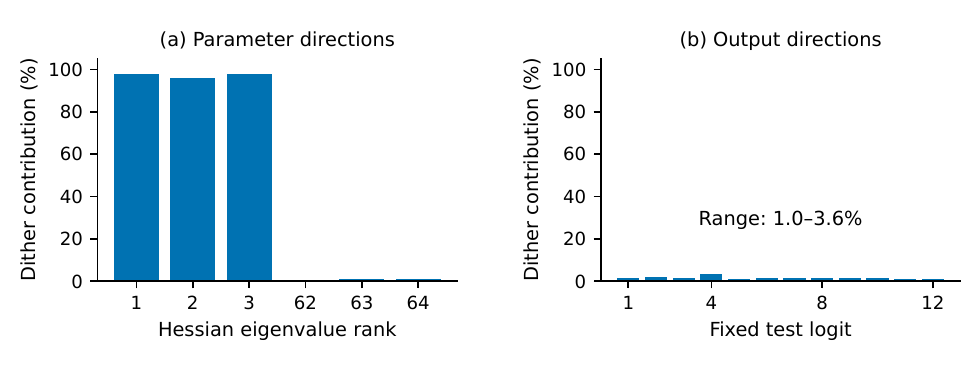}
\caption{\textbf{Directional effect of dither.} Its fraction of theoretical variance in six Hessian eigenvectors (a) and 12 fixed logit directions (b). Both panels use the same 0--100 percent scale; the output range is annotated. These are deterministic contributions to the local Lyapunov covariance.}
\label{fig:dither}
\end{figure}

\begin{figure}[t]
\centering\includegraphics[width=\linewidth]{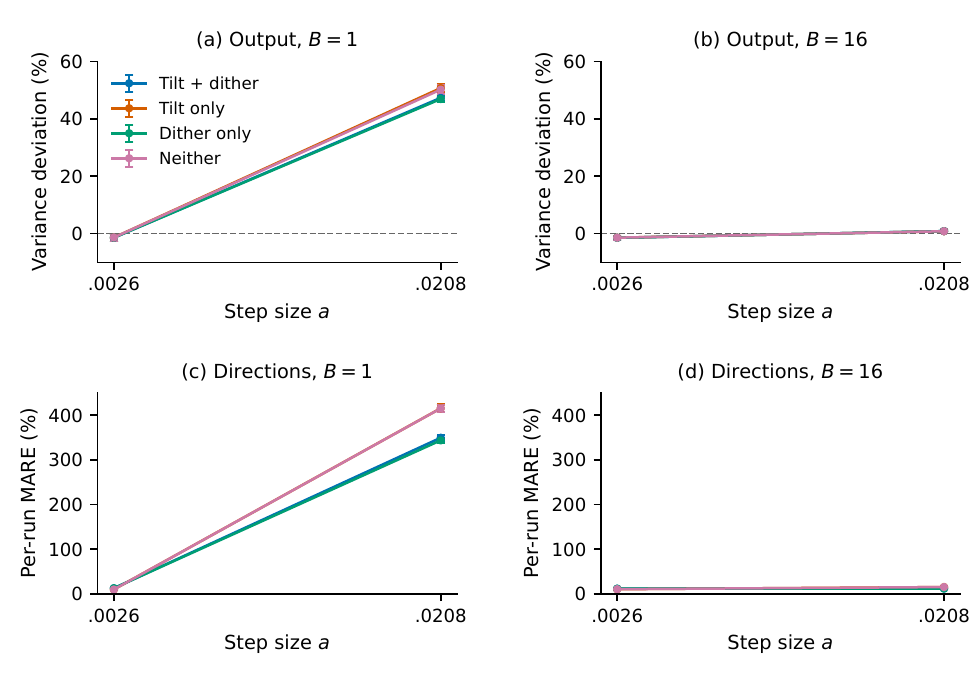}
\caption{\textbf{Augmentation ablation across step and batch.} Logit-variance deviations from prediction (top) and MARE in six Hessian directions (bottom). Each row shares its vertical range across batches. Each arm uses its own root-derived reference; bars are one standard error across 16 paths. Comparable small-step output accuracy coexists with substantial directional errors at the largest step and $B=1$.}
\label{fig:augmentation}
\end{figure}

\subsection{Same-state finite-step diagnostics}
The same 32 baseline $B=1$ paths at the two step endpoints supply these diagnostics and the corresponding variance and loss statistics.
Parameters are recorded on the 201-state loss schedule.
On each of those states we compare $F(X)-F(m)$ with $\tfrac12(X-m)^TH(X-m)$ and record output linearization residuals.
On every hundredth saved state, we evaluate $g(X)-H(X-m)$ and the complete training-sum covariance $C(X)$.
All diagnostics use common state schedules within a comparison.

The actual and quadratic mean-gap ratios are $1.029\pm0.011$ and $1.023\pm0.010$ at $a_s$, versus $1.893\pm0.015$ and $3.000\pm0.041$ at $a_l$.
The scaled mean drift-residual norm rises from $0.0185$ to $0.1471$.
The mean relative covariance variation rises from $0.261$ to $2.605$.
These measurements distinguish the objective's Taylor remainder from changes in the fluctuation covariance; they do not identify a universal step threshold.

\begin{figure}[t]
\centering\includegraphics[width=\linewidth]{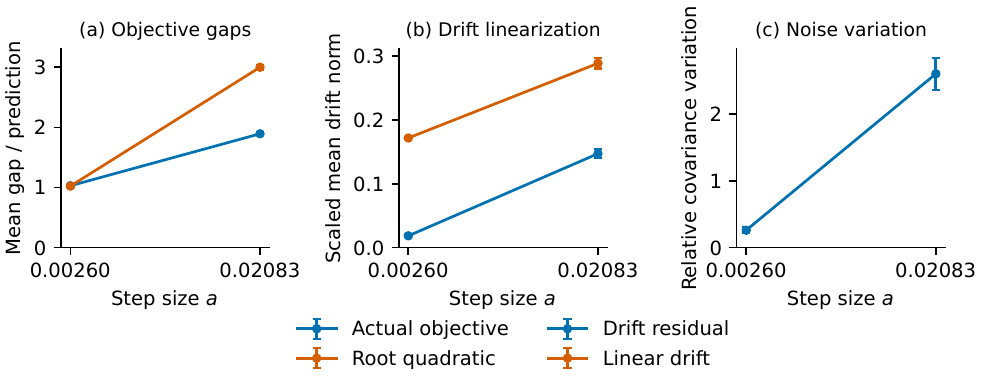}
\caption{\textbf{Neural finite-step decomposition on paired states.} Actual and quadratic objective gaps (a), scaled drift residual and linear drift norm (b), and relative state-dependent covariance variation (c). Bars are one standard error across 16 paths at each step. The covariance norm is Frobenius; drift norms are Euclidean.}
\label{fig:boundary}
\end{figure}

\subsection{Observation-window sensitivity}\label{app:exposure}
To separate trajectory measurement error from persistent finite-step departures, 64 independent baseline paths use $B=1,16$, steps $a_s,a_l$, burn-in 600 and retained physical exposure 8000.
Seeds are $33092700+100000i+1000k+r$, with ascending batch and step indices.
The actual save interval is $0.9996112563352947$, giving 8004 states per path.
Each path's first 2001 states provide a paired short-window estimate; the two window estimates are correlated.
A further 64 stationary OU paths use the 8004-state schedule, seed stream 30132700 and the baseline matrices.

At $a_s$, the per-run logit MARE falls from $6.70\pm0.41\%$ to $3.20\pm0.17\%$ for $B=1$, and from $6.81\pm0.29\%$ to $2.99\pm0.15\%$ for $B=16$.
The corresponding OU references are $5.89\pm0.19\%$ and $3.34\pm0.10\%$.
At $a_l,B=1$, the long-window variance ratio remains $1.478\pm0.008$ and the actual loss ratio is $1.916\pm0.010$.
At $a_s,B=1$, the long-window loss ratio is $1.028\pm0.004$, resolving a smaller finite-step objective bias even where output variance is close to its local reference.
Thus increasing the observation window reduces the variance-estimation floor while preserving the large-step departures.

\begin{figure}[t]
\centering\includegraphics[width=\linewidth]{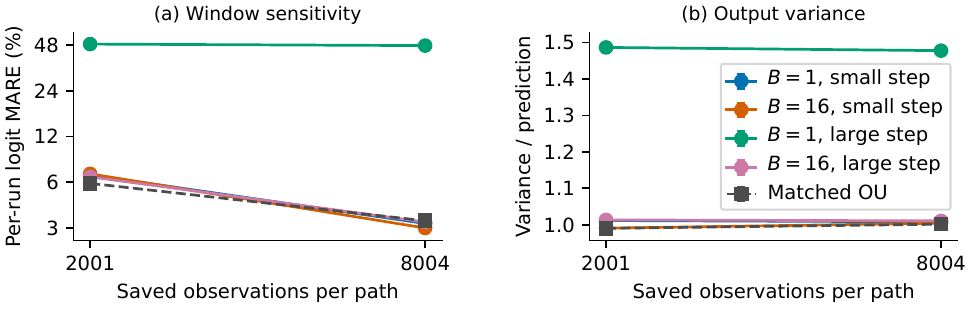}
\caption{\textbf{Observation-window sensitivity.} Per-run logit MARE (a) and mean variance ratio (b) for paired short and full windows of 16 independent neural paths per setting. The stationary OU references use 64 paths at each window length. Bars are one standard error across paths; shared paths make the two neural estimates dependent.}
\label{fig:exposure}
\end{figure}

\section{Model dimension, data geometry and communication horizons}\label{app:expansion}

\subsection{Controlled increases in neural dimension}\label{app:dimension}
We extend the fixed-output model of Proposition~\ref{prop:neural} to 16 and 64 hidden units, training 256 and 1024 input weights, respectively.
The output vector repeats $(0.6,-0.8,1.1,-1.3)$ and is multiplied by $\sqrt{4/h}$ at width $h$.
The optical-digits task retains the same writer-disjoint split and 16 training-fitted whitened principal components.
Ridge, tilt and dither remain $\lambda=0.02$, $\tau=10^{-4}$ and $\gamma=0.015$.
Each model has its own reference computed from its full empirical gradient covariance.
This comparison tests moment prediction in higher fixed dimension; both the architecture and its covariance change with width.

To assess a different data geometry, a 16-unit model uses the UCI Wisconsin Diagnostic Breast Cancer dataset \citep{WolbergEtAl1993}.
The 569 examples have 30 real-valued features; identifier columns are excluded.
A fixed within-class permutation (seed 35092717) assigns $\lfloor0.7n_c\rfloor$ examples of each class to training, giving 397 training and 172 test examples.
Feature centering, scaling, PCA and whitening are fitted on training data only, retaining 16 components.
The labels are $-1$ for benign and $+1$ for malignant.
This is a fixed classification benchmark for fluctuation measurements.

For each of the three cases, four Gaussian initializations of scale $0.2$ undergo L-BFGS optimization and exact-Hessian Newton refinement.
The root-search stream is $35092700+1000c$, where $c=0,1,2$ indexes digits at widths 16 and 64, followed by WDBC at width 16; it also generates the fixed tilt direction $r\sim\N(0,I_p/p)$.
Every candidate satisfies gradient norm below $10^{-8}$ and minimum Hessian eigenvalue above $10^{-6}$; the smallest qualified objective selects the reference.
All candidates and their objective values are retained.
The three Hessian spectral intervals are $[0.02088,0.24696]$, $[0.02018,0.25429]$ and $[0.02044,0.34261]$.
The respective test accuracies of the deterministic roots are $97.20\%$, $97.20\%$ and $95.93\%$.
No fluctuation outcome or test accuracy enters root selection.

Four pilot paths per case and step use $B=1$, retained physical exposure 1000 and seeds $35192700+100000c+1000i+r$.
The independent study uses 16 paths at every combination of $a\lambda_{\max}(H)\in\{0.000625,0.005\}$ and $B\in\{1,16\}$, with physical exposure 8000 and burn-in 600.
Its seeds are $37092700+100000c+10000i+1000b+r$, with ascending step and batch indices.
All 192 evaluation paths are retained separately from the 24 pilots.
The first six test examples of each class define the 12 probes before any trajectories are sampled.

With $b_0=\lceil600/a\rceil$, $N=\lceil8000/a\rceil$ and $q=\max\{1,\operatorname{round}(1/a)\}$, the saved states are $X_{b_0+1+kq}$ for $kq<N$.
Full objective values and parameters are retained every tenth save; all probe series, six parameter projections and root distances are retained.
No observations are excluded according to their displacement.
These local finite-horizon measurements do not enumerate the neural root set or estimate its invariant component weights.

\paragraph{Matched OU measurements.}
Every case and step has 64 independent stationary OU paths with the exact same number and spacing of saves as its nonlinear paths.
The reference is shared across batches after $\sqrt{B/a}$ scaling.
If $H=V\Lambda V^T$ and $\widetilde S=V^TSV$, the eigenbasis transition has diagonal mean map $D=e^{-\Lambda\Delta t}$ and innovation covariance $\widetilde S-D\widetilde S D$.
We initialize with $\N(0,\widetilde S)$, evaluate the linearized probes at every save, and evaluate the quadratic objective on the same loss schedule.
The reference seed streams are $38092700+100000c+1000i$.
This yields 384 OU paths; within each path all probes share the same multivariate process.
Table~\ref{tab:neural-expansion} and Figure~\ref{fig:neural-expansion} report the resulting measurements.

\begin{table}[t]
\centering\small
\caption{Independent dimension/data study. Each row uses 16 paths with retained physical exposure 8000. Values are mean $\pm$ one standard error. MARE is a percentage.}
\label{tab:neural-expansion}
\begin{tabular}{llrrr}
\toprule
Model & $a\lambda_{\max},B$ & Variance ratio & MARE (\%) & Loss ratio\\
\midrule
Digits 256 & $0.000625, 1$ & $1.012\pm0.005$ & $3.018\pm0.225$ & $1.030\pm0.005$\\
Digits 256 & $0.000625, 16$ & $0.998\pm0.003$ & $3.245\pm0.203$ & $0.997\pm0.003$\\
Digits 256 & $0.005, 1$ & $1.520\pm0.006$ & $52.034\pm0.571$ & $1.927\pm0.010$\\
Digits 256 & $0.005, 16$ & $1.001\pm0.004$ & $3.004\pm0.207$ & $1.012\pm0.003$\\
Digits 1024 & $0.000625, 1$ & $1.024\pm0.003$ & $3.689\pm0.273$ & $1.027\pm0.003$\\
Digits 1024 & $0.000625, 16$ & $0.996\pm0.005$ & $3.288\pm0.185$ & $0.999\pm0.003$\\
Digits 1024 & $0.005, 1$ & $1.525\pm0.008$ & $52.514\pm0.764$ & $1.675\pm0.007$\\
Digits 1024 & $0.005, 16$ & $1.006\pm0.004$ & $2.622\pm0.140$ & $1.011\pm0.002$\\
WDBC 256 & $0.000625, 1$ & $1.003\pm0.003$ & $2.742\pm0.174$ & $1.012\pm0.005$\\
WDBC 256 & $0.000625, 16$ & $0.997\pm0.003$ & $2.769\pm0.144$ & $1.003\pm0.004$\\
WDBC 256 & $0.005, 1$ & $1.108\pm0.004$ & $10.979\pm0.396$ & $1.170\pm0.006$\\
WDBC 256 & $0.005, 16$ & $1.005\pm0.004$ & $3.099\pm0.164$ & $1.009\pm0.005$\\
\bottomrule
\end{tabular}
\end{table}
\begin{figure}[t]
\centering\includegraphics[width=\linewidth]{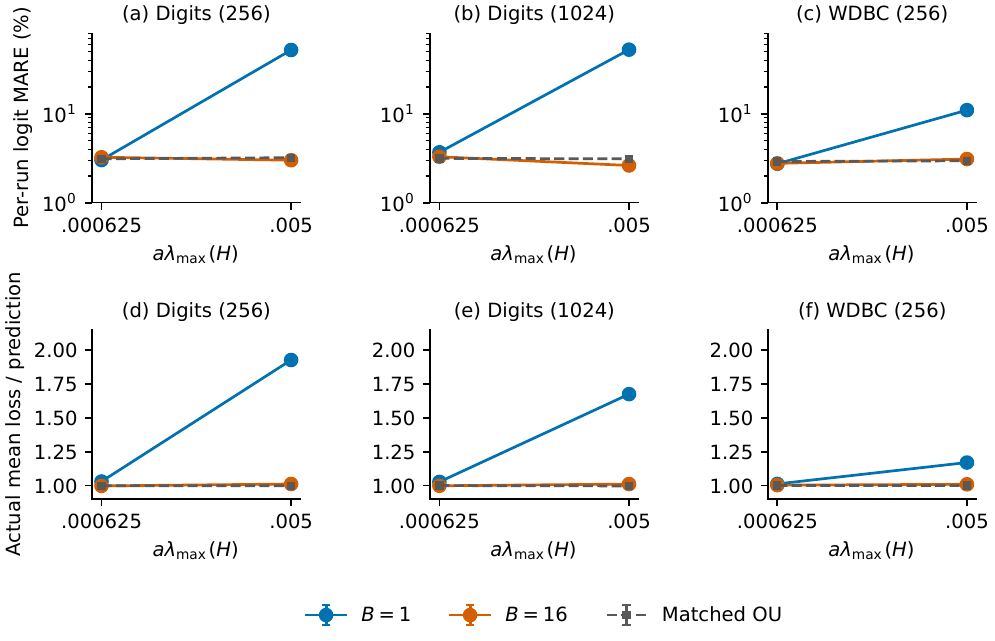}
\caption{\textbf{Moment predictions across dimension and data.} Per-run logit MARE (top) and actual objective-gap ratio (bottom) for 16 independent paths per setting. Exact stationary OU references use 64 paths per case and step, matching the save and loss schedules and sharing the scaled reference across batches. All bars are one standard error across paths. The OU loss is quadratic.}
\label{fig:neural-expansion}
\end{figure}

At the smaller normalized step, per-run logit MARE ranges from $2.74\%$ to $3.69\%$ across the six model/batch settings, compared with matched OU means of $2.91$--$3.15\%$.
The corresponding aggregate MARE ranges from $0.64\%$ to $2.35\%$.
Actual objective-gap ratios range from $0.997$ to $1.030$.
These statistics support small-step moment prediction across the tested dimensions while resolving modest objective bias in some settings.

At the larger step with $B=1$, the two digits models have variance ratios $1.520\pm0.006$ and $1.525\pm0.008$, and actual loss ratios $1.927\pm0.010$ and $1.675\pm0.007$.
WDBC has smaller departures, with variance ratio $1.108\pm0.004$ and loss ratio $1.170\pm0.006$.
At $B=16$, all three larger-step loss ratios lie between $1.009$ and $1.012$, and per-run MARE between $2.62\%$ and $3.10\%$.
Thus increasing batch size improves the approximation at this fixed step; dimension alone does not explain the observed departures.

Dither accounts for $13.86\%$, $38.39\%$ and $7.00\%$ of the full noise trace in the three cases, respectively, but only $0.88$--$2.82\%$, $0.84$--$2.31\%$ and $0.49$--$1.22\%$ of the predicted probe variances.
The growing parameter dimension therefore changes the role of dither in parameter-space measurements even when its output-space contribution stays small.
These are controlled fixed-dimensional networks with trained input weights and fixed output weights; they do not test the theorem for unrestricted deep architectures or nonisolated stationary sets.

\subsection{Independent long-horizon communication measurements}\label{app:long-communication}
The scalar model $g(x)=x+5\cos x$ separates local fluctuation accuracy from basin exploration.
A fixed-budget pilot extends the Gaussian and standardized $t_3$ comparators to 20 million retained updates at each of $a=0.08,0.04$, with four paths per noise and step, half initialized at each minimum.
All use 100,000 burn-in updates and the same radius-$0.25$ core definition as Appendix~\ref{app:numerics}.
The seed is $36092700+100000k+1000i+r$, with $k=0$ for Gaussian and $k=2$ for $t_3$.
The Gaussian paths have no completed transitions; the $t_3$ paths have 450 at $0.08$ and 46 at $0.04$.
The Gaussian result describes initialization dependence over this horizon, rather than noncommunication of the chain.

An independent $t_3$ study uses $a=0.04,0.02$ and 12 paths per step, half from each minimum.
The one-jump proxy fixes retained budgets of 51,400,000 and 410,900,000 updates per path before execution, targeting 400 transitions per step in expectation under that proxy.
Each path nevertheless runs to its fixed horizon regardless of the realized count.
The seed is $39292700+1000i+r$, and the burn-in is 100,000 updates.
The total executed budget is $5{,}550{,}000{,}000$ updates.
The nearly eightfold budget increase for a halved step follows the $a^{-3}$ iteration scale of the scalar tail model.

We retain every completed transition time and measurements at one tenth, one quarter, one half and the full post-burn horizon.
These windows are nested within paths.
Full-window left occupations are $0.83646\pm0.01579$ and $0.83964\pm0.01012$, compared with the limiting proxy $0.84621$.
The completed-transition totals are 388 and 413, respectively.
Completion means reaching the opposite radius-$0.25$ root neighborhood; raw separator crossings total 486 and 587 and include returns before completion.
The left-to-right rate divided by its root-based one-jump probability is $0.970$ and $1.028$; the corresponding right-to-left ratios are $0.921$ and $1.002$.
Both occupation and directional-rate measurements are compatible with the proxy at the resolution of this fixed-budget study.
Rate intervals resample the 12 independent paths 10,000 times and take percentile endpoints at $2.5\%$ and $97.5\%$; complete paths, including their counts and exposures, are resampled together.
The bootstrap seeds are $40192700+\operatorname{int}(100a)$.

\begin{figure}[t]
\centering\includegraphics[width=\linewidth]{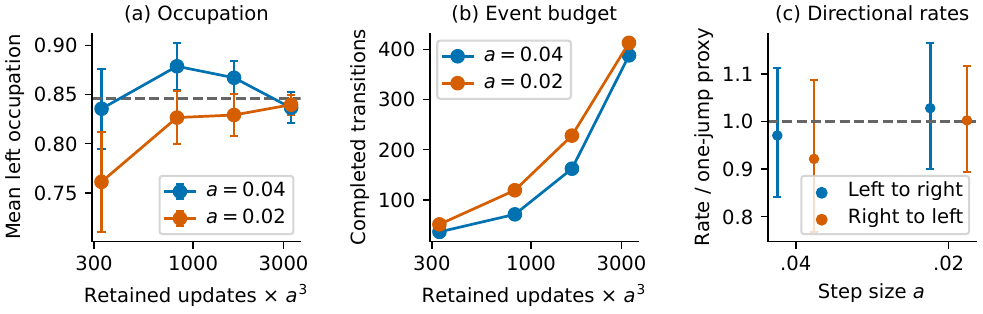}
\caption{\textbf{Long-horizon communication at a smaller step.} (a) Mean occupation and one standard error over 12 paths at nested horizons; the dashed line is the heavy-tail limit. (b) Total completed transitions against retained updates scaled by $a^3$. (c) Directional rates relative to the finite-step one-jump proxy, with 95\% path-bootstrap intervals. Nested windows are dependent.}
\label{fig:communication-expansion}
\end{figure}

At the full horizon, all four path-bootstrap rate intervals include their corresponding one-jump probability.
At $a=0.02$, the left-to-right and right-to-left ratios have 95\% intervals approximately $[0.901,1.165]$ and $[0.893,1.117]$.
The smaller-step budget yields a comparable event count, rather than the eightfold loss of information that an unchanged update budget would incur.

To examine the one-jump mechanism directly, each update also records whether its innovation exceeds the root-to-separatrix threshold in either direction.
At $a=0.02$, the full streams contain 233 positive and 1322 negative threshold exceedances; the exact noise law predicts 236.4 and 1300.5 in expectation.
Restricting these offers to times when the last visited core is the departure well gives 203 left-to-right and 209 right-to-left offers, compared with 204 and 209 completed transitions.
An offer is defined at the deterministic root threshold and need not coincide with a completed transition from the actual state.
Their close counts support the one-jump approximation at this step, without establishing a transition-law limit or a mixing-time bound.

\section{Root continuity without neighborhood uniform integrability}\label{app:root-continuity-example}
The stationary tail control in Lemma~\ref{lem:stationary-noise-tails} does not require innovation-square uniform integrability on a fixed root neighborhood.
The following example separates these two properties.

\begin{proposition}\label{prop:root-continuity-example}
There is a scalar recursion satisfying Assumption~\ref{ass:standing} with uniformly bounded innovation variance, but, for every $r>0$,
\[
\lim_{R\to\infty}\sup_{|x|\le r}
\E[\xi(x,\zeta)^2\one\{|\xi(x,\zeta)|>R\}]>0.
\]
\end{proposition}
\begin{proof}
Set $g(x)=x$ and $U(x)=x^2/2$, with the single hyperbolic root $0$.
Let $\eta_0,\eta_1$ be independent Rademacher variables and let $V$ be independent and uniform on $[0,1]$.
For integers $k,n\ge1$, define
\[
x_k=2^{-k},\quad x_{kn}=x_k+2^{-k-n-2},\quad r_{kn}=2^{-k-n-5},
\quad v_k=2^{-k},\quad M_{kn}=2^{k+n}.
\]
The intervals $I_{kn}=(x_{kn}-r_{kn},x_{kn}+r_{kn})$ are pairwise disjoint.
For adjacent intervals at fixed $k$, the center separation is $2^{-k-n-3}$ and the sum of radii is $3\cdot2^{-k-n-6}$.
Intervals farther apart in that cluster are therefore also disjoint.
The whole $k$th cluster lies between $x_k$ and $(73/64)x_k$, so distinct clusters are separated as well.
Define
\[
\psi_{kn}(x)=\left(1-\frac{|x-x_{kn}|}{r_{kn}}\right)_+,
\qquad p_{kn}(x)=v_k\psi_{kn}(x)/M_{kn}^2.
\]
For $x\in I_{kn}$ let
\[
\xi(x,\zeta)=\eta_0+M_{kn}\eta_1\one\{V\le p_{kn}(x)\},
\]
and set $\xi(x,\zeta)=\eta_0$ elsewhere.
This mechanism is measurable, centered and independent of $a$.
Its variance equals $1+v_k\psi_{kn}(x)$ on an active interval and equals $1$ elsewhere.
Thus $1\le C(x)\le3/2$, and $C(x)\to1=C(0)$ as $x\to0$.

The innovation laws are weakly continuous.
Inside each interval and at its endpoints, this follows from continuity of its mixing probability.
At a cluster point $x_k$, the active probabilities tend to zero as $n\to\infty$.
At $0$, they tend to zero as $k\to\infty$; these are the only other accumulation points.
For a bounded continuous test, its expectation differs from the Rademacher expectation by at most twice its supremum norm times the active probability.
This verifies continuity also at the accumulation points, even though the rare jump sizes diverge.
The affine next-state map consequently gives a Feller kernel for every sufficiently small $0<a<a_0<1$.
For the $C_0$ convention, vanishing at infinity follows because outside the bounded interval union the update is $(1-a)x+a\eta_0$.
All remaining conditions follow from $g(x)=x$ and $U(x)=x^2/2$.

Fix $r>0$ and choose $k$ so the whole $k$th cluster lies in $(0,r)$.
For any $R>1$, choose $n$ with $M_{kn}-1>R$.
At the interval center, only the rare-jump event contributes above this threshold, and
\[
\E[\xi(x_{kn},\zeta)^2\one\{|\xi(x_{kn},\zeta)|>R\}]
=\frac{v_k}{M_{kn}^2}(M_{kn}^2+1)\ge v_k>0.\qedhere
\]
\end{proof}

\section{Numerical invariant-law transport with alternating weights}\label{app:transport-study}

\paragraph{Model and purpose.}
Use the invariant rectangles of Proposition~\ref{prop:counter}, with
$g(x,y)=((x+1)x(x-2)/(1+x^2),y)$ and independent coordinate noises taking values $\pm0.2$ with equal probabilities.
The roots are $m_-=(-1,0)$ and $m_+=(2,0)$, and the covariance targets are
$S_-=\diag(1/75,1/50)$ and $S_+=\diag(1/60,1/50)$.
The rectangle laws are unique by contraction.
Since the coordinate updates and noises are independent, the product of their marginal invariant laws is invariant, and uniqueness identifies it with the rectangle law.
Let $\mu_{a,j}$ be this law centered at $m_j$ and divided by $\sqrt a$.
We choose actual invariant mixtures with weights
\[
w_a=\tfrac12+\tfrac14\cos\!\left(\pi\log_2(0.04/a)\right),\qquad
\pi_a=w_a\pi_{a,-}+(1-w_a)\pi_{a,+}.
\]
Along the tested dyadic steps, the weights alternate between $3/4$ and $1/4$.
This selection extends to a nonconvergent sequence of weights as $a\downarrow0$.
The experiment directly examines the same-label assertion of Theorem~\ref{thm:W2}.

\paragraph{Transport calculation.}
For product laws $\mu=\mu_1\otimes\mu_2$ and $\nu=\nu_1\otimes\nu_2$ with finite second moments,
\[
W_2^2(\mu,\nu)=W_2^2(\mu_1,\nu_1)+W_2^2(\mu_2,\nu_2).
\]
Indeed, each marginal of any coupling is an admissible one-dimensional coupling, which proves the lower bound after summing coordinate costs.
The product of the optimal marginal couplings attains the bound.
Consequently, one-dimensional quantile integrals give the full two-dimensional component costs, and
\begin{equation}\label{eq:transport-study}
D_a^2=w_aW_2^2(\mu_{a,-},\N(0,S_-))+(1-w_a)W_2^2(\mu_{a,+},\N(0,S_+)).
\end{equation}
Choosing a common label and rescaling its optimal coupling proves
$W_2(\pi_a,w_a\N(m_-,aS_-)+(1-w_a)\N(m_+,aS_+))\le\sqrt a D_a$.
$D_a$ is the optimal scaled cost among label-preserving couplings; unrestricted couplings may have smaller cost.
As a structural comparison, replace the first-coordinate variance in both targets by $0.015$, leaving their second-coordinate variances unchanged.
This is a fixed equal-weight pooled reference: its covariance stays fixed as the selected mixture weights alternate.
Proposition~\ref{prop:shared-covariance} below gives the corresponding comparison with an optimally chosen common covariance.

\paragraph{Discretization and quadrature.}
For each marginal, discretize its invariant interval by a grid with spacing $h$.
Apply each of the two deterministic update maps to each grid point and split its mass linearly between adjacent destination points.
This gives a positive, mass-preserving finite-state transition operator $P_{a,h}$.
Starting at the root, iterate until the one-step probability-vector residual is below $a\times10^{-11}$ in $\ell^1$.
For the resulting sorted scaled grid values $z_i$ with masses $p_i$, define $c_i=\sum_{k\le i}p_k$, $c_0=0$, and $q_i=\Phi^{-1}(c_i)$.
The one-dimensional quantile representation of $W_2$ \citep{PeyreCuturi2019} gives, for target variance $s^2$,
\[
W_2^2\!\left(\sum_i p_i\delta_{z_i},\N(0,s^2)\right)
=\sum_i p_i z_i^2+s^2-2s\sum_i z_i\{\phi(q_{i-1})-\phi(q_i)\},
\]
where $\Phi$ and $\phi$ are the standard Gaussian distribution and density and $\phi(\pm\infty)=0$.
This follows by expanding the squared quantile difference and using
$\int_{c_{i-1}}^{c_i}\Phi^{-1}(u)\,du=\phi(q_{i-1})-\phi(q_i)$.
Thus the remaining numerical approximation is the finite-grid invariant law.

\paragraph{Resolution and results.}
All five steps $a\in\{0.04,0.02,0.01,0.005,0.0025\}$ use grids with
$16\,385$, $32\,769$, $65\,537$, $131\,073$ and $262\,145$ points per coordinate interval.
Figure~\ref{fig:root-moments}c uses the finest grid; the full sequence, probability vectors, residuals and iteration counts are supplied.
The two finest resolutions change the reported root-specific costs by at most $1.05\times10^{-6}$ in absolute value, or $0.38\%$ relative to the finest value; both maxima occur at $a=0.0025$.
At the endpoint steps, the costs are $0.00310464$ and $0.00027922$, a reduction by a factor of $11.1$.
The pooled endpoint costs are $0.00659665$ and $0.00685505$.
These deterministic calculations support decreasing same-label scaled transport and the importance of root-specific covariances.
Grid refinement measures numerical sensitivity; the theorem supplies the asymptotic conclusion independently of the discretization.

\begin{proposition}[Optimal shared-covariance comparison]\label{prop:shared-covariance}
Set $\sigma_-=\sqrt{1/75}$ and $\sigma_+=\sqrt{1/60}$.
For $w\in(0,1)$, define
\[
J_w(\Sigma)^2=wW_2^2(\N(0,S_-),\N(0,\Sigma))
+(1-w)W_2^2(\N(0,S_+),\N(0,\Sigma)),\qquad \Sigma\succeq0.
\]
The minimum is
\begin{equation}\label{eq:shared-optimum}
\inf_{\Sigma\succeq0}J_w(\Sigma)
=\sqrt{w(1-w)}\,|\sigma_--\sigma_+|,
\end{equation}
attained at $\Sigma_w^*=\diag((w\sigma_-+(1-w)\sigma_+)^2,1/50)$.
For the actual scaled invariant components, let
\[
\widetilde J_a(\Sigma)^2
=w_aW_2^2(\mu_{a,-},\N(0,\Sigma))
+(1-w_a)W_2^2(\mu_{a,+},\N(0,\Sigma)).
\]
Along $a_k=0.04\,2^{-k}$,
\[
\inf_{\Sigma\succeq0}\widetilde J_{a_k}(\Sigma)
\longrightarrow c_*:=\frac{\sqrt3}{4}|\sigma_--\sigma_+|
=0.00590169943749\ldots,
\qquad D_{a_k}\longrightarrow0.
\]
\end{proposition}
\begin{proof}
Projection onto the first coordinate contracts $W_2$.
Writing $s=\sqrt{\Sigma_{11}}$, the one-dimensional Gaussian formula gives
\begin{align*}
J_w(\Sigma)^2
&\ge w(\sigma_--s)^2+(1-w)(\sigma_+-s)^2\\
&=\{s-w\sigma_--(1-w)\sigma_+\}^2
+w(1-w)(\sigma_--\sigma_+)^2.
\end{align*}
This lower bound covers every positive-semidefinite covariance, including nondiagonal matrices.
At $\Sigma_w^*$, the second marginals coincide and the product of the optimal first- and second-coordinate couplings attains the bound.

For the actual components, the reverse triangle inequality for each Wasserstein distance, followed by the weighted Euclidean triangle inequality, gives
\[
|\widetilde J_a(\Sigma)-J_{w_a}(\Sigma)|\le D_a
\qquad(\Sigma\succeq0).
\]
Taking infima in both resulting inequalities yields
\[
\left|\inf_{\Sigma\succeq0}\widetilde J_a(\Sigma)
-\inf_{\Sigma\succeq0}J_{w_a}(\Sigma)\right|\le D_a.
\]
Theorem~\ref{thm:W2} gives $D_a\to0$ for these invariant-law selections.
On the dyadic sequence $w_a\in\{1/4,3/4\}$, both Gaussian optima equal $c_*$, proving the limit.
\end{proof}
Thus root-specific covariances permit vanishing same-label scaled transport even when every shared covariance retains a positive limiting cost.
The associated original-scale label-preserving cost is multiplied by $\sqrt a$.
The comparison fixes component centers at the roots and concerns label-preserving couplings; unrestricted mixture couplings may exchange mass between labels.

\section{Relation to diffusion, heavy-tail and manifold theories}\label{app:literature}

\paragraph{The random object and order of limits.}
The object here is an arbitrary invariant law of the discrete recursion at step $a$, followed by $a\downarrow0$.
The additive finite-variance analysis of \citet{ChenMouMaguluri2022} addresses this order of limits around a single center.
Its strongly convex, linear and contractive settings lead to stationary scaling limits and Gaussian identification under the stated uniqueness conditions.
Our assumptions instead permit multiple hyperbolic roots and state-dependent innovations.
Energy matching establishes the quadratic uniform integrability needed for moment accuracy across all components and invariant-law selections.
With a uniform scaled third moment, a Gaussian weak limit already implies qualitative $W_2$ convergence; energy matching supplies the required tail control using second innovation moments instead.
Classical stochastic approximation distinguishes this constant-gain problem from diminishing-gain convergence \citep{KushnerYin2003}.
Stochastic modified equations approximate the weak dynamics of SGD over specified time intervals and support algorithmic interpretations \citep{LiTaiE2017}.
An approximation on a fixed interval does not by itself identify a stationary limit: passing to infinite time requires additional uniform control.
\citet{MattinglyStuartTretyakov2010} develop Poisson-equation methods for time averages and invariant-measure errors of numerical SDE approximations under smoothness and ergodicity conditions.
Our proof instead uses the discrete stationary energy balance to obtain quadratic uniform integrability under second innovation moments, with potentially several invariant probabilities.
The Poisson equation enters later, for the separate smooth-test rate.
The finite-step study of \citet{LiMalladiArora2021} further motivates checking the accuracy of diffusion descriptions at the actual learning rate.
Our matched OU trajectories calibrate the variability of local moment measurements; the nonlinear trajectories independently measure their departures from the limiting prediction.

\paragraph{Heavy-tailed innovations and heavy-tailed iterates.}
\citet{GurbuzbalabanEtAl2021} show how multiplicative SGD noise can produce heavy-tailed stationary iterates even with light-tailed data; \citet{HodgkinsonMahoney2021} develop related mechanisms beyond linear models.
Innovation moments and stationary-iterate moments are therefore distinct assumptions.
We assume finite innovation variance with covariance growth controlled by the drift, together with dissipativity and descent.
For sufficiently small steps these conditions imply finite stationary second moments, and energy matching controls the second moments after rootwise rescaling.
They allow the standardized $t_3$ examples, whose third absolute moment is infinite.
The symmetric smooth-test argument exploits this particular innovation structure; it does not assert a universal $\sqrt a$ Wasserstein rate or cover stationary infinite-variance regimes.

\paragraph{Local geometry and global communication.}
At isolated hyperbolic roots, the local covariance follows a Lyapunov equation and the mean shift follows a second-order drift balance.
The new guarantee is that these local formulas remain moment-accurate across root-labeled components, with actual weights, singular covariances and second innovation moments.
Single-center stationary expansions and harmonic loss identities provide the corresponding antecedents \citep{DieuleveutDurmusBach2020,Yaida2019,AllmeierGast2024}.
Equilibrium manifolds have additional tangential dynamics: \citet{LiWangArora2022} study how SGD moves along a zero-loss manifold under general noise covariance.
Such motion is absent from a finite set of hyperbolic roots, so a manifold extension would require an additional limit along the equilibrium set.
Separately, Freidlin--Wentzell theory concerns global rare transitions in small-noise diffusions \citep{FreidlinWentzell2012}, while discrete learning dynamics and heavy-tail communication require their own transition analysis \citep{AzizianEtAl2024,WangRhee2023}.
The invariant-rectangle example retains arbitrary weights, and the scalar communicating example uses a heavy-tail transition proxy.
Both use the same principle: local moment coefficients and global occupation probabilities carry different information.

\end{document}